\documentclass[jair, twoside,11pt,theapa]{article}
\usepackage{jair, theapa, rawfonts}
\usepackage{times}
\usepackage[T1]{fontenc}
\usepackage{microtype}
\usepackage{graphicx}
\usepackage{booktabs} 
\usepackage{url}  
\usepackage{hyperref}
\usepackage{mathrsfs} 
\usepackage{amssymb}
\usepackage{algorithm}
\usepackage{algpseudocode}
\usepackage{natbib}

\usepackage{amsmath}
\usepackage{stackengine}
\usepackage{nccmath}
\usepackage{amsthm}
\usepackage{amssymb}
\usepackage{float}
\usepackage{caption}
\usepackage{subfig}

\usepackage{comment}

\usepackage{xcolor}

\setcitestyle{citesep={;}}

\DeclareMathOperator*{\argmax}{arg\,max}

\DeclareMathOperator{\E}{\mathbb{E}}
\def\delequal{\mathrel{\ensurestackMath{\stackon[1pt]{=}{\scriptstyle\Delta}}}}
\newtheorem{defn}{Definition}
\newtheorem{theorem}{Theorem}

\newtheorem{lemm}{Lemma}
\newtheorem{lemm2}{Lemma}
\newtheorem*{theorem*}{Theorem}

\newtheorem{assumption}{Assumption}
\newtheorem*{assumption*}{Assumption}
\newtheorem{corr}{Corollary}
\newtheorem*{corr*}{Corollary}

\jairheading{74}{2022}{1-74}{10/2021}{05/2022}
\ShortHeadings{Multi-Agent Advisor Q-Learning}
{Subramanian, Taylor, Larson, \& Crowley}
\firstpageno{1}

\begin{document}

\title{Multi-Agent Advisor Q-Learning}

\author{\name Sriram Ganapathi Subramanian \email s2ganapa@uwaterloo.ca \\
\addr University of Waterloo \\ 200 University Ave W, Waterloo, ON  N2L 3G1 \\
Vector Institute \\ 661 University Ave Suite 710, Toronto, ON M5G 1M1
       \AND
       \name Matthew E. Taylor \email matthew.e.taylor@ualberta.ca \\
       \addr University of Alberta \\ 
       116 Street and 85 Avenue, 
       Edmonton, AB  T6G 2R3  \\
       Alberta Machine Intelligence Institute (Amii) \\
       10065 Jasper Ave Edmonton, AB T5J 3B1 
       \AND
       \name Kate Larson \email kate.larson@uwaterloo.ca \\
       \addr University of Waterloo \\ 
       200 University Ave W, Waterloo, ON  N2L 3G1  
       \AND
       \name Mark Crowley \email mcrowley@uwaterloo.ca \\
       \addr University of Waterloo \\ 
       200 University Ave W, Waterloo,
       ON  N2L 3G1}


\maketitle

\begin{abstract}

In the last decade, there have been significant advances in multi-agent reinforcement learning (MARL) but there are still numerous challenges, such as high sample complexity and slow convergence to stable policies, that need to be overcome before wide-spread deployment is possible. However, many real-world environments already, in practice,  deploy sub-optimal or heuristic approaches for generating policies. An interesting question that arises is how to best use such approaches as \emph{advisors} to help improve reinforcement learning in multi-agent domains. 
In this paper, we provide a principled framework for incorporating action recommendations from online sub-optimal advisors in multi-agent settings. We describe the problem of \emph{ADvising Multiple Intelligent Reinforcement Agents} (ADMIRAL) in nonrestrictive \emph{general-sum stochastic game} environments and present two novel $Q$-learning based algorithms:  \textbf{ADMIRAL - Decision Making (ADMIRAL-DM)} and \textbf{ADMIRAL - Advisor Evaluation (ADMIRAL-AE)}, which allow us to improve learning by appropriately incorporating advice from an advisor (ADMIRAL-DM), and evaluate the effectiveness of an advisor (ADMIRAL-AE). We analyze the algorithms theoretically and provide fixed point guarantees regarding their learning in general-sum stochastic games. Furthermore, extensive experiments illustrate that these algorithms: can be used in a variety of environments, have performances that compare favourably to other related baselines, can scale to large state-action spaces, and are robust to poor advice from advisors. 

\end{abstract}

\section{Introduction}\label{sec:introduction}

Reinforcement learning (RL) research is growing and expanding rapidly, however, this method still finds only limited applications in practical real-world settings \citep{dulac2021challenges}. One major reason for this is that RL algorithms typically have high sample complexity and can learn effective policies only after experiencing millions of data samples in simulation \citep{kakade2003sample}. Multi-agent reinforcement learning (MARL) extends RL to domains where more than one agent learn simultaneously in the environment \citep{shoham2008multiagent}. Moving from single-agent to multi-agent settings introduces new challenges including non-stationary environments and the curse-of-dimensionality \citep{hernandez2019survey}, while concerns from single-agent RL such as exploration-exploitation trade-offs and sample efficiency remain~\citep{yogeswaran2012reinforcement}. In MARL environments, it has been reported that learning complex tasks from scratch is even impractical due to its poor sample complexity \citep{da2019survey}. In this regard, it becomes necessary for agents to obtain guidance from an external source to have any possibility of scaling up to real-world domains. Furthermore, during the early stages of learning, agents' policies may be quite random and dangerous, which makes it almost impossible to use them in real-world environments. Thus, it is hard to improve upon these policies by only using direct interactions with the environment. In this paper, we aim to tackle the problem of improving sample efficiency in MARL through the use of other available sources of knowledge, particularly during the early stages of training. 

In single-agent RL, the use of external knowledge sources such as \emph{advisors} to drive exploration has been successful in a variety of domains.  The advisors provide actions to the agent at different states to bootstrap learning by targeted exploration~\citep{nair2018overcoming}. However, the biases of sub-optimal advisors pose a challenge to successful learning~\citep{gao2018reinforcement}. Further, many of these approaches do not directly extend to multi-agent environments due to the additional complications present in the multi-agent environments. Although learning from external sources of knowledge has been explored in multi-agent settings, many previous works assume the presence of fully optimal experts \citep{natarajan2010multi, hadfield2016cooperative, yu2019multi}. Generally, they entail additional assumptions such as having simplified environments with only two agents \citep{lin2019multi} and consider restrictive environments such as  competitive zero-sum settings \citep{wang2018competitive} or fully cooperative settings where all agents  share a common goal \citep{natarajan2010multi, le2017coordinated, Peng2020}. Additionally, some approaches such as \cite{lin2019multi} restrict themselves to simple multi-agent environments with discrete state and action spaces. The use of sub-optimal advisors in multi-agent general-sum settings with an arbitrary number of agents has been less explored, and to the best of our knowledge, there has been no comprehensive analysis of this approach, especially from a theoretical perspective.


\subsection{Motivational Examples}\label{sec:motivationalexample}

We  describe two motivational examples relevant to the goals of our paper. These examples provide an intuition about the kind of practical problems where our approach can be used and clarify the potential impact of this line of research. The first example uses a cooperative wildfire response setting and the second example uses the competitive product marketing domain.

\textbf{Motivational Example 1: }  Wildfire response is a  complicated process that requires systematic planning of important resources and a good understanding of wildfire behaviour for proper estimation and combat \citep{Thompson2018}. The firefighters and fire managers need to make many critical decisions related to wildfire control, and on many occasions, these decisions could be the difference between life and death \citep{thompson2019risk}. Additionally, these decisions have a high ecological impact. This is a multi-agent problem (multiple fire-fighters aim to fight fire) where artificial learning agents can learn suitable policies to aid in fire-fighting efforts \citep{nikitin2019development}. Machine learning, particularly reinforcement learning, has a huge potential  in this area but has been underutilized so far \citep{Jain2020}. Notably, in these sustainability-based domains, there is a general paucity of data \citep{tymstra2020wildfire}, since obtaining good quality high-resolution sensor data is expensive and hard. Hence, current state-of-the-art MARL algorithms are incapable of being used in such problems due to poor sample efficiency. Notably, current practical fire-fighting efforts use physics-based models \citep{rothermel1972mathematical} that help in predicting the spread of fires given current location and intensity. Despite being the current state-of-the-art, these models are  sub-optimal, have low accuracies \citep{jahdi2015evaluating}, and possess other problems like under-prediction bias and lack of generalizability to regions outside North America \citep{Cruz2010}. Thus, we have particular knowledge sources that are not optimal but are still used in practice (particularly due to lack of alternatives). Our work in this paper will enable MARL training to use these physics-based models to speed up learning. The policies that the MARL algorithms will finally arrive at will have the potential to be better than these physics-based models since the MARL algorithms will also simultaneously learn from data.

\textbf{Motivational Example 2: } Multi-agent algorithms have the potential to learn from available data and formulate effective marketing and price management strategies to improve financial profit for companies \citep{ganesh2019reinforcement}. However, the problem of poor sample efficiency prevents the usage of MARL for this problem, since many companies would find it difficult to procure sufficient data for MARL training. There are many mathematical marketing models in the literature that typically help companies formulate marketing strategies \citep{Eryigit2017}, however, these models are sub-optimal, with scope for improvement, especially in adapting to changing trends \citep{storbacka2020changing}. These models can serve as external knowledge sources that MARL training can leverage to learn good policies. 

Hence,  learning from possibly sub-optimal external content sources, which we broadly refer to as ``advisors'', is useful for MARL training. We  formally introduce this problem and study it further in this paper.

\subsection{Related Work}\label{sec:relatedwork}


\textit{Imitation learning} includes various methods for learning the behaviour of advisors, 
the simplest being \textit{behaviour cloning}, where supervised learning is used to mimic the advisor policy. This method dates back to the early 90s, where agents were shown to be successful in copying the behaviour of the demonstrator in autonomous driving tasks such as road following and perception \citep{pomerleau1998autonomous, pomerleau1991efficient}. This method comes with certain theoretical guarantees, where prior works have conducted formal analysis and established that near-optimal advisors are the easiest to imitate, requiring far fewer demonstrations than sub-optimal advisors to achieve the same performance as the advisors being imitated \citep{syed2010reduction}. Behaviour cloning is hard to generalize to different unseen environments since the agent is learning in a supervised fashion \citep{kiran2021deep}. Prior works in the area of behaviour cloning have also introduced methods to detect and safeguard against a few noisy/bad demonstrations \citep{zheng2014robust, hussein2021robust}, however, in general, the demonstrations are assumed to be near-optimal to enable learning reasonable policies.  Further, it has been noticed that behaviour cloning methods are prone to a problem of \emph{distribution drift}, where the trajectory distribution at the test time drifts away from the distribution learned from the advisor during training \citep{ross2011reduction, rajeswaran2017learning}.  In autonomous driving environments, on-policy data collection has been shown to mitigate this problem to some extent \citep{santana2016learning}. Some recent approaches propose off-policy solutions, along with techniques of expanding the input (image)-action space using data-augmentation methods \citep{codevilla2018end, laskey2017dart}. However, several other limitations like dataset bias and high variance in neural network-based solutions have been reported to limit the application of behaviour cloning to real-world environments \citep{codevilla2019exploring}.

Another popular imitation learning framework is \textit{inverse reinforcement learning (IRL)} \citep{ng2000algorithms}, where the objective is to learn the reward function from demonstrations. This framework typically assumes that the environment does not have a reward function and/or it is difficult to formulate good reward functions, but expert demonstrations of good behaviour are easier to obtain. A good example is autonomous driving, where formulating a complete reward function that covers all scenarios is hard while obtaining demonstrations of good driver behaviour is much simpler.   Initial approaches to IRL used maximum margin methods where an initial estimate of the reward function for the demonstrator keeps being iteratively improved, such that the performance of the demonstrator is at least more than a ``margin'' of the previous reward estimate for the demonstrator \citep{abbeel2004apprenticeship, ratliff2006maximum}. This iteration is repeated until no such improvement is possible. The problem with the maximum margin methods are their sensitivity to noise and imperfection in the demonstrator behaviour. To alleviate this problem, probabilistic approaches using principles of maximum entropy have been proposed for the IRL framework \citep{ziebart2008maximum, boularias2011relative}. These approaches reason over a set of possible behaviours, rather than monotonically improving upon estimates of reward or policies.  Neural networks have also been considered to learn a suitable reward function, where convolutional networks aim to map the relationship between input images to final rewards \citep{wulfmeier2015maximum}. Several techniques from supervised learning such as Gaussian processes \citep{rasmussen2003gaussian} and ensemble methods \citep{dietterich2000ensemble} have also been used in the IRL paradigm \citep{levine2011nonlinear, ratliff2006boosting}. A recent approach, Generative Adversarial Imitation Learning (GAIL), aims to recover the policy of the expert directly instead of extracting an explicit reward function using Generative Adversarial Networks (GANs) \citep{ho2016generative}. Since GAIL is not learning a reward function, it may not be considered an IRL technique and since it is not learning in a supervised fashion it may not be considered a behaviour cloning technique as well. This approach opened up a new class of methods in the intersection of imitation learning and generative adversarial networks \citep{miyato2018spectral, kuefler2017imitating}. Some of these approaches aim to extract an explicit reward function from the demonstrations using GANs \citep{finn2016connection, fu2018learning} and these can be considered to fall within the IRL framework.

The DAGGER algorithm introduced by \citet{ross2011reduction} is yet another imitation learning method. This algorithm has strong guarantees of performance while learning stationary deterministic policies in environments with an online advisor that can be queried interactively for additional feedback. However, the major aim of this algorithm is to obtain a policy that guarantees ``no-regret'' under its induced distribution of states and does not aim to improve upon the provided demonstrator. This method belongs to a wider set of online algorithms that aim to provide the no-regret guarantee while learning based on demonstrations from a perfect advisor \citep{hazan2007logarithmic, cesa2004generalization, kakade2008mind}. 

In general, imitation learning methods assume fully optimal or near-optimal advisors (or \emph{experts}), with the major goal being to copy the policy or behaviour of the experts. Since MARL problems are non-stationary, there is little expectation of obtaining perfect experts. Rather, we expect guidance on action choices that aid agents in learning and improving over time. Further, several multi-agent imitation-based methods in the literature are restricted to cooperative games \citep{barrett2017making, bogert2014multi} or games with strict restrictions on the nature of the reward function (such as having linear relations to some underlying feature) \citep{reddy2012, waugh2013computational}, in addition to assuming the availability of (near) optimal experts. On the other hand, other approaches based on IRL in multi-agent settings are restricted to zero-sum  games \citep{lin2017multiagent, wang2018competitive}. Some recent approaches aim to apply IRL in general-sum games \citep{yu2019multi, song2018multi}, however the assumption of availability of perfect experts are present in these works as well. A different approach from \citet{price2003accelerating} studies imitating more experienced peers in a multi-agent setting. However, this work considers a very restrictive environment, where each agent's dynamics is independent of other agents. Further, strict assumptions on the reward function exist, such as obtaining the same numerical reward over a part of the state space and having an independent reward function that does not depend on other agents' actions. Such assumptions are hard to verify in real-world multi-agent environments. Pure imitation methods generally lack the ability to exceed the performance of the available expert/advisor.

Another approach, \textit{Learning from Demonstrations (LfD)}, combines the imitation-based behaviour cloning approach of learning from expert demonstrations and the reinforcement learning-based approach of directly learning from the environment to reach a suitable goal.  Here, the objective is not to simply perform imitation learning, but to use imitation learning as a bootstrap mechanism that can speed up the training of RL agents. The RL algorithm enables further fine-tuning of the policy learned from imitation, which provides an opportunity for improving upon the performance of the advisor and learning goal-oriented policies. 
The algorithms in this approach use the environmental rewards along with expert/advisor demonstrations collected offline. Unlike the IRL paradigm, the environmental rewards are assumed to be available, and the rewards need not be extracted from suitable expert demonstrations. Our work in this paper is most related to the LfD framework. Early works in this area studied the LfD approach using model-based reinforcement learning, which found applications in classical RL environments like cart-pole \citep{schaal1997learning} and robot arm learning to balance a pendulum \citep{atkeson1997robot}. More recently, the model-free RL approaches gained prominence, especially after the exceptional performance shown by Deep $Q$-learning (DQN) on Atari games \citep{mnih2015human}. Model-free RL using replay buffers for training have been successful in the LfD framework as well \citep{piot2014boosted, chemali2015direct}. One state-of-the-art algorithm, \textit{Deep $Q$-learning from Demonstrations (DQfD)} \citep{hester2018deep} pre-trains the agent using demonstration data, keeps this data permanently for training, and combines an imitation-based hinge loss with a temporal difference (TD) error. This additional loss helps DQfD learn faster from demonstrations, but also makes DQfD prone to the problem of overfitting to the demonstration data.  
\textit{Normalized Actor-Critic (NAC)} \citep{jing2020reinforcement} drops the imitation loss and hence is more robust to imperfect demonstrations (from bad, almost adversarial advisors) than DQfD. However, we find that the performance of DQfD is at least as good as NAC for reasonable advisors (due to the imitation loss). \citet{goecks2019integrating} introduce the Cycle-of-Learning (CoL) algorithm that provides a novel LfD mechanism in which additional human inputs can be obtained during training in environments where humans are present in the loop to help agents train faster. The agents can make use of the human feedback in addition to the demonstration from other sources for training. Notably, LfD algorithms have found numerous applications in robotics. Some early applications have focussed on teaching primitive movements of motors to policy gradient based RL algorithms that can learn to perform a suite of relatively simple robotic tasks, such as manoeuvring gaps  \citep{peters2008reinforcement, theodorou2010reinforcement}. Recently, \citet{rajeswaran2017learning} introduce an effective algorithm that can learn highly complex dexterous manipulation such as object relocation and in-hand manipulation in response to sensor inputs. They introduce an algorithm, Demonstration Augmented Policy Gradient (DAPG) that uses an on-policy policy gradient \citep{richard1999policy} update as opposed to the off-policy nature of prior approaches such as \citet{hester2018deep}. \citet{zhu2018reinforcement} provides yet another approach that uses the LfD technique for robot manipulation tasks such as block lifting, block stacking and pouring liquids, where the agents need to learn effective visuomotor policies that can take actions in response to image inputs from a camera or sensor. The well-known RL algorithm, Deep Deterministic Policy Gradients (DDPG) \citep{lillicrap2015continuous} has been adapted to the LfD framework to incorporate human demonstrations and learn continuous control object manipulation robotic tasks such as peg-insertion and clip-insertion in both real and simulated environments \citep{vecerik2017leveraging}.

While powerful, the requirement for offline demonstrations commonly seen in prior works on LfD is not a good fit for MARL. In MARL, since the environments are non-stationary with dynamic opponents, real-time action advising would be more useful as the advisors can teach agents to adapt to changing opponents. Recently, some multi-agent approaches have used the LfD method, where the algorithms can, in-theory, do without fully optimal advisors \citep{silver2016mastering, wang2018competitive, hu2018knowledge}. These works, however,  are applicable only to a restrictive set of MARL environments. The Alpha-Go approach from \citet{silver2016mastering} and the approach from \citet{wang2018competitive} are restricted to zero-sum competitive games and cannot naturally extend to general-sum games. The work by \citet{hu2018knowledge} is designed for a very particular application (StarCraft Micromanagement), where the authors require the availability of specialized human-made opponents that contain specific pre-defined tactics about game-playing. In MARL, prior works have also studied peer-to-peer teaching, where each agent can learn when and what advice needs to be extended to peers in addition to learning how to use the available advice and improve its own learning \citep{leno2017simultaneously, omidshafiei2019learning, dayong2020differential}. The agents can switch between the roles of student or teacher at different points of time based on the situation. However, as evident from the setting, this method is only applicable to fully cooperative environments.

Expert demonstrations have also been used for \textit{reward shaping} in single-agent RL \citep{laud2004theory}, but this undermines the convergence guarantees of $Q$-learning based algorithms. Using prior knowledge to define a potential function over the state space provides an approach known as potential-based reward shaping, which preserves the total order over policies and does not affect the convergence guarantees \citep{ng1999policy, wiewiora2003principled}. The approach of reward shaping has been popular in single-agent RL \citep{ofir2018belief}, and very recently adopted to the MARL setting as well \citep{devlin2011empirical, tanmay2020learning, baicen2021shaping}. The work by \cite{tanmay2020learning} uniformly redistributes the rewards  accumulated at the end of a trajectory, to each state-action pair along the length of the trajectory. This approach is based on independent learning, which hurts convergence guarantees in the MARL setting \citep{tan1997multi}. Although it shows good empirical performance in single-agent tasks, this approach performs poorly in many MARL tasks, since the credit-assignment is not always accurate \citep{baicen2021shaping}. On the other hand, the approach by \cite{baicen2021shaping} adapts potential-based reward shaping to the MARL setting. There are two important limitations of this potential-based reward shaping approach, formulated by \cite{baicen2021shaping} in MARL. The first is that the shaping advice is a heuristic that needs to come from an expert who has complete prior knowledge about the entire problem domain and is capable of designing these heuristics. Obtaining such experts for complex MARL tasks is not always possible. Second, the shaping advice is provided at the beginning and then fixed for the duration of training. However, MARL requires adaptive advising at different parts of the state based on opponent behaviour.

Another group of methods such as \emph{human agent transfer} (HAT) \citep{taylor2011integrating} aim to summarize limited offline demonstrations (from sources like humans) into decision tree-based expert rules that boost learning online. \emph{Confidence-based human-agent transfer} (CHAT) \citep{wang2017improving} improves HAT by adding confidence measurements to safeguard against bad demonstrations. 
Both these methods demonstrate good performance in a multi-agent ``Keep Away'' game, although the algorithms themselves are single-agent only. These algorithms are independent methods that consider the other agent(s) as part of the state and do not track opponent actions or perform any kind of opponent modelling. In MARL environments, changes in opponent behaviour play a crucial role in determining the agent rewards \citep{shoham2008multiagent}. Particularly in competitive environments, these agents are expected to adapt between risk-seeking and risk-averse strategies based on the nature of their opponents \citep{conitzer2007awesome}. Without tracking opponent behaviour, this adaptability is not possible, since the differences in opponent behaviour in the same state would not stimulate different responses by the learning agent. A related approach known as the \emph{Teacher-Student Framework} \citep{Torrey2013Teaching} aims at accelerating the learning process under limited communication with an advisor. Almost all works in this framework assume either fully optimal or a moderate level of expertise for the advisors  \citep{Amir2016Interactive, Zhan2016Thoeretical}.

 At their core, RL (and MARL) algorithms are fixed point iterative methods that iterate recursively until no further iterations are desirable or required \citep{littman1996reinforcement}. An RL algorithm's  ability to converge to a fixed point provides a clear picture of the goal towards which the algorithm is progressing. The fixed point defines the completion of the task of an RL agent in the given environment and is like a terminal point in the sequence of RL iterations. For example, single-agent RL methods use the optimal $Q$-value that provides the maximum expected discounted sum of rewards in the given environment as the fixed point \citep{sutton1998introduction}. In MARL the fixed point is defined by the solution concept of the game. We  note that many of the prior works referenced here do not contain a theoretical analysis of the learning algorithms that provide conditions for fixed point guarantees regarding learning in MARL environments. Since all these methods involve the presence of external sources in the learning process, it is unknown if previous guarantees in RL convergence extend to these approaches. Without such guarantees, it is unclear whether the algorithms will learn reasonable policies in any generic environment (beyond those considered in the paper) and whether the algorithms will progress towards obtaining a suitable solution for the current problem. Since there are many solutions concepts in multi-agent environments \citep{shoham2008multiagent}, the kind of solutions that are likely to be obtained by these algorithms are unclear. Some approaches establish the existence of unique solution concepts in the particular model of multi-agent environment considered, yet still lack fixed point guarantees for any RL method and theoretical guarantees of arriving at the established solution concept by any learning algorithm \citep{yu2019multi, song2018multi}.

From the above discussion, we see that there are five fundamental limitations of existing algorithms that learn from external sources of knowledge, which hinders their applicability to real-world multi-agent environments. All prior works can be seen to contain one or more of these limitations. 1) Strict assumptions on the quality of advisors, 2) algorithms designed as single-agent based independent methods that consider other agents as simply part of the state in the environment, though the actions of these agents strongly influence the rewards for the learning agent, 3) offline advising, where demonstrations are collected and used for training agents offline, which is not well-suited for MARL due to the adaptive nature of opponents, where real-time feedback is critical, 4) algorithms designed towards a restricted class of MARL environments that are not generally applicable to many other environments, and 5) lack of thorough analysis for the conditions under which theoretical fixed point guarantees can be provided. 

\subsection{Our Approach}\label{sec:ourapproach}

In this paper, we study advising in MARL under the stochastic game model \citep{shapley1953stochastic} and aim to resolve the five major limitations of prior methods discussed in Section~\ref{sec:relatedwork}. We will explore the use of advisors in multi-agent reinforcement learning (MARL) under general-sum settings, where advisors suggest (possibly sub-optimal) actions to different agents in the environment. The advisors can belong to a broad class of categories, such as pre-trained policies, rule-based systems or other systems that continue to learn and/or adapt during gameplay. 
We do not make any assumptions or place constraints on the quality or type of the advisors themselves. The advisors are assumed to be available online so that agents can get real-time feedback while training in dynamic MARL environments. We will also assume that each agent  has access to at most one advisor.  Communication between the agent and the advisor is assumed to be free. The advisor receives the state of an agent and provides an action recommendation for the current state. These action recommendations can be deterministic or stochastic. 

We introduce a principled framework for studying the problem of \textbf{ADvising Multiple Intelligent Reinforcement Agents} (ADMIRAL). We propose two $Q$-learning based algorithms~\citep{watkins1992q}. 
The first algorithm, \textbf{ADvising Multiple Intelligent Reinforcement Agents - Decision Making} (ADMIRAL-DM), learns to act in the environment using advisor-guidance, while the second,  \textbf{ADvising Multiple Intelligent Reinforcement Agents - Advisor Evaluation} (ADMIRAL-AE),
provides a principled method to evaluate the usefulness of the advisor in the current MARL context. To the best of our knowledge, we are the first to propose a method to evaluate a knowledge source before using it for learning in MARL\footnote{In their comprehensive survey of literature that aims at reusing knowledge to accelerate MARL, \citet{da2019survey} state that several prior works \citep{Amir2016Interactive, Zhan2016Thoeretical} assume (at least) a moderate level of expertise for the advisors for action advising and are only applicable to single-agent environments, in line with our discussions in Section~\ref{sec:relatedwork}. While \citet{leno2017simultaneously} relax the assumption of optimal advisors by allowing agents to advise each other (in cooperative games), they do not provide a method to evaluate the available agents/advisors before using them.}. We  empirically study the performance of our algorithms in suitable test-beds, along with a comparison to related baselines. Theoretically, we  establish conditions under which we can provide fixed point guarantees regarding the learning of our ADMIRAL algorithms in general-sum stochastic games. 


Specifically, our contributions in this work are: 1) introducing a general paradigm for learning from external advisors in MARL, 2) analyzing two important challenges in learning from advisors in MARL, 3) presenting a suitable algorithm for each of these challenges, 4) establishing conditions for appropriate fixed point guarantees in these algorithms, 5) proving that it is possible to provide convergence results under less restrictive assumptions compared to prior work, and 6) empirically showing that our algorithm can adapt and perform well in many challenges in MARL.

\section{Background}\label{sec:background}

In this section, we present the key concepts used in this paper. We start with a brief introduction to single-agent RL before describing the generalized multi-agent RL setting we use in this paper.  

\begin{defn}
A Markov decision process (MDP) is defined as $\langle \mathcal{S},A,R,T, \gamma \rangle$ where $\mathcal{S}$ is the set of states, $A$ is a set of actions, $R:S\times A\mapsto \mathbb{R}$ is the reward function, $T:S\times A \times S\mapsto [0,1]$ is the transition function and  $0\leq\gamma<1$ is the discount factor.
\end{defn}
Given an MDP, it is assumed that the agent starts in some state $s \in \mathcal{S}$, takes some action $a\in A$, and transitions to another state $s'$ with probability $T(s,a,s')$ where it collects reward $R(s,a)$.
A \emph{policy}, $\pi:\mathcal{S}\mapsto \Delta A$, specifies a probability distribution over the set of actions for each state $s\in\mathcal{S}$ (the notation $\Delta$ is used to denote a probability distribution). The value, or expected discounted sum of future rewards, of following some policy $\pi$ when starting in state $s$ is defined as $v(s,\pi)=\sum_{t=0}^\infty \gamma^t E[r_t|s_0=s, \pi]$ where $r_t$ is the reward collected at time $t$. 
The objective  of the agent is to find the optimal policy, $\pi^*$, which maximizes the expected discounted sum of future rewards at each state, $v(s, \pi^*)$. An alternative approach is to use Q-values. A  Q-value, $Q_{\pi}(s,a)$, of a state-action pair, is the expected future reward estimate that can be obtained from taking action $a$ in the state $s$ and following the policy $\pi$ from there. The optimal policy $\pi^{*}$ is obtained by $\pi^{*}(s) = \argmax_{a \in A} Q^*(s,a)$ where $ Q^*(s,a)$ is the optimal action-value function that returns the maximum $Q$-value in all the states.

In multi-agent settings, the optimal policy of an agent may depend on the policies followed by the other agents.  Generalizations of MDPs, called \emph{stochastic games}, are used to model multi-agent settings and form the basis of MARL. 

\begin{defn}
A stochastic game is defined as $\langle \mathcal{S},N,\mathbf{A},P,\mathbf{R}, \beta \rangle $ where
$\mathcal{S}$ is a finite set of states, $N$ is the finite set of agents, $|N|=n$, and $\mathbf{A}=A^1\times\ldots \times A^n$ is the set of joint actions, where $A^i$ is the finite action set of an agent $i$, and  $\boldsymbol{a}=(a^1,\ldots,a^n)\in \mathbf{A}$ is the joint action where an agent $i$ takes action $a^i\in A^i$. Furthermore, $P:S\times\mathbf{A}\times S\mapsto [0,1]$ is the transition function, $\boldsymbol{R} = \{R^1, \ldots, R^n\}$ is the set of reward functions, where  $R^i:S\times\mathbf{A}\mapsto\mathbb{R}^n$ is the reward function of the agent $i$, and $\beta$ is the discount factor ($0\leq\beta<1)$.
\end{defn}

In a stochastic game, the common assumption is that all agents share the same set of states $\mathcal{S}$ (where $\mathcal{S}$ is the state space), which contains information about all agents participating in the stochastic game \citep{shapley1953stochastic}. The environment provides the (global) state to each agent participating in the stochastic game. At each time step $t$, each agent $i$ observes the current state $s \in \mathcal{S}$ and takes a local action $a^i \in A^i$ (where $A^i$ is called as the action space of the agent $i$). Subsequently, the agent obtains a reward $r^i$ according to its reward function $R^i$, and the joint action $\boldsymbol{a}$ of all agents in the environment at state $s$. The transition function, $P$, determines the transition of the environment to the next state, $s'$.   This transition depends on the current state ($s$) and the joint action ($\boldsymbol{a}$) of all agents in the environment. Further, the transition function is fixed and satisfies the constraint, $\sum_{s'\in \mathcal{S}} P(s' |s, \boldsymbol{a}) = 1$ for all $s \in \mathcal{S} $ and $\boldsymbol{a} \in \boldsymbol{A}$. Given a stochastic game, a \emph{joint policy} is represented by $\boldsymbol{\pi}=(\pi^1,\ldots,\pi^n)$, where $\pi^i$ is the stochastic policy followed by an agent $i$. 
As in the single-agent setting, each individual agent tries to maximize their value function. However, this optimization depends on the policies of others: $
     v^i(s, \pi^1, \ldots, \pi^n) = \sum_{t=0}^{\infty} \beta^t \E (r^i_t | \pi^1, \ldots, \pi^n, s_0 = s)$.

There are several formulations of the stochastic game model. The \emph{general-sum} model is the most general formulation, where the rewards that an agent receives at any time step can be related to the rewards obtained by other agents in an arbitrary fashion. Special cases of general-sum games are \emph{zero-sum} games that restrict the sum of rewards obtained by all the agents at any time step to be zero, and identical interest \emph{coordination} games that require all agents in the environment to obtain the same numerical reward at every time step. We will use the general-sum formulation in this work.

 It is useful to think of a stochastic game as a multi-period stage game.
In a stage game, agents select an individual action and then receive some (possibly different) payoff, which depends on the joint action taken.
This can be formally defined as follows.

\begin{defn}\label{def:stagegame}
    An $n$-player stage game is defined as $(\mathbf{A}, M^1, \ldots, M^n)$, where $M^k:\mathbf{A}\mapsto \mathbb{R}$ is agent~$k$'s payoff function, specifying a payoff for agent~$k$ for each joint action $(a_1,\ldots, a_n)\in\mathbf{A}$.
\end{defn}

The main solution concept we are interested in is the \emph{Nash equilibrium} \citep{nash20167}, namely a stable point in the joint policy-space. We first formally define a Nash equilibrium in a stage game, before moving on to the generalization of this concept in stochastic games. We switch terminology slightly and will refer to agents' strategies to be consistent with the game-theoretic literature, although we can use strategy and policy interchangeably. In particular, an agent's strategy in a stage game is simply a probability distribution over actions, given the underlying state $s$. 
Let $\phi^k$ be the strategy of agent $k$ in the stage game and $\phi^{-k}$ be the product of strategies of all agents other than $k$, $\phi^{-k} \triangleq \phi^1 \cdots \phi^{k-1} \cdot \phi^{k+1} \cdots \phi^n$.

\begin{defn}\label{def:nasheqstagegame}
    A joint strategy $(\phi^1, \ldots, \phi^n)$ can be considered as a Nash equilibrium for the stage game $(\mathbf{A}, M^1, \ldots, M^n)$, for $k = 1, \ldots, n$, if 
    \begin{equation}
        \phi^k \phi^{-k} M^k \geq \hat{\phi}^k \phi^{-k} M^k, \quad \textrm{for all } \hat{\phi}^k \in \phi(A^k)
    \end{equation}
    where $\phi(A^k)$ is the set of all probability distributions over $A^k$.
\end{defn}

\noindent The term $\phi^k \phi^{-k} M^k $ is a scalar value. The product of strategies denotes the product of probabilities of taking specific actions by an agent. This is multiplied with the value of that action as denoted by $M^k$. The dimensionality of $M^k$ is equal to the action space $|A|$. 

For a stochastic game, the strategies for an agent apply to the entire time horizon of the game.

 \begin{defn}\label{defn:nasheqstochasticgame}
     In a stochastic game $\Gamma$, a Nash equilibrium is a tuple of $n$ strategies $(\pi^1_{*}, \ldots, \pi^n_{*})$, such that for all states $s \in \mathcal{S}$ and agents $i = 1,\ldots, n$,
      \begin{equation}
    \begin{array}{l}
         v^i(s, \pi^1_{*}, \ldots, \pi^{i-1}_{*}, \pi^i, \pi_{*}^{i+1}, \ldots, \pi_{*}^n)  \leq v^i(s, \pi^1_{*}, \ldots, \pi^n_{*})
    \end{array}{}
\end{equation}
for all $\pi^i \in \Pi^i$ where $\Pi^i$ is the set of strategies available to the agent $i$. 

\end{defn}
 
\noindent That is, no agent has incentive to unilaterally change their strategy in a Nash equilibrium.
Now, we define the Nash $Q$-function~\citep{hu2003nash} as follows,

 \begin{defn}
     Agent $i$'s Nash $Q$-function is defined as the sum of the agent $i$'s immediate reward and its discounted future rewards when all agents follow a joint Nash equilibrium strategy $(\pi^1_{*}, \ldots, \pi^n_{*})$ 
     \begin{equation}
     \begin{array}{l}
    Q^i_{*}(s, \boldsymbol{a}) = r^i(s, \boldsymbol{a})  +   \beta \sum_{s' \in \mathcal{S}} P(s'| s, \boldsymbol{a}) v^i(s', \pi^1_{*}, \ldots, \pi^n_{*})
        \end{array}
\end{equation}
where $r^i(s, \boldsymbol{a})$ is the immediate one-stage reward of the agent $i$ at state $s$ and the corresponding joint action $\boldsymbol{a}$, and $v^i(s', \pi^1_*, \ldots, \pi^n_*)$ is the agent $i$'s total discounted reward over infinite periods starting from $s'$, given that all agents follow the joint equilibrium strategy. 

 \end{defn}

\noindent The $Q$-values of the Nash $Q$-function are denoted as the \emph{Nash $Q$-value} in \citet{hu2003nash}. Finally, we define an approximate Nash equilibrium concept, an $\epsilon$-equilibrium, which bounds the benefits of agents' deviations from the joint Nash equilibrium strategy.
\begin{defn}\label{def:epsionadvisorq}

In a stochastic game $\Gamma$, a joint strategy $(\pi^{1}_{*'}, \ldots, \pi^{n}_{*'})$ is an  ($\epsilon$)-equilibrium if it satisfies (for all $\pi^{i'} \in \Pi^i$ and $\forall s$)
\begin{equation}
    \begin{array}{l}
        v^i(s, \pi^1_{*'}, \ldots, \pi^{i-1}_{*'}, \pi^{i'}, \pi^{i+1}_{*'}, \ldots, \pi^n_{*'}) 
      - v^i(s, \pi^1_{*'}, \ldots, \pi^n_{*'})  \leq \epsilon
    \end{array}
\end{equation}
\end{defn}

As seen by the definition of the Nash equilibrium (Definition~\ref{defn:nasheqstochasticgame}), given the strategies of the other agents, the Nash Equilibrium guarantees that any given agent is obtaining the best possible payoff. This is the best guarantee we can provide in a general-sum stochastic game setting with fully independent agents \cite{hu2003nash}, without any restrictions on the nature of the environment or the agent. Hence, we choose to use the Nash equilibrium as our solution concept.

\section{Advisor Q-Learning}\label{sec:advisorqlearning}

In this section, we introduce the problem of ADvising Multiple Intelligent Reinforcement Agents (ADMIRAL). We have a set of agents that can either take an action using their own policy or consult an advisor that provides action recommendations, given the current state,  at each time step. Each agent has access to at most one advisor. An advisor can be any external source of knowledge, such as a rule-based agent, a pre-trained policy, or any other system that continues to learn during gameplay.  The advisor is assumed to be available online with the possibility of providing instantaneous action recommendations to an agent. Furthermore, we consider a centralized training setting where agents can observe the state, the local actions, and the rewards of all other agents. Another specification is that in our setting, the advisor and agent communication is free, while the agents cannot communicate among themselves. There is no communication amongst the agents themselves since establishing reliable communication protocols amongst every single agent may be prohibitively expensive in large multi-agent environments. For example, in the case of wildfire fighting, it has been noted that communication, even if available, could be very limited since the individual agents (fire-fighters) may be very far apart \citep{phan2008cooperative}. However, all agents can receive global inputs from satellite/airborne sensors \citep{leblon2012use}. The centralized setting we consider is similar to that in prior works \citep{hu2003nash}. Subsequently, in Section~\ref{sec:nnimplementation} we will show that our method can be adapted to the popular centralized training and decentralized execution \citep{lowe2017multi} paradigm, which provides a suitable relaxation of the centralization assumption.

We study two challenges that arise when learning from advisors in MARL and provide algorithms for each problem. 
The first challenge is learning a policy with the help of an advisor. We introduce an algorithm for this challenge, which we call ADvising Multiple Intelligent Reinforcement Agents - Decision Making (ADMIRAL-DM). In this setting, each agent aims to learn a suitable policy that provides the best responses to the opponent(s) and performs effectively in the given multi-agent environment. An agent has access to a (possibly sub-optimal) advisor that could be leveraged to improve the speed of learning. Hence, at each time step, the agent could choose to follow its own policy or that of the advisor. In  early stages of learning, the dependence on the advisor is greater, and this dependence gradually declines  as the agent's policy improves. If an agent does not receive an action recommendation at some time step, they can simply use their own current policy. Hence, we do not require the advisor to be capable of providing action advice at every state in the state space. A schematic of this setting is provided in Figure~\ref{fig:ADMIRALDM}. 

\begin{figure}
    \centering
    \includegraphics[width=0.6\textwidth]{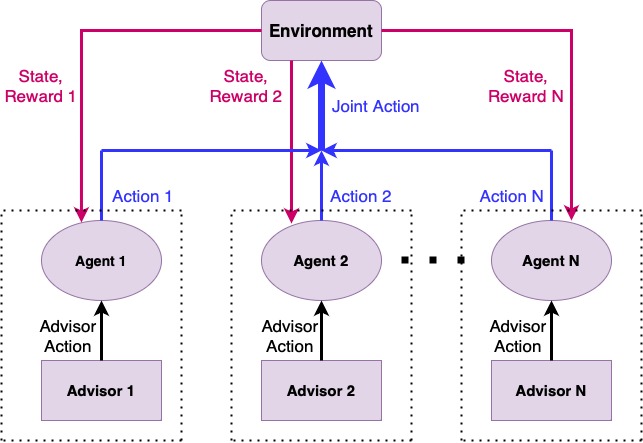}
    \caption{Architecture of the ADMIRAL-DM algorithm. Each agent has access to at most one advisor. The advisor provides action recommendations to the agent, and the agent can decide how much to rely on the advisor.}
    \label{fig:ADMIRALDM}
\end{figure}

The second challenge is the evaluation of the advisor itself. We provide an algorithm called ADvising Multiple Intelligent Reinforcement Agents - Advisor Evaluation (ADMIRAL-AE) that tackles this challenge. Before using an advisor, it is beneficial to  evaluate it  to determine whether the advisor will provide effective advice. Hence, we propose a `pre-learning' phase (i.e., a distinct phase before the beginning of training of ADMIRAL-DM) where the ADMIRAL-AE is used with the goal of getting a good understanding of the capabilities of the advisor in the current environment. We  assume that a single advisor exists in the system and this advisor could be evaluated by one or more agents. A schematic of this setting is provided in Figure~\ref{fig:ADMIRALAE}. 


\begin{figure}
    \centering
    \includegraphics[width=0.6\textwidth]{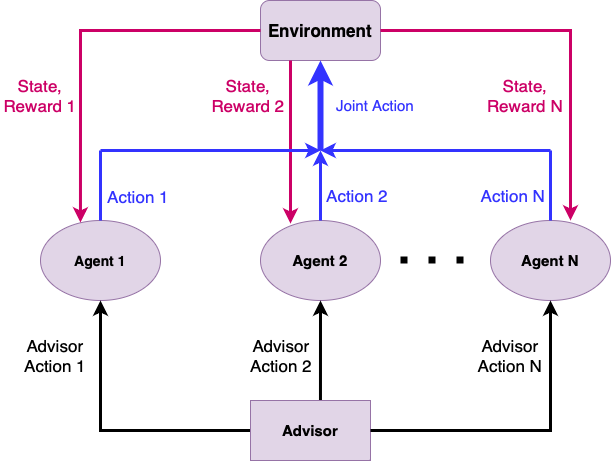}
    \caption{Architecture of the ADMIRAL-AE algorithm. All agents aim to evaluate a single advisor. The advisor can be queried by each agent to obtain action recommendations, which the agent can choose to execute.}
    \label{fig:ADMIRALAE}
\end{figure}

Many real-world multi-agent application domains may have a state-of-the-art solution method being currently used in practice. This method could be normally useful, but may not be suitable for every context and against every possible opponent in MARL environments. Hence, it is possible to face situations in which the available advisor  provides advice that is optimal or close to optimal for some states, but whose advice is poor in others.  Intuitively, to learn well, agents must listen more to the advisor when the available advisor is good and well-suited to the current context and listen less (or not at all) to the available advisor when it is bad. To make this issue more concrete, recall the 
 wildfire example discussed in Section~\ref{sec:motivationalexample}. A well-known fire simulator model is the Farsite \citep{finney1998farsite} simulator that is actively  used in practice to model the spread of fire. This fire simulator model predicts the spread of fire in the future and forms the basis of  fire-control strategies of firefighters. Notably, through extensive experimentation, it has been reported that this model does not give satisfactory performances under certain environmental conditions such as extreme downslope winds \citep{zigner2020evaluating}. However, since real fires are affected by many dynamic environmental factors at the same time, coming up with suitable heuristics for deciding the fit of the given model is almost impossible. Thus, it is important to understand when an advisor is providing useful advice and when its advice has limitations. Always relying on a poor advisor may lead to poor policies being learned or hurt the overall sample complexity of the algorithms. Thus, we recommend evaluating the advisor before using it for learning, if possible. We argue that it is important to  decouple the problem of ``advisor-evaluation'' where the objective is to study the suitability of the advisor in the given MARL environment from the problem of ``decision-making'' which aims to improve the training of MARL algorithms by making use of advisor knowledge.
 

We propose to conduct a `pre-learning' phase before training  the decision-making algorithm. In this phase we perform an ``advisor-evaluation'' study,  either in simulation or in real-world environments (particularly in environments that are not safety-critical like recommendation systems and board games such as chess) that helps to answer two important questions before beginning to learn from the advisor. 1) Does the advisor have some good knowledge of the domain that could be helpful for the MARL algorithms during training? and 2) How much should one listen to the given advisor? Performing advisor evaluation before beginning the process of learning helps the agent gain a good understanding of the advisor before learning from it, which would help in leveraging it effectively in learning a suitable decision-making policy. Most importantly, in MARL, advisors (especially good ones) could continue to learn and/or adapt online, during gameplay, based on the nature of opponents. Such advisors cannot be evaluated effectively unless their learning and evolution is captured using a principled method designed to evaluate them. If advisors are being evaluated by agents along with agents simultaneously using them for learning decision-making policies, the evaluation becomes limited and advisors are prone to be discarded quickly based on the metrics of performance and consistency at the early stages of training. During this time, the advisor could be still evolving its strategies based on the nature of the opponent(s), and hence, the evaluation is not accurate. 
For example, in the algorithm CHAT from  \citet{wang2017improving}, the confidence of an agent on a demonstrator is determined based on the demonstrator's consistency in action recommendations for the same state. This approach works well in the single-agent context. However, in MARL, due to the adaptive nature of opponents, good advisors could adapt based on the opponent and possibly evolve mixed (stochastic) strategies that will provide different actions at the same state. An approach such as CHAT would reduce the confidence of such an advisor, but this is not accurate since the advisor is good and should be leveraged more for better performance. 

\subsection{Decision-Making Using Advisor}

Our first algorithm learns to  act in the environment by leveraging the available advisor. 



\begin{algorithm}[t]
  \caption{ADvising Multiple Intelligent Reinforcement Agents - Decision Making (ADMIRAL-DM)}
  \label{alg:advisorQ2}
\begin{algorithmic}[1]
  \State Set $t=0$, get the initial state $s_0$. Let the learning agent be indexed by $j$
  \State For all $j$, $s \in \mathcal{S}$, and $a^j \in A^j$, let $Q^j_0(s, \boldsymbol{a})=0$, where $\boldsymbol{a} = (a^1, \ldots, a^n)$. Initialize a value for hyperparameters $\epsilon$ and $\epsilon'$ (i.e. value for $\epsilon_0$ and $\epsilon'_0$)
  \State Define policy derived from $Q$ to return a random action with probability $\epsilon_t$, advisor suggested action with probability $\epsilon'_t$ and greedy action with probability $1 - \epsilon_t - \epsilon'_t$ 
  \State Choose $a^j_0$ at state $s_0$ for each $j$  using policy derived from $Q$
  \While{$Q^j$ is not constant for each $j \in \{1, \ldots, n\}$}
  \State For each agent $j$, execute the action $a^j_t$ and observe $r_t^1, \ldots, r^n_t$; $a^1_t, \ldots, a^n_t$; and $s_{t+1} = s'$
  \State For each $j$, choose the next greedy action for all other agents from the copy of their respective
  \Statex \hskip1.5em  policies using $s'$. The next greedy actions are chosen using the current observed actions of 
  \Statex \hskip1.5em other agents
  
  \State For each $j$, let $u$ be a uniform random number between 0 and 1 
  
  \If{$u$ < $\epsilon'_t$}
  
    \State Obtain next action $a^{j}_{t+1} = a^{j'}$ from the advisor (using state $s'$)

    \ElsIf{$u$ > $\epsilon'_t$ and $u$ < $\epsilon_t$}
    
    \State Set the next action $a^{j}_{t+1} = a^{j'}$ as a random action from the action space $A^j$
    
    \Else
    
    \State Choose a greedy action $a^{j}_{t+1} = a^{j'}$ from the $Q$-function using $s'$ and the next greedy  
    \Statex \hskip3.0em actions of other agents
    
    \EndIf

  \State Update $Q^j_t$ for each $j \in \{1, \ldots, n\}$ using Eq.~\ref{Eq:qupdate}
  
  \State Let $t:=t+1$
  \State For each agent $j$, set the current action $a^j = a^{j'}$ and current state $s = s'$
  \State At the end of each episode, linearly decay $\epsilon_t$ and $\epsilon'_t$
\EndWhile
   
\end{algorithmic}
\end{algorithm}

ADMIRAL-DM is described in  Algorithm~\ref{alg:advisorQ2}. First, for simplicity, we  assume that all the agents in the environment use the same algorithmic steps for learning, as done in prior work \citep{hu2003nash}. Subsequently, the same algorithm can be used in other scenarios where different agents use different algorithms for learning as well. Further, as in \citet{hu2003nash}, we assume that all agents maintain a copy of the $Q$-updates of the other agents. This is possible since, during training, agents are in a centralized setting and can observe the local actions and rewards of all other agents at each time step. This helps in predicting the actions of opponents needed for providing the best responses. 

A learning agent (represented by $j$) starts with an arbitrary initialization of its  $Q$-value $Q^j_0(s, \boldsymbol{a})$. One such assignment could be to set $Q^j_0(s, \boldsymbol{a}) = 0$, for all agents $j$, all states $s \in \mathcal{S}$ and actions $a^1 \in A^1, \ldots, a^n \in A^n$. Recall, in this setting, each agent has access to an online advisor that it could query during learning. Whenever the agent needs to choose an action, it does so based on its current $Q$-value, the advisor's recommendation, or simply a random action, as the case may be (lines~8--15). The dependence on the advisor's recommendation and the random exploration is captured by two hyperparameters, $\epsilon'_t$ and $\epsilon_t$, respectively. This action is subsequently executed and the actions and rewards of the other agents are observed, including the next state $s'$ (line~6). During training, at each time step, the agent picks the possible next actions of other agents in line~7 using its copy of other agents $Q$-values. Then, in lines~8--15, the agent $j$ picks its next action based on ADMIRAL-DM algorithm's policy which chooses a random action and an advisor action with diminishing probabilities, and a greedy action with increasing probabilities, such that it becomes greedy in the limit with infinite exploration (GLIE). Thus, the agent is guaranteed to train without any further advisor influence after some finite time $t$ in the training process. Accordingly, the dependence on the advisor's recommendation is decayed linearly (line~19). In this process, the dependence of an agent is more on the advisor during the earlier stages of learning, when its own policy is quite bad. This dependence gradually reduces as its own policy improves. The $Q$-values are updated (line~16) following an update scheme given by, 
\begin{equation}\label{Eq:qupdate}
    \begin{array}{l}
         Q^j_{t+1}(s,\boldsymbol{a})  =  (1-\alpha_t)Q_t^j(s,\boldsymbol{a}) +
         \alpha_t[r_t^j + \beta Q^j_{t}(s', \boldsymbol{a'})]
    \end{array}
\end{equation}
where $\boldsymbol{a} = ( a^{1}, \ldots, a^{n}) $ denotes the actions for all agents at state $s$ and $ \boldsymbol{a'} = (a^{1'}, \ldots, a^{n'})$ denotes the actions for all the agents at state $s'$. $\beta$ denotes the discount factor, and $\alpha_t \in (0,1)$ is the learning rate. The other variables have the usual meaning described in Section~\ref{sec:background}. The algorithm's steps are repeated continuously until either the $Q$-values fully converge or come within a small threshold of convergence, as is commonly done in practice \citep{sutton1998introduction}. 



The ADMIRAL-DM algorithm's time and space complexity can be compared to the NashQ algorithm from \citet{hu2003nash}. At each time step, a learning agent $j$ needs to update $(Q^1, \ldots, Q^n)$, for all states $s \in \mathcal{S}$, and all actions $a^1 \in A^1, \ldots, a^n \in A^n$. Let the total number of states in the environment be represented by $|\mathcal{S}|$, and $|A^j|$ be the total number of actions in the action space of the agent $j$. Further, assuming that $|A^1| = \cdots = | A^n| = |A|$, we get the total number of entries in $Q^j$ to be $|\mathcal{S}||A|^n$. If the learning agent needs to update a total of $n$ $Q$-tables, then the space complexity can be given by $n|\mathcal{S}||A|^n$.  Thus, regarding the space complexity, a tabular version of ADMIRAL-DM is linear in the number of states, polynomial in the number of actions, and exponential in the number of agents, which is the same as the guarantees for the NashQ algorithm in \citet{hu2003nash}. However, in the case of time complexity, ADMIRAL-DM has the same guarantees as given for the space complexity, since in the worst case, each entry in the table needs to be queried before updating a $Q$-value. Note that this is better than the algorithm by \citet{hu2003nash}, which had exponential time complexity in the states and actions, even in the case of a two-player game. This is because the NashQ algorithm has a requirement of determining the Nash equilibrium at each stage game, which has exponential time complexity even for two-player games \citep{neumann1928theorie}. We do not have this requirement.

\subsection{Evaluation Of Advisors}\label{sec:offpolicy}

The second challenge is that of evaluating the advisor to determine its nature and its suitability for the given context. As described in the previous sub-section, the ADMIRAL-DM uses an advisor if one exists. In this sub-section, we provide an algorithm (ADMIRAL-AE) that evaluates a potential advisor and helps guide the configuration of Algorithm~\ref{alg:advisorQ2} by setting the initial value of $\epsilon'$. The objective is to make an agent following ADMIRAL-DM listen more to good advisors and listen less (or not at all) to bad advisors. The ADMIRAL-AE is used in the `pre-learning' phase discussed earlier, where agent(s) are evaluating the advisor in the context of the given environment and opponents.    



We start with a definition of an advisor strategy.

\begin{defn}\label{defn:advisorstrategy}
In a stage game, $(\boldsymbol{A}, M^1, \ldots ,M^n)$ an advisor strategy (or advisor solution) is a tuple of $n$ strategies $(\sigma^1, \ldots, \sigma^n)$, an advisor specifies for all $n$ agents.

\end{defn}


Since the advisor is defined to be a general model that receives a state and provides action recommendations to an agent in the multi-agent environment, the advisor, in general, is capable of predicting the actions of other agents as well. All the advisor's predictions and/or recommendations towards every agent in the environment constitute the advisor solution as in Definition~\ref{defn:advisorstrategy}. Here, we do not restrict our setting to environments where the advisor can predict the actions of all agents. In practice, it is possible to encounter situations where the advisor cannot predict or recommend actions to some agents. In this case, these agents can get random or placeholder strategies in the advisor solution formulated in Definition~\ref{defn:advisorstrategy}. 

Similar to our decision-making setting, we first provide an algorithm that will have all the agents in the environment using the same algorithmic steps. In each state $s$, and time $t$ during training, we form a stage game $(Q^1_t(s), \ldots, Q^n_t(s))$ using the $Q$-values of all agents. Here, the notation $Q^j_t(s)=(Q^j_t(s,a^1),\ldots, Q^j_t(s,a^n))$. The advisor receives the state and provides its predictions/recommendation for each agent, which will form the advisor solution $(\sigma_t^1(s), \ldots, \sigma_t^n(s))$ for the stage game $(Q^1_t(s), \ldots, Q^n_t(s))$. 
In the stochastic game, having access to the state and the advisor allows an agent to have access to the full advisor solution at every state $s \in \mathcal{S}$ for all time $t$.

\begin{algorithm}[t]
   \caption{ADvising Multiple Intelligent Reinforcement Agents - Advisor Evaluation (ADMIRAL-AE)}
   \label{alg:advisorQ}
\begin{algorithmic}[1]
  \State Set $t=0$ and get the initial state $s_0$. Let the learning agent be indexed by $j$ 
  \State For all $j$, $s \in \mathcal{S}$, and $a^j \in A^j$, let $Q^j_0(s, \boldsymbol{a})=0$ where $\boldsymbol{a} = ( a^1, \ldots, a^n )$
  
  \State Set the value of hyperparameters $\eta$ and $\eta'$
  
  \While{$Q^j$ is not constant for each $j \in \{1, \ldots, n\}$}
  
  \State For each $j$, let $u$ be a uniform random number between 0 and 1

  \If{$u$ < $\eta'$}
  
    \State Obtain action $a^{j}_t$ for the current state $s_t$ from the advisor

    \ElsIf{$u$ > $\eta'$ and $u$ < $\eta$}
    
    \State Set the action $a^{j}_t$ as a random action from the action space $A^j$
    
    \Else
    
    \State Choose a greedy action $a^{j}_t$ from the $Q$-function using $s_t$ and the observed previous  
     \Statex \hskip3.0em actions of other agents
    
    \EndIf
  
  \State Execute $a^j_t$ and then observe $ a^1_t, \ldots, a^n_t$;  $r^j_t$; and $s_{t+1} = s'$ for each $j \in \{1, \ldots, n\}$
  \State Obtain the advisor solution from the advisor for state $s'$
  \State Update $Q^j_t$ for each $j \in \{1,\ldots,n\}$ using Eq.~\ref{eq:updateoffpolicyQ}
 \State Let $t:=t+1$ and current state $s_t = s_{t+1}$
\EndWhile
   
\end{algorithmic}
\end{algorithm}

Algorithm~\ref{alg:advisorQ} describes our ADMIRAL-AE algorithm. A  learning agent $j$ starts with an arbitrary value of $Q^j_0(s, \boldsymbol{a})$, which represents the value at the initial time step $t=0$. We define an action selection scheme (lines~5--12) that chooses to directly use the advisor's recommendation with probability $\eta'$, a random action with probability $\eta$ and an action that maximizes the $Q$-values at the current state with probability $1- \eta- \eta'$. The idea is to mix between directly following the advisor at the current time step and choosing an action that maximizes the value of following the advisor at later stages for the action selection. We also perform a small percentage of random actions to ensure sufficient exploration of the environment. At each time $t$, the agent $j$ observes the current state $s$, and takes a local action $a^j$ (line~13) and observes the action of all agents (including itself), the reward it obtains and the new state $s'$. Note that, unlike the decision-making setting, here we do not require all agents to maintain copies of the updates of other agents and hence the rewards of other agents are not required to be observed by the agent $j$. The agent then obtains the advisor solution (Definition~\ref{defn:advisorstrategy}) from the advisor for the next state $s'$ (line~14). Subsequently, each agent $j$ updates its $Q$-value as follows (using $\beta \in [0,1)$, 
 as the discount factor): 
\begin{equation}\label{eq:updateoffpolicyQ}
    \begin{array}{l}
         Q^j_{t+1}(s,\boldsymbol{a})  = (1-\alpha_t)Q_t^j(s,\boldsymbol{a}) +  
         \alpha_t[r_t^j + \beta AdvisorQ_t^j(s')]
    \end{array}
\end{equation}
where $\alpha_t \in (0,1)$ is the learning rate. The term 
$AdvisorQ_t^j(s')$, is the total value (payoff) that the agent $j$ will obtain at the state $s'$ when all agents (including itself) play the advisor solution. This is calculated as $AdvisorQ_t^j(s') = \sigma^1_t(s') \cdots \sigma^n_t(s') \cdot Q^j_t(s')$, where $(\sigma^1_t(s'), \ldots, \sigma^n_t(s'))$ denotes the advisor solution at state $s'$ and time $t$. This can be seen as a solution to the stage game $(Q^1_t(s'), \ldots, Q^n_t(s'))$, since the value of each agent's payoff at state $s'$ is reflected in their corresponding $Q$-values at state $s'$. Since the advisor recommendation to each agent can be a stochastic sample, the $\sigma^j_t(s')$ is interpreted as a vector that contains the probability of taking each action in the action space of $j$. Similarly, from the $Q$-function of the agent $j$, we can obtain $Q^j_t(s')$, which consists of the $Q$-value of taking each action in the action space of $j$. Hence, $AdvisorQ^j_t(s')$ is a scalar value obtained using a component-wise multiplication of the advisor solution and the $Q$-values. 

The $Q$-values of all agents are updated using the advisor strategies at each iteration, as given in line~15 of Algorithm~\ref{alg:advisorQ}. The above-described steps continue until convergence, or until the $Q$-values come within a small threshold of convergence, as in ADMIRAL-DM.

Recall that the primary purpose of this algorithm is advisor evaluation. After implementing ADMIRAL-AE using the given advisor in the `pre-learning' phase, the performance of the algorithm can be used to determine $\epsilon'_0$ (hyperparameter of ADMIRAL-DM) in different ways. We provide one heuristic in this paper. We propose that the performance of ADMIRAL-AE (in terms of cumulative rewards) be compared against the maximum possible performance of any algorithm (maximum cumulative rewards). The ADMIRAL-AE's performance using the given advisor should lie between the maximum possible performance in that environment and the performance of ADMIRAL-AE using a random advisor. This can then be normalized in the range of $[0,1]$ to determine a value for $\epsilon'_0$. This normalization is given in Eq.~\ref{eq:normalize}. 

\begin{equation}\label{eq:normalize}
    \epsilon'_0 = \frac{CR - RCR}{MCR - RCR}
\end{equation}

\noindent where $CR$ denotes the cumulative reward obtained by ADMIRAL-AE using the advisor (averaged across multiple trials), $RCR$ denotes the cumulative reward obtained by ADMIRAL-AE using a random advisor (averaged across multiple trials), $MCR$ denotes the maximum possible cumulative reward in the given environment \footnote{We clarify that in this method if the RCR and MCR are not very “tight,” the difference between almost similar performing advisors could become small. However, this will not be a problem for our method, since the ADMIRAL-DM can also learn directly through environmental interactions.}. To be more accurate, a correction can be applied to $MCR$ to compensate for the loss in performance from random exploration in ADMIRAL-AE (hyperparameter~$\eta$). 

After obtaining the value of $\epsilon'_0$ from ADMIRAL-AE, this hyperparameter is used in the training of ADMIRAL-DM where its value is linearly decayed in line with the steps in Algorithm~\ref{alg:advisorQ2}. We experimentally illustrate these steps later in Section~\ref{sec:gridworldappendix}. A more elaborate study demonstrating the effectiveness of this technique is given in Section~\ref{sec:pommerman}.

Another way of using ADMIRAL-AE is to study the effectiveness of the available advisor against simulated or baseline opponents (Section~\ref{sec:experimentswithmaeqlee}). An important advantage of ADMIRAL-AE is in situations of adapting advisors, as discussed earlier. An experimental illustration of this advantage is given in Appendix~\ref{sec:adaptive}. The ADMIRAL-AE algorithm is off-policy, as the update policy (line~15) and policy followed (lines~5--12) are different. Due to this off-policy nature of ADMIRAL-AE, we do not require an agent to follow the advisor at every state in this setting. The evaluation is happening through the $Q$-values, while the action selection policy can be independent of the policy being updated as in any off-policy algorithm. The convergence guarantees in off-policy methods do not require using specific action selection policies as long as sufficient exploration is guaranteed \citep{jaakkola1994convergence, sutton1998introduction}. 

We would like to clarify that our main algorithm, ADMIRAL-DM, uses the advisor in a fully-online fashion and does not require any `pre-learning' for implementation. ADMIRAL-AE is a principled method that helps in determining how much to listen to an advisor (through the hyperparameter $\epsilon'_0$) in an optional `pre-learning' phase. If the `pre-learning' phase is not conducted, an approximate value for $\epsilon'_0$ can still be obtained using experimental heuristics. Using such an approximate value is not a problem since ADMIRAL-DM is also capable of learning from environmental rewards (through direct environmental interactions), in addition to the advisor. Here the advisor only aims to accelerate the process of training. Hence, the use of ADMIRAL-AE does not violate our contribution of relaxing the offline limitation in prior methods. 

In a tabular implementation of the ADMIRAL-AE algorithm, both space and time complexities will be linear in the number of states, polynomial in the number of actions, and exponential in the number of agents, same as that described for ADMIRAL-DM. However, the space and time complexity of ADMIRAL-AE can be represented by $|\mathcal{S}||A|^n$. Here, notice that the complexity does not have the product term of the number of agents $n$, unlike the requirement for ADMIRAL-DM. This is due to the fact that ADMIRAL-AE does not have the requirement of each agent maintaining copies of the $Q$-values of other agents as  in ADMIRAL-DM. As described in the previous sub-section, this time complexity of ADMIRAL-AE is much better than that of NashQ \citep{hu2003nash}. 


\subsection{Illustrative Example For Algorithm~\ref{alg:advisorQ}}\label{sec:algorithmexample}

In this sub-section, we show the calculations of some steps in the $Q$-updates of Algorithm~\ref{alg:advisorQ} for a single state system, to serve as a demonstration of this algorithm. Our objective is to provide a practical illustration of the various steps involved since, to the best of our knowledge, pre-evaluation of advisors have not been considered before. Since Algorithm~\ref{alg:advisorQ2} has similarities to the well-known $Q$-learning algorithm \citep{watkins1992q}, we omit a demonstrative example for Algorithm~\ref{alg:advisorQ2}.

Let us consider a two agent game with all the initial $Q$-values set to 0. The first agent (column agent) can perform two actions ``Up'' and ``Down'' and the second agent (row agent) can also perform two actions ``Left'' and ``Right''. The system has only one state, $s_1$. Let the learning rate $\alpha$ be 0.9 and discount factor $\beta$ be 0.9. Let us assign the hyper-parameters $\eta = 0.05$ and $\eta' = 0.45$. At the initial state, at time $t=0$, the stage game constructed from the $Q$-values of both the agents is given in Table~\ref{tab:initials1}.

\begin{table}
\subfloat[Initial stage game at time $t=0$\label{tab:initials1}]{
\begin{tabular}{||c | c | c||} 
 \hline
 ($Q^{1}_0$, $Q^{2}_0$) & Left & Right \\ [0.5ex] 
 \hline\hline
 Up & (0,0) & (0,0) \\ 
 \hline
 Down & (0,0) & (0,0) \\ [1ex] 
 \hline
\end{tabular}
}
\subfloat[Stage game at time $t = 1$ \label{tab:initials2}]{
\begin{tabular}{||c | c | c||} 
 \hline
 ($Q^{1}_1$, $Q^{2}_1$) & Left & Right \\ [0.5ex] 
 \hline\hline
 Up & (1.8, 1.8) & (0,0) \\ 
 \hline
 Down & (0,0) & (0,0) \\ [1ex] 
 \hline
\end{tabular}
}
\subfloat[Stage game at time $t=2$\label{tab:initials3}]{
\begin{tabular}{||c | c | c||} 
 \hline
 ($Q^{1}_2$, $Q^{2}_2$) & Left & Right \\ [0.5ex] 
 \hline\hline
 Up & (2.34, 2.34) & (0,0) \\ 
 \hline
 Down & (0,0) & (0,0) \\ [1ex] 
 \hline
\end{tabular}
}
\caption{Stage games constructed in the example.}
\end{table}



At state $s_1$, in time $t=0$, let us assume that both agents decide to use the advisor recommended actions ``Up'' and ``Left''  respectively. They execute these actions and obtain a reward of 2 each. Now the agents receive the next state, which is $s_1$ (single state system). Let the advisor solution at this state for the column agent be $\sigma^{1}_0(s_1) = (1, 0)$ and that for the row agent be $\sigma^{2}_0(s_1) = (1, 0)$. The first value of the tuple in this notation is the probability of taking the first action, and the second value of the tuple is the probability of the second action for the respective agents. This means that the advisor recommends both the agents to perform the ``Up'' and ``Left'' actions respectively, with probability 1, and the other action with probability 0. The Q update will be as follows: 
\begin{equation}
    \begin{array}{l}
Q^{1}_1(s_1, \textrm{Up}, \textrm{Left}) =  Q^{1}_0(s_1, \textrm{Up}, \textrm{Left})   +  \alpha \Big(r^{1}_0 + \beta AdvisorQ^{1}_0(s_1) - Q^{1}_0(s_1, \textrm{Up}, \textrm{Left}) \Big)
\\ \\ 
Q^{1}_1(s_1, \textrm{Up}, \textrm{Left}) = Q^{1}_0(s_1, \textrm{Up}, \textrm{Left})  +  \alpha \Big(r^{1}_0 + \beta \sigma^1_0(s_1) \cdot \sigma^2_0(s_1) \cdot Q^1_0(s_1) - Q^{1}_0(s_1, \textrm{Up}, \textrm{Left}) \Big)
\\ \\
Q^{1}_1(s_1, \textrm{Up},\textrm{Left}) =  0 + 0.9 \Big(2 + 0.9 \times 0 - 0 \Big)
\\ \\ 
Q^{1}_1(s_1, \textrm{Up}, \textrm{Left}) =  1.8
    \end{array}
\end{equation}

The superscript for $Q$ represents the agent index (1 represents the column player). The above equation also holds for the row agent and the new stage game with $Q$-values at state $s1$, in time $t = 1$ is given in Table~\ref{tab:initials2}.

Now, at state $s_1$ and time $t=1$, let us assume that the agents decide to perform the best actions from their respective $Q$-values, i.e., ``Up'' and ``Left'' again. The actions are executed, and the agents obtain a reward of 2 each. The next state is observed, and let the advisor solution here be $\sigma^{1}_1(s_1) = (0.5, 0.5)$ and $\sigma^{2}_1(s_1) = (0.5, 0.5)$. This means that the advisor assigns a probability of 0.5 for each of the actions for both the agents. Again, we calculate the $Q$-update for the column agent, 
\begin{equation}
    \begin{array}{l}
Q^{1}_2(s_1, \textrm{Up}, \textrm{Left}) =  Q^{1}_1(s_1, \textrm{Up}, \textrm{Left}) + \alpha \Big(r^{1}_1 + \beta AdvisorQ^{1}_1(s_1) - Q^{1}_1(s_1, \textrm{Up},\textrm{Left}) \Big)
\\ \\
Q^{1}_2(s_1, \textrm{Up}, \textrm{Left}) =  Q^{1}_1(s_1, \textrm{Up}, Left)  + \alpha \Big(r^{1}_1 + \beta \sigma^1_1(s_1) \cdot \sigma^2_1(s_1) \cdot Q^1_1(s_1) - Q^{1}_1(s_1, \textrm{Up}, \textrm{Left}) \Big)
\\ \\ 
Q^{1}_2(s_1, \textrm{Up}, \textrm{Left}) 
\\ =  1.8 + 0.9 \Big(2 + 0.9 \big( (0.5)^2 \times 1.8  
+  (0.5)^2 \times 0 +  (0.5)^2 \times 0 +  (0.5)^2 \times 0 \big) - 1.8 \Big)
\\ \\
Q^{1}_2(s_1, \textrm{Up}, \textrm{Left}) =  2.34
    \end{array}
\end{equation}

The above equation also holds for the row agent and the new stage game at $t = 2$ is given in Table~\ref{tab:initials3}. Similarly, the $Q$-values continue to be updated until convergence or until the values come to a small threshold of convergence.   

In the above steps, we have demonstrated the $Q$-values of different actions at the given state based on the advisor solutions. Such a process, at convergence, will lead to the agent(s) evaluating the advisor obtain a value for the advisor at each state in the state space of
the environment.

\subsection{Neural Network And Actor-Critic Implementations}\label{sec:nnimplementation}
It is well known that tabular algorithms are not applicable for environments with large state and action spaces in RL \citep{mnih2016asynchronous}. All our algorithms can be extended to large state-action  environments using function approximations as is common in RL \citep{mnih2015human, mnih2016asynchronous}, where neural networks serve as the function approximators. In this section, we give a neural network-based implementation of ADMIRAL-DM and ADMIRAL-AE.

We incorporate techniques introduced in the well-known Deep $Q$-learning (DQN) algorithm \citep{mnih2015human} with the update rule in ADMIRAL-DM to obtain its neural network implementation in Algorithm~\ref{alg:maeqldm}.  To highlight differences from the tabular implementation, we  make note of some parts of Algorithm~\ref{alg:maeqldm}. All agents maintain two networks (evaluation and target) throughout the training process. Both the evaluation and target networks start with the same configuration. The evaluation network is updated periodically at every training step and used for action selection at each step. The target network provides the target value for calculating the Bellman errors during training and is updated less frequently compared to the evaluation network. In line~18 of Algorithm~\ref{alg:maeqldm}, all agents store the experience tuples in their respective replay buffers. After every episode in lines~21 -- 25, the evaluation networks are trained using the temporal difference (TD) errors as the loss function. The TD target is obtained from the Bellman equation given in Eq.~\ref{Eq:qupdate}. After every finite number of training steps, the target network parameters are updated by copying over values from the evaluation network in line~26 as previously performed in \cite{mnih2015human}. 


\begin{algorithm}
   \caption{ADMIRAL-DM Neural Network Implementation}
   \label{alg:maeqldm}
\begin{algorithmic}[1]
   \State Initialize $Q_{\phi^j}, Q_{\pi^j}$, where $\phi$ represent the evaluation (eval) net and $\pi$ represents the target net, for all $j \in \{1, \ldots, n\}$. Initialize a value for hyperparameters $\epsilon$ and $\epsilon'$ (i.e. value for $\epsilon_0$ and $\epsilon'_0$)
  \State  At $t=0$, get the initial state $s_0$ from the environment
  \State Let the learning agent be indexed by $j$ 
  \State For all $s \in \mathcal{S}$, $a^j \in A^j$, and $j \in \{1,\ldots,n\}$, let $Q^j_t(s, a^1, \ldots, a^n)=0$
  \State Define policy derived from $Q$ to return the random action $a$ with probability $\epsilon_t$, advisor suggested action with probability $\epsilon'_t$ and greedy action with probability $1 - \epsilon_t - \epsilon'_t$
  \State Choose action $a^j_0$ for $s_0$ using policy derived from $Q$ for each $j \in \{1, \ldots, n\}$
  \While{training is not finished}
    \State For each agent $j$, execute the action $a^j_t$ and observe $r_t^1, \ldots, r^n_t$; $a^1_t, \ldots, a^n_t$; and $s_{t+1} = s'$
   \State For each $j$, choose the next greedy action for all other agents from the copy of their respective
  \Statex \hskip1.5em  policies using $s'$. The next greedy actions are chosen using the current observed actions of 
  \Statex \hskip1.5em other agents
  
   \State For each $j$, let $u$ be a uniform random number between 0 and 1 
  
  \If{$u$ < $\epsilon'_t$}

    \State Obtain next action $a^{j}_{t+1} = a^{j'}$ from the advisor (using state $s'$)

    \ElsIf{$u$ > $\epsilon'_t$ and $u$ < $\epsilon_t$}
    
    \State Set the next action $a^{j}_{t+1} = a^{j'}$ as a random action from the action space $A^j$
    
    \Else
    
    \State Choose a greedy action $a^{j}_{t+1} = a^{j'}$ from $Q_{\phi^j}$ using $s'$ and the next greedy actions of 
    \Statex \hskip3em other agents  
    
    \EndIf

  \State Store $\langle s,\boldsymbol{a}, \boldsymbol{r}, s',\boldsymbol{a'} \rangle$ in replay buffer $\mathscr{D}$, where 
  \Statex \hskip1.5em $s = s_t$, $\boldsymbol{a} = {a^1_t, \ldots, a^n_t}$; $\boldsymbol{r} = {r^1_t, \ldots, r^n_t}$ and $\boldsymbol{a'} = {a^{1}_{t+1}, \ldots, a^{n}_{t+1}}$; for each agent $j$
  \State Set the current action $a^j_t = a^{j}_{t+1}$ and the current state $s_t = s_{t+1}$; for each agent $j$
  \State At the end of each episode linearly decay $\epsilon'_t$ and $\epsilon_t$
  \While{$j$ = 1 to $n$} 
    \State Sample a minibatch of $K$ experiences $ \langle s,\boldsymbol{a}, \boldsymbol{r}, s', \boldsymbol{a'} \rangle$ from $\mathscr{D}$
    \State Set $y^j = r^j + \beta Q_{\pi^j}(s',\boldsymbol{a'})$ according to Eq. \ref{Eq:qupdate}
    \State Update the $Q$ eval network by minimizing the loss $L(\phi^j) = \frac{1}{K} \sum (y^j - Q_{\phi^j}(s,\boldsymbol{a}))^2$
  \EndWhile
  \State Update the parameters of the target network for each agent by copying over the eval network 
  \Statex \hskip1.5em every $\mathcal{T}$ steps: $\pi^j \xleftarrow{} \phi^j $
  
\EndWhile
   
\end{algorithmic}
\end{algorithm}

Similar to our approach with ADMIRAL-DM, we incorporate the techniques of DQN \citep{mnih2015human} with the update rule in ADMIRAL-AE to obtain Algorithm \ref{alg:maeqlae}. The important changes are similar to our discussion with the ADMIRAL-DM case, where the algorithm uses two networks and a replay buffer for training. The target for the loss function is obtained from Eq.~\ref{eq:updateoffpolicyQ} (line~21 in Algorithm~\ref{alg:maeqlae}). If more exploration is desired, $\epsilon$ need not be decayed in these implementations.


\begin{algorithm}[t]
\fontsize{11pt}{11pt}\selectfont
   \caption{ADMIRAL-AE  Neural Network Implementation}
   \label{alg:maeqlae}
\begin{algorithmic}[1]
   \State Initialize $Q_{\phi^j}, Q_{\pi^j}$, where $\phi$ represent the evaluation (eval) net and $\pi$ represents the target net, for all $j \in \{1, \ldots, n\}$. Initialize the hyperparameters $\eta$ and $\eta'$
  \State At $t=0$, get the initial state $s_0$ from the environment
  \State Let the learning agent be indexed by $j$ 
  \State For all $s \in \mathcal{S}$, $a^j \in A^j$, and $j \in \{1,\ldots,n\}$, let $Q^j_t(s, a^1, \ldots, a^n)=0$
  \State Set the value of hyperparameters $\eta$ and $\eta'$
  \While{training is not finished}
   \State For each $j$, let $u$ be a uniform random number between 0 and 1

  \If{$u$ < $\eta'$}
  
    \State Obtain action $a^{j}_t$ for the current state $s_t$ from the advisor

    \ElsIf{$u$ > $\eta'$ and $u$ < $\eta$}
    
    \State Set the action $a^{j}_t$ as a random action from the action space $A^j$
    
    \Else
    
    \State Choose a greedy action $a^{j}_t$ from $Q_{\phi^j}$ using $s_t$ and the observed previous actions of other agents
    
    \EndIf
  \State For each agent $j$, execute the action $a^j_t$ and observe $r_t^j; a^1_t, \ldots, a^n_t$, and $s_{t+1} = s'$
  \State For each agent $j$, obtain the advisor action  $a^{1}_{e,t+1}, \ldots, a^{n}_{e,t+1}$ for all agents at state $s'$
  \State For each agent $j$, store $\langle s,\boldsymbol{a}, r^j,s',\boldsymbol{a'} \rangle$ in replay buffer $\mathscr{D}$, where 
   \Statex \hskip1.5em $s = s_t$, $\boldsymbol{a} = {a^1_t, \ldots, a^n_t}$ and $\boldsymbol{a'} = {a^{1}_{e, t+1}, \ldots, a^{n}_{e, t+1}}$
   \State Set the current state $s_t = s_{t+1}$
  
  \While{$j$ = 1 to $n$}
    \State Sample a minibatch of $K$ experiences $\langle s,\boldsymbol{a}, r^j, s', \boldsymbol{a'} \rangle$ from $\mathscr{D}$
    \State Set $y^j = r^j + \beta AdvisorQ_{\pi^j}(s')$ according to Eq. \ref{eq:updateoffpolicyQ}
    \State Update the $Q$ eval network by minimizing the loss $L(\phi^j) = \frac{1}{K} \sum (y^j - Q_{\phi^j}(s,\boldsymbol{a}))^2$
  
  \EndWhile
  
  \State Update the parameters of the target network for each agent by copying over the eval network 
  \Statex \hskip1.5em every $\mathcal{T}$ steps: $\pi^j \xleftarrow{} \phi^j $
  
\EndWhile
   
\end{algorithmic}
\end{algorithm}

We also extend Algorithm~\ref{alg:maeqldm} to an actor-critic method --- \textbf{ADvising Multiple Intelligent Reinforcement Agents - Decision Making (Actor-Critic)} abbreviated as ADMIRAL-DM(AC). This algorithm uses the $Q$-function as the critic and the policy derived from $Q$ as the actor. The algorithm follows a \emph{Centralized Training and Decentralized Execution} (CTDE) scheme \citep{lowe2017multi}, where the critic uses the information associated with other agents during the training time and the actors can act independently without access to other agent information during execution. 
This allows our methods to be applicable in environments where global information (i.e., information associated with other agents) is available during training but not available during execution such as, autonomous driving \citep{Ming2020scalable}.
The CTDE scheme extends our algorithm to partially observable environments, where the actor can just use the local observations of the agent for action selection (during both training and execution), while the critic can use the joint observation of all agents during training. As discussed in \cite{lowe2017multi} in the simplest case, the $Q$-function (used as the critic) would consist of the observations of all agents, but it could also include additional state information when available. The critic is not required during execution in this setting.   Furthermore, ADMIRAL-DM(AC) makes our method applicable to continuous action spaces as well.

\begin{algorithm}
\fontsize{11pt}{11pt}\selectfont
   \caption{ADMIRAL-DM(AC) Neural Network Implementation}
   \label{alg:maeacdm}
\begin{algorithmic}[1]
   \State Initialize $V_{\phi^j}, \pi_{\theta^j}$, the critic and actor networks for all $j \in \{1, \ldots, n\}$. Initialize a value for hyperparameters $\epsilon$ and $\epsilon'$ (i.e. value for $\epsilon_0$ and $\epsilon'_0$)
  \State At $t=0$, get the initial state $s_0$
  \State Let the learning agent be indexed by $j$ 
  \State For all $s \in \mathcal{S}$ and $a^j \in A^j$, $j \in \{1,\ldots,n\}$, let $Q^j_t(s, a^1, \ldots, a^n)=0$
  \State Define policy derived from $Q$ to return the random action $a$ with probability $\epsilon_t$, advisor suggested action with probability $\epsilon'_t$ and greedy action with probability $1 - \epsilon_t - \epsilon'_t$
  \State For each $j$ sample action $a^j_0$ from the actor $\pi_{\theta^j}$ at state $s_0$
  \While{training is not finished}
  \State Execute the action $a^j_t$ and observe $r_t^1, \ldots, r^n_t; a^1_t, \ldots, a^n_t$; and $s_{t+1} = s'$, for all agents $j$
  \State For each $j$, let $u$ be a uniform random number between 0 and 1
  
  \If{$u$ < $\epsilon'_t$}

    \State Obtain next action $a^{j}_{t+1} = a^{j'}$ from the advisor (using state $s'$)
    
    \Else
    
    \State Choose an action $a^{j}_{t+1} = a^{j'}$ using the respective actor $\pi_{\theta^j}(s')$
    
    \EndIf

  \State Set $y^j = r^j + \beta V_{\phi^j}^j(s', \boldsymbol{a'}^{-j})$, where $\boldsymbol{a'}^{-j} = ({a'}^1, \ldots, {a'}^{j-1}, {a'}^{j+1}, \ldots, {a'}^N)$, for all $j$
\State For each $j$, update the critic by minimizing the loss $\mathcal{L}(\phi^j) = (y^j - V_{\phi^j}^j(s, \boldsymbol{a}^{-j}))^2 $,
\Statex \hskip1.5em where $\boldsymbol{a}^{-j} = (a^1, \ldots, a^{j-1}, a^{j+1}, \ldots, a^N)$ and $s = s_t$
\State For each $j$, update the actor using the log loss 
$\mathcal{J}(\theta^j) = \log \pi_{\theta^j}(a^j|s)\mathcal{L}(\phi^j) $
\State Set the current action $a^j = a^{j'}$ and the current state $s = s'$ for each agent $j$
 \State At the end of each episode linearly decay $\epsilon'_t$ and $\epsilon_t$
  \EndWhile
   
\end{algorithmic}
\end{algorithm}

We provide the complete pseudocode for the actor-critic implementation of ADMIRAL-DM in Algorithm~\ref{alg:maeacdm}. 
All agents maintain two networks during training. The first network is the value network that serves as a critic, and the second network is a policy network that serves as the actor (line~1).  After each experience, the value (critic) network is updated in line~16 using the TD errors (from Eq.~\ref{Eq:qupdate}) as the loss function. The actor is updated in line~17 using policy gradients. Since the algorithm is maintaining a stochastic policy which explores naturally, we do not need to perform a random action selection for exploration  (unlike in the other algorithms).

\section{Theoretical Results}\label{sec:theoriticalresults}

In this section, we first show that $Q$-updates following Algorithm~\ref{alg:advisorQ} will converge to an $\epsilon$-equilibrium in the stochastic game. From the update rule provided in Eq.~\ref{eq:updateoffpolicyQ}, we note that this equilibrium corresponds to the value of the advisor, which is the action-value function that provides the expectation of immediate reward and future discounted rewards when all agents follow the advisor solutions for infinite periods starting from the current state and joint action. This $\epsilon$ of the $\epsilon$-equilibrium will depend on the nature of the advisor used. Further, we prove that the $Q$-updates following Algorithm~\ref{alg:advisorQ2} converges to the Nash $Q$-value, thus finding the Nash equilibrium of the stochastic game. 

The primary convergence result for $Q$-learning based algorithms in a general-sum stochastic game was provided by \citet{hu2003nash}. However, this result relies on a very restrictive assumption that states that every stage game of the stochastic game contains a Nash equilibrium that is either a global optimum or a saddle point. Additionally, an agent must use the payoff at this equilibrium to update its $Q$-value in every stage game of the stochastic game. As shown by \citet{bowling2000convergence}, this assumption implies that every stage game should use the same kind of equilibrium, it cannot oscillate between being a global optimum or saddle point between stage games. There is almost no game that satisfies this condition in practice \citep{hu2003nash}. We will show in this section that the convergence results in our setting can be provided under a set of assumptions weaker than that used by \citet{hu2003nash}.   

In this section, we provide three important theorems with their detailed proofs. The proofs of each theorem depend on a set of lemmas that we provide. We have included the statement of these lemmas in this section, while the complete proofs of the lemmas can be found in Appendix~\ref{appendix:proofs}.

We start by providing a general result for stochastic processes. Theorem~\ref{maintheorem} is a technical result, extending a result of   \citet{szepesvari1999unified}, which will form the foundation of our convergence result in Theorem~\ref{convergencetheorem}.  Theorem~\ref{maintheorem} aims to relax a requirement of contraction conditions in the result from \citet{szepesvari1999unified}. The presence of these conditions would necessitate strong assumptions in a MARL setting as shown in \citet{hu2003nash}, which we aim to avoid. Towards our Theorem~\ref{maintheorem}, we restate some formal definitions relating to translation and invariance of operators from \citet{szepesvari1999unified} in Appendix~\ref{appendix:definitions} to stay self-contained.

\begin{theorem}\label{maintheorem}
Let $\mathscr{X}$ be an arbitrary set and assume that $\mathcal{B}$ is the space of bounded functions over $\mathscr{X}$. Let  $T:\mathcal{B} \xrightarrow{} \mathcal{B}$, be an arbitrary operator. Let $F \subseteq \mathcal{B} $, be a subset of $\mathcal{B}$ and let $\mathcal{F}_0: F \xrightarrow{} 2^{\mathcal{B}}$ be a mapping that associates subsets of $\mathcal{B}$ with the elements of $F$. Let $v^*$ be a fixed point of $T$ and let $\mathcal{T} = (T_0, T_1, \ldots)$ be a sequence of random operators, $T_t$ mapping $\mathcal{B}\times \mathcal{B}$ to $\mathcal{B}$, that approximate $T$ at $v^*$ and for initial values from $\mathcal{F}_0(v^*)$. Further, assume that $\mathcal{F}_0$ is invariant under $\mathcal{T}$. Let $V_0 \in \mathcal{F}_0(v^*)$, and define $V_{t+1} = T_t(V_t, V_t)$. If there exist random functions $0 \leq F_t(x) \leq 1$ and $0 \leq G_t(x) \leq 1$ satisfying the conditions below with probability 1 (w. p. 1), then $V_t$ converges to a point $(v^* - S)$\footnote{Note that the variable $S$ here does not denote the state space but a deviation from the fixed point $v^*$. We use the (calligraphic) $\mathcal{S}$ to denote the state space.} w. p. 1 in the norm of $\mathcal{B}(\mathscr{X})$:

1. For all $U_1$ and $U_2 \in \mathcal{F}_0$ and all $x \in \mathscr{X}$, 
\begin{equation*}
    \begin{array}{cc}
        T_t(U_1, v^*)(x) - T_t(U_2, v^*)(x)   
        =  G_t(x) (U_1(x) - U_2(x)).  
    \end{array}
\end{equation*}

2. For all $U$ and $V \in \mathcal{F}_0$, and all $x \in \mathscr{X}$, we can find a finite sequence $k_t(x)$ such that, 

\begin{equation*}
    \begin{array}{l}
         T_t(U,v^*)(x) - T_t(U,V)(x) 
          = F_t(x)(|| v^* - V|| + \lambda_t + k_t(x) ||v^* - V||)
    \end{array}
\end{equation*}
\noindent where $\lambda_t \xrightarrow{} 0$ w. p. 1 as $t \xrightarrow{} \infty$ and $k_t(x)$ is finite for all values of $x$ and $t$. 

3. $k_t(x)$ converges to a finite point (independent of time) $K(x)$, as $t \xrightarrow{} \infty$.

4. For all $l>0$, $\Pi_{t=l}^n G_t(x)$ converges to 0 uniformly in $x$ as $n \xrightarrow{} \infty$.

5. There exists $0\leq \gamma < 1$ such that for all $x \in \mathscr{X}$ and large enough $t$, 
        $F_t(x) = \gamma(1-G_t(x))$

The point $S$ can be represented by the equation $S (x) = \frac{1}{\hat{\beta}} (\gamma C_1 + K(x) C_1) $, if $K(x) \neq 0$, $\forall x \in \mathscr{X}$, where $0 < \hat{\beta} \leq 1$ and $C_1$ is a small positive constant. If $K(x) = 0$ $\forall x \in \mathscr{X}$, then $S(x) = 0$ $\forall x$. 
\end{theorem}

Before providing the proof of Theorem~\ref{maintheorem}, we  provide an intuitive grasp of the result by relating the different variables in Theorem~\ref{maintheorem} to RL. The operators $T$ and $\mathcal{T}$ are similar to Bellman operators commonly seen in $Q$-learning, where $T_t$ is the component of $\mathcal{T}$ at time $t$. The $\mathcal{B}$ is the space of all $Q$-functions and $\mathcal{F}_0$ provides a mapping on the subsets of $\mathcal{B}$ (since a set of $n$ elements has $2^n$ subsets, the $\mathcal{F}_0$ maps to $2^{\mathcal{B}}$). Specific instances ($U,V$) of $\mathcal{F}_0$ are considered. These can be seen as particular $Q$-functions. The variable $x$ denotes the parameters of the $Q$-function (state, action pair). Here $v^*$ is the fixed point of the $Q$-function ($Q^*$). The $G_t(x)$ and $F_t(x)$ are functions of the learning rate ($\alpha_t$). These relations will become more explicit in our upcoming results.

Although our proof for Theorem~\ref{maintheorem} is structurally similar to that from \citet{szepesvari1999unified},  there are significant differences in our detailed proof arguments which stems from the differences in the nature of our results. First, \citet{szepesvari1999unified} required an inequality condition in condition 2 to hold for all $x$ and all $t$. 
We do not have this requirement. Our condition~2 uses an exact equivalence, includes an additional term to capture the difference in the other terms and the constraint on this additional term needs to hold only as $t\xrightarrow{} \infty$ (condition~3). As a consequence, we are restricted to showing convergence to a point close to the fixed point $v^*$, instead of exactly to $v^*$. Second, as discussed in \citet{szepesvari1999unified}, condition~2 in their theorem combined with the other conditions turns $T_t$ operator into a contraction condition for all $t$ which is hard to satisfy or ensure in multi-agent environments. While \citet{hu2003nash} use a very restrictive assumption to overcome this problem, we show that this problem can be altogether avoided using our Theorem~\ref{maintheorem} (which is the core motivation for this theorem).
We also use an exact equivalence in condition 1 and condition 5 to avoid the contraction condition. The complete proof of Theorem~\ref{maintheorem} is given next. 

\begin{proof}

Let $U_0$ be a value function in $\mathcal{F}_0(v^*)$ and let $U_{t+1} = T_t(U_t, v^*)$. Since $T_t$ approximates $T$ at $v^*$, $U_t$ converges to $T v^* = v^*$ w. p. 1 uniformly over $\mathscr{X}$. 
Let 
\begin{equation}\label{eq:delequation}
    \begin{array}{l}
         \delta_t(x) = U_t(x) - V_t(x)
    \end{array}
\end{equation}
and let, 

\begin{equation}
    \begin{array}{l}
         \Delta_t(x) =  v^*(x) - U_t(x).
    \end{array}
\end{equation}

We know that $\Delta_t(x)$ converges to 0 because $U_t$ converges to $v^*$. We will show that $\delta_t$ converges to a point (independent of $t$) $S$, w. p. 1, which implies that $V_t$ converges to the point $(v^* - S)$. 

Now from the conditions of Theorem \ref{maintheorem} we have,

\begin{equation}\label{deltaeq}
    \begin{array}{l}
        \delta_{t+1}(x) = U_{t+1}(x) - V_{t+1}(x) 
        \\ \\
        = T_t(U_t, v^*)(x) - T_t(V_t, V_t)(x) 
        \\ \\
       =  T_t(U_t, v^*)(x) - T_t(V_t, v^*)(x)  +  T_t(V_t, v^*)(x) - T_t(V_t, V_t)(x) 
        \\ \\
        = G_t(x) (U_t(x) - V_t(x))  +  F_t(x) (|| v^* - V_t|| + \lambda_t + k_t(x) ||v^* - V||)
        \\ \\ 
        = G_t(x) \delta_t(x) +  F_t(x) (|| v^* - V_t|| + \lambda_t + k_t(x) ||v^* - V||)
        \\ \\
        = G_t(x) \delta_t(x)  +
        F_t(x) (|| v^* - U_t + U_t - V_t ||+ \lambda_t  + k_t(x) || v^* - U_t + U_t - V_t ||)
        \\ \\
        = G_t(x) \delta_t(x)  + F_t(x) (||\delta_t + \Delta_t ||+ \lambda_t +   k_t(x) (||\delta_t + \Delta_t ||))
        \\ \\ 
        \overset{1}{\approx}  G_t(x) \delta_t(x)  + F_t(x) (||\delta_t||+ \lambda_t +  k_t(x)||\delta_t|| + ||\Delta_t|| + k_t(x) ||\Delta_t||)
        \\ \\
        = G_t(x) \delta_t(x) + F_t(x) (||\delta_t|| +  k_t(x) ||\delta_t|| + \epsilon_t).
    \end{array}
\end{equation}

The (1) comes from the fact that $\Delta_t$ is guaranteed to converge to 0 in the limit, so the effect of splitting the sum under the norm is negligible. Regarding the last step, let us denote the term $(\lambda_t + ||\Delta_t|| + k_t(x) ||\Delta_t|| )$ by $\epsilon_t$ as all these terms converge to 0 in the limit. Another assumption in Theorem~\ref{maintheorem} is that $k_t(x)$ converges to $K(x)$ in the limit. We consider two cases, in which the first case is $K(x) = 0$  $\forall x \in \mathcal{X}$ and the second case is when $K(x) \neq 0$. In the first case, notice that, the theorem will then effectively change to Theorem 1 in  \citet{szepesvari1999unified}, where the authors prove that $\delta_t(x)$ converges to 0 and hence $V_t(x)$ converges to $v^*(x)$.

For the second case, we provide the proof for the process $\delta_t$ to converge to a point (independent of time) by keeping a modified process $\hat{\delta}_t$ bounded by rescaling $\delta_t$. Since $\delta_t$ is a homogeneous process, it can be written in the form $\delta_{t+1} = G_t(\delta_t, ||\Delta_t|| + \lambda_t)$, such that $\hat{\beta} G_t(x,y) = G_t(\hat{\beta} x, \hat{\beta} y)$ holds for all $\hat{\beta}>0$. Now we will prove that, if $\hat{\delta}_t$ converges to a point $S$, then $\delta_t$ will also converge to another point, that is $\frac{1}{\hat{\beta}}S$, where $\hat{\beta}$ is the scale factor applied to $\delta_t$ to get the modified process  $\hat{\delta}_t$.

Similar to \citet{szepesvari1999unified}, we will begin by considering a homogenous process and another equivalent formulation of this process that can be obtained by keeping it bounded by scaling. 

Let us consider a process of the form, 

\begin{equation}\label{eq:homogeneous}
    \begin{array}{l}
         x_{t+1} = G_t(x_t, \epsilon_t)
    \end{array}
\end{equation}

\noindent where $G_t: \mathcal{B} \times \mathcal{B} \xrightarrow{} \mathcal{B}$
is a homogeneous random function, i.e., 

\begin{equation}
    \begin{array}{l}
         G_t(\hat{\beta} x, \hat{\beta} \epsilon) = \hat{\beta} G_t(x, \epsilon) 
    \end{array}{}
\end{equation}

\noindent holds for all $\hat{\beta}>0$, $x$ and $\epsilon$. We want to prove that $x_t$ converges to some point independent of $t$. 

Now, consider another process, that is obtained from modifying process in Eq.~\ref{eq:homogeneous} by keeping it bounded by re-scaling, namely the process 

\begin{equation}\label{eq:homogeneous2}
      y_{t+1} = \Bigl\{  
        \begin{array}{ll}
          G_t(y_t, \epsilon_t) &\textrm{if:} || G_t(y_t, \epsilon_t)|| \leq C \\
          C G_t(y_t, \epsilon_t)/ ||G_t(y_t, \epsilon_t)|| &~~~\textrm{otherwise}
        \end{array}
\end{equation}

We denote the solution of Eq.~\ref{eq:homogeneous} corresponding to an initial condition of $x_0 = \omega$ and a sequence $\epsilon = \{\epsilon_k\}$ by $x_t(\omega, \epsilon)$. Similarly, we denote the solution of Eq.~\ref{eq:homogeneous2} corresponding to the initial condition of $y_0 = \omega$ and the sequence $\epsilon$ by $y_t(\omega,\epsilon)$. 

Next, we state a lemma that provides a relationship between convergence of the sequence represented by $x_t$ and the sequence represented by $y_t$. The result is used later in Lemma~\ref{lemma:finallemma}.

\begin{lemm}\label{lemm:rescaling}
Let us fix an arbitrary positive constant $C$, an arbitrary $w_0$, and a sequence $\epsilon$. Then provided that 

(i) $ y_t(w_0, \epsilon)$ converges to a point (independent of $t$) $\mathcal{D}$.

(ii) The sequence $\epsilon$ converges to 0 in the limit ($t \xrightarrow{} \infty$).

\noindent The homogeneous process $x_t(w_0, \epsilon)$ converges to a point $\frac{1}{\hat{\beta}} \mathcal{D}$ w. p. 1, where $\hat{\beta}$, satisfying $0 < \hat{\beta} \leq 1$, is the scaling factor applied.  
\end{lemm}

Next, we  state another lemma that provides conditions for the convergence of a cascade of two converging processes. Again, this result will be used later in Lemma~\ref{lemma:finallemma}.

\begin{lemm}\label{lemma:normedvectorspace}
Let $X$ and $Y$ be normed vector spaces, $U_t: X \times Y \xrightarrow{} X (t=0,1,2, \ldots)$ be a sequence of mappings, and $\theta_t \in Y$ be an arbitrary sequence. Let $\theta_\infty \in Y$ and $x_\infty \in X$. Consider the sequences $x_{t+1} = U_t(x_t, \theta_\infty)$, and $y_{t+1} = U_t(y-t, \theta_t)$ and suppose that $x_t$ and $\theta_t$ converge to $x_\infty$ and $\theta_\infty$ respectively, in the norm of the appropriate spaces.

Let $L^\theta_k$ be the uniform Lipschitz index of $U_k(x,\theta)$ with respect to $\theta$ at $\theta_\infty$ and, similarly, let $L^\mathscr{X}_t$ and $L^\theta_t$ satisfy the relations $L^\theta_t \leq C(1 - L^\mathscr{X}_t)$, and $\Pi_{m=t}^\infty L^\mathscr{X}_m = 0$ where $C>0$ is some constant and $t = 0,1,2, \ldots,$ then $\lim_{t \xrightarrow{} \infty} || y_t - x_\infty|| = 0$.
 
\end{lemm}

Now, we show that stochastic processes having certain special structure will converge to some point independent of $t$, under a set of conditions. This result will also be used later in the proof of Lemma~\ref{lemma:finallemma}. 

\begin{lemm}\label{lemm:boundedlemma}

Let $\mathcal{Z}$ be an arbitrary set and consider the process
\begin{equation}\label{eq:limitlemma}
    \begin{array}{l}
         x_{t+1}(z) = G_t(z)x_t(z) + F_t(z) (C + k_t(z) C)
    \end{array}{}
\end{equation}

\noindent where $x_1, F_t, G_t \geq 0 $ are random processes,  $||x_1|| < C < \infty$ w. p. 1 for some $C>0$, and $z$ is an element in $\mathcal{Z}$. Assume that for all $k$, $\lim_{n \xrightarrow{} \infty} \Pi_{t=k}^n G_t(z) = 0$ uniformly in $z$ w. p. 1 and $F_t(z) = \gamma(1 - G_t(z))$, for some $0 \leq \gamma < 1$, and $\forall z \in \mathcal{Z}$,  w. p. 1. Also, $k_t(z)$ converges to $K(z)$ in the limit. Then, $x_t(z)$ converges to a point $D(z) = \gamma(C + K(z)C)$ w. p. 1.

\end{lemm}

Given the previous results, we are now in a position to conclude the proof by showing that a stochastic process that can be represented by Eq.~\ref{deltaeq} will converge to a point independent of $t$, which is our required result.

\begin{lemm}\label{lemma:finallemma}
Consider an equation of the form
\begin{equation}\label{eq:perterbationterm}
    \begin{array}{l}
         x_{t+1}(z) = G_t(z) x_t(z) + F_t(z) (||x_t|| + \epsilon_t + k_t(z) ||x_t||)
    \end{array}
\end{equation}

\noindent where the sequence $\epsilon_t $ converges to zero w. p. 1. Assume that for all $k$, $\lim_{n \xrightarrow{} \infty} \Pi_{t=k}^n G_t(z) = 0$ uniformly in $z$ w. p. 1 and $F_t(z) = \gamma(1 - G_t(z))$, for some $0 \leq \gamma < 1$, and $\forall z \in \mathcal{Z}$,  w. p. 1. Assume further that $k_t(z)$ is finite, and it converges to $K(z)$ in the limit ($t \xrightarrow{} \infty$). Then $x_t(z)$ converges to a point represented by $S'(z) = \frac{1}{\hat{\beta}} (\gamma C_1 + K(z) C_1) $, where $C_1$ is a small positive constant, w. p. 1 uniformly over $\mathcal{Z}$. Here $\hat{\beta}$ is a scaling factor satisfying $0 < \hat{\beta} \leq 1$.

\end{lemm}

Hence, we have proved that Eq.~\ref{eq:delequation} has converged to a point independent of $t$, and thus Theorem~\ref{maintheorem} follows. The expression for point $S$ is derived in Lemma~\ref{lemma:finallemma}.

\end{proof}



Let us define a relaxation process of the form 

\begin{equation}\label{process}
    \begin{array}{l}
         V_{t+1}(x) = (1 - f_t(x))V_t(x) + f_t(x) [P_t V_t](x)
    \end{array}
\end{equation}
where $0 \leq f_t(x) \leq 1$ is a relaxation parameter and the sequence $P_t: \mathcal{B}(\mathscr{X}) \xrightarrow{} \mathcal{B}(\mathscr{X})$ can be considered a randomized version of the operator $T$.  Let us consider a process $U_t$ such that $\E[V_t(x)] = U_t(x)$.  Also let, $\E[P_t V_t](x) = T V(x)$.
Now we can state the following corollary to Theorem~\ref{maintheorem}.

\begin{corr}\label{corr:conditionalCorollary}
Assume that the process defined by 

\begin{equation}\label{process2}
    \begin{array}{l}
         U_{t+1}(x) = (1 - f_t(x))U_t(x) + f_t(x) [P_t v^*](x)
    \end{array}
\end{equation}
converges to $v^*$ w. p. 1. Assume further that the following conditions hold: 

1. There exist a scalar $\gamma$ satisfying $0\leq \gamma<1$ and a sequence $\lambda_t \geq 0$ converging to 0 w. p. 1 such that $|| P_t v^* - P_t V||  = \gamma ||v^* - V|| + \lambda_t  + \gamma k_t (x) ||v^* - V||$ holds for all $V \in \mathcal{B}(\mathscr{X})$ and for finite $k_t(x)$.

2. $k_t(x)$ converges to finite point $K(x)$ as $t$ goes to $\infty$. 

3. $0 \leq f_t(x) \leq 1$, $t \geq 0$, and $\sum_{t=1}^{n} f_t(x)$ goes to infinity uniformly in $x$ as $n \xrightarrow{} \infty$. 

\noindent Then, the iteration defined by Eq.~\ref{process} converges to a point $(v^* - S)$, where $S$ is defined as in Theorem~\ref{maintheorem}. 
\end{corr}

\begin{proof}

Here, $P_t$ is a randomized version of an operator $T$. It can be proved that a process of the form,
\begin{equation}\label{eq:averageequation}
    \begin{array}{l}
         U_{t+1}(x) = (1 - f_t(x))U_t(x) + f_t(x) [P_t V](x)
    \end{array}
\end{equation}
will converge to $TV$ w. p. 1 where $V \in \mathcal{B}(\mathscr{X})$. The convergence is due to Lemma~\ref{lemm:conditionalaveraging} below.

\begin{lemm}\label{lemm:conditionalaveraging}
Let $\mathcal{F}_t$ be an increasing sequence of $\sigma$-fields, let $0\leq \alpha_t$ and $w_t$ be random variables such that $\alpha_t$ and $w_{t-1}$ are $\mathcal{F}_t$ measurable. Assume that the following hold w. p. 1: $E[w_t|\mathcal{F}_t, \alpha_t \neq 0] = A$, $E[w^2_t|\mathcal{F}_t] < B < \infty$, $\sum_{t=1}^{\infty} \alpha_t = \infty$ and $\sum_{t=1}^{\infty} \alpha^2_t < C < \infty$ for some $B,C > 0$. Then the process 

\begin{equation}
    \begin{array}{l}
         Q_{t+1} = (1 - \alpha_t)Q_t + \alpha_t w_t
    \end{array}
\end{equation}

\noindent converges to A w. p. 1. 

\end{lemm}

Let the random operator sequence $T_t: \mathcal{B}(\mathscr{X}) \times \mathcal{B}(\mathscr{X}) \xrightarrow{} \mathcal{B}(\mathscr{X})$ be defined by 
\begin{equation}
    \begin{array}{l}
         T_t(U,V)(x) = (1 - f_t(x))U(x) + f_t(x) [P_t V](x).
    \end{array}
\end{equation}

\noindent We know that the operator $T_t$ approximates $T$ at $v^*$, since the process defined by Eq.~\ref{eq:averageequation} converges to $TV$ for all $V \in \mathcal{B}(\mathscr{X})$ by Lemma~\ref{lemm:conditionalaveraging} and the process defined by Eq.~\ref{process2} converges to $v^*$ by definition.  Moreover, observe that $V_t$ as defined by Eq.~\ref{process} satisfies $V_{t+1} = T_t (V_t, V_t)$. Due to conditions 1,2, and 3, it can be readily verified that coefficients $G_t(x) = 1 - f_t(x)$, and $F_t(x) = \gamma f_t(x)$ satisfy the rest of the conditions of Theorem~\ref{maintheorem}, and this yields that the process $V_t$ converges to $(v^* - S)$ w. p. 1. 

\end{proof}

Now, we define the $P_t$ operator in the context of Algorithm~\ref{alg:advisorQ}.

\begin{defn}\label{defn:Pt}
Let $Q = (Q^1, \ldots, Q^n)$. We define an operator $P_t: \mathcal{Q} \xrightarrow{} \mathcal{Q}$ such that $P_t Q = (P_t Q^1, \ldots, P_t Q^n)$, where 
\begin{equation}\label{eq:Ptdefn}
    \begin{array}{l}
         P_t Q^k(s, \boldsymbol{a}) =  r^k_t(s, \boldsymbol{a}) + 
         \beta \sigma_t^1(s') \cdots \sigma_t^n(s') Q^k(s')
    \end{array}
\end{equation}

\noindent and $k = (1, \ldots, n)$, $s'$ is the state at time $t+1$, and $(\sigma_t^1(s'), \ldots, \sigma_t^n(s'))$ is an advisor solution for the stage game $(Q^1(s'), \ldots, Q^n(s'))$ at time $t$. 
\end{defn}

The $P_t$ operator's value depends on the advisor solution. 
Next, we will state some assumptions. The first two are commonly used in RL  \citep{jaakkola1994convergence, singh2000convergence}.

\begin{assumption}\label{assumption:visitassumption}
    Every state $s\in \mathcal{S}$ and action $a^j \in A^j$, for each agent $j \in \{1, \ldots, n\}$, 
    are visited infinitely often, and the reward function for all agents stay bounded. 
\end{assumption}


\begin{assumption}\label{assumption:learningrate}
For all $s, t$, and $\boldsymbol{a}$, $0 \leq \alpha_t(s,\boldsymbol{a}) < 1$, $\sum_{t=0}^\infty \alpha_t(s, \boldsymbol{a}) = \infty$, $\sum_{t=0}^\infty[\alpha_t(s, \boldsymbol{a})]^2 < \infty$.

         


          
 
\end{assumption}

\begin{assumption}\label{assumption:advisorassumption}
The advisor solution to any stage game $(Q^1(s), \ldots, Q^n(s))$ pertaining to state $s \in \mathcal{S}$ will become stationary in the limit ($t \xrightarrow{} \infty$). In other words, the advisor can adapt its solutions, but in the limit it is guaranteed to settle down and provide the same advisor strategy for a particular state $s \in \mathcal{S}$. 
\end{assumption}

         


Assumption~\ref{assumption:advisorassumption} allows the advisor to learn and adapt during gameplay. However, in the limit, the advisor is required to converge and provide the same strategy. Now, directly, it can be seen that this assumption is much weaker than the restrictive assumption in \citet{hu2003nash}. Firstly, Assumption~\ref{assumption:advisorassumption} does not involve the computation of Nash equilibrium of every stage game in the stochastic game. Secondly, agents need not use the Nash equilibrium of every stage game to update their $Q$-values unlike the assumption in \citet{hu2003nash}. If this condition were needed, then the advisor will need to only suggest the Nash equilibrium at every stage game, greatly reducing the scope of the advisor. 
Thirdly, the Nash equilibrium at every stage game does not need to be a global optimum or saddle point. This condition in \citet{hu2003nash} is not satisfied in any practical game, which we have relaxed. 

Next, we state a lemma needed to prove convergence. 

\begin{lemm}\label{corrolorylemma}
Assume that $\alpha_t$ satisfies Assumption 2 and the mapping $P_t: \mathcal{Q} \xrightarrow{} \mathcal{Q}$ satisfies the condition that, there exists a scalar $\gamma$ satisfying $0 \leq \gamma < 1$, a sequence $\lambda_t \geq 0$ converging to zero w. p. 1, and a finite sequence $k_t(s)$ such that $|| P_tQ  - P_t Q_{*} || = \beta ||Q - Q_{*} || + \lambda_t + \beta k_t(s)||Q - Q_{*} ||$ for all $Q$, and all $s \in \mathcal{S}$. Assume further that, $k_t(s)$ converges to a finite point $K(s)$ in the limit. Additionally, $Q_{*}(s, \boldsymbol{a}) = E[P_tQ_{*}(s, \boldsymbol{a}) ]$, then the iteration defined by 
\begin{equation}
    \begin{array}{l}
         Q_{t+1}(s, \boldsymbol{a}) = (1 - \alpha_t)Q_t(s, \boldsymbol{a})  +\alpha_t[P_t Q_t(s, \boldsymbol{a})]
    \end{array}
\end{equation}
converges to $(Q_{*} - S)$ w. p. 1, where $S$ is as given in Theorem~\ref{maintheorem}. 
\end{lemm}











Theorem~\ref{convergencetheorem} proves that the $Q$-updates in Algorithm~\ref{alg:advisorQ}, converges to an $\epsilon$-equilibrium consistent with Definition~\ref{def:epsionadvisorq}. The $Q$-values at convergence denotes the value of the advisor. The proof of this theorem will be an application of Lemma~\ref{corrolorylemma}. The bound $\epsilon$ of the $\epsilon$-equilibrium will depend on the advisor through its advisor solutions.  

As proved by \cite{fink1964equilibrium}, every stochastic game is guaranteed to have at least one Nash equilibrium point. However, there can be more than one Nash equilibrium point, in which case, the $Q_{*}$ in Theorem~\ref{convergencetheorem}, could be the Nash $Q$-function of any Nash equilibrium strategy. We do not require the Nash equilibrium point of the stochastic game to be unique in Theorem~\ref{convergencetheorem} as assumed in prior works \citep{hu2003nash}.

\begin{theorem}\label{convergencetheorem}
Under the Assumptions~\ref{assumption:visitassumption}, \ref{assumption:learningrate}, and \ref{assumption:advisorassumption}, the $Q$-functions updated by Eq.~\ref{eq:updateoffpolicyQ} 
converges to a bounded distance from the Nash $Q$-function $Q_{*} = (Q^1_{*}, \ldots, Q^n_{*})$, represented as $( Q_{*} - S)$, in the limit ($t \xrightarrow{} \infty$). The point $S$ is as given in Theorem \ref{maintheorem}. 

\end{theorem}

\begin{proof}

We will state a lemma before we begin the proof of the theorem.

\begin{lemm}\label{lemm:expectation}
For a $n$-player stochastic game, $E[P_t Q_{*}] = Q_{*}$ where $Q_* = (Q^1_{*}, \ldots, Q^n_{*}).$
\end{lemm}

The proof of the Theorem~\ref{convergencetheorem} is a direct application of Lemma~\ref{corrolorylemma}, which establishes convergence given three conditions. The condition pertaining to the expectation relationship is satisfied by Lemma~\ref{lemm:expectation}.

To satisfy the first condition (distances between $Q$-functions), we will begin by providing two expressions pertaining to distances between $Q$-functions seen in Lemma~\ref{corrolorylemma}.  The first expression gives a formal way of determining the distance between any $Q$-function and the Nash $Q$-function. 

Consider two $Q$ functions, $Q, Q_{*} \in \mathcal{Q}$. Then we can get,
    \begin{equation}\label{eq:Qtapplication}
    \begin{array}{l}
        ||Q - Q_{*}|| \triangleq  max_j max_s ||Q^j(s) - Q_{*}^j(s) || 
        \\ \\
         = max_j max_s max_{a^1, \ldots, a^n} ||Q^j(s, a^1, \ldots, a^n) -  Q_{*}^j(s, a^1, \ldots, a^n) ||
    \end{array}
    \end{equation}
    where $Q_{*} = (Q^1_*, \ldots, Q^n_*)$, $Q^j_*$ is the Nash $Q$-function of the agent $j$.

The second expression below pertains to distances under corresponding $P_t$ operators.

\noindent Consider two $Q$ functions $Q, Q_{*} \in \mathcal{Q}$. Then we can get, 
\begin{equation}\label{eq:Ptapplication}
    \begin{array}{l}
     ||P_t Q  - P_t Q_{*}|| \triangleq max_j || P_t Q^j - P_t Q_{*}^j||
     \\ \\ 
     = max_j max_s|| \beta \sigma^1_t(s) \cdots \sigma^n_t(s) Q^j(s)  - \beta \sigma_{*,t}^1(s) \cdots {\sigma}_{*,t}^n(s) Q_{*}^j(s)||
     \\ \\
     = max_j \beta|| \sigma^1_t(s) \cdots \sigma^n_t(s) Q^j(s)  - \sigma_{*,t}^1(s) \cdots \sigma_{*,t}^n(s) Q_{*}^j(s)||.
    \end{array}
\end{equation}

\noindent Here, $(\sigma^1_t(s), \ldots, \sigma^n_t(s))$ is an advisor solution for the stage game $(Q^1(s), \ldots, Q^n(s))$ at time $t$ and  $(\sigma^1_{*,t}(s'), \ldots, \sigma^n_{*,t}(s'))$ is an advisor solution for the stage game $(Q^1_*(s'), \ldots, Q^n_*(s'))$ at time $t$. Also $Q_{*}^j$ is the Nash $Q$-function of agent $j$.

In the second step, the corresponding reward terms from the $P_t$ operators get cancelled from Definition~\ref{defn:Pt}. Since the state $s$ changes for every $t$, the dependence on $max_s$ in Eq.~\ref{eq:Ptapplication} is removed in the last step. 

\begin{equation}\label{eq:assumptionfix}
    \begin{array}{l}
   || P_tQ  - P_t Q_{*} || = \beta ||Q - Q_{*} || + \lambda_t + \beta k_t(s)||Q - Q_{*} ||
    \end{array}
\end{equation}

We state the equation in Lemma~\ref{corrolorylemma} again in  Eq.~\ref{eq:assumptionfix}. Now, we can find a finite  $k_t(s)$ such that the Eq.~\ref{eq:assumptionfix} is satisfied for all $Q \in \mathcal{Q}$ and all $s \in \mathcal{S}$. This is due to the fact that all the other terms in that expression are guaranteed to be finite due to Assumption~\ref{assumption:visitassumption}. This satisfies another condition of Lemma~\ref{corrolorylemma}.

To satisfy the last condition of Lemma~\ref{corrolorylemma}, we rewrite the Eq.~\ref{eq:assumptionfix}, to get an expression for $k_t(s)$ in the limit ($t \xrightarrow{} \infty$), 

\begin{equation}\label{eq:rewrite}
    \begin{array}{l}
   k_t(s) = \frac{||P_t Q  - P_t {Q_{*}}||}{\beta||Q - Q_{*}||} - 1 \\  \\
   k_t(s) = \frac{\max_j \beta | \sigma^1_t(s) \cdots \sigma^n_t(s) Q^j(s)  - \sigma^1_{*,t}(s) \cdots \sigma_{*,t}^n(s) Q_{*}^{j}(s) |}{\beta||Q - Q_{*}||} - 1
   \\ \\ 
   k_t(s) = \frac{\max_j | \sigma^1(s) \cdots \sigma^n(s) Q^j(s)  - \sigma^1_{*}(s) \cdots \sigma_{*}^n(s) Q_{*}^{j}(s) |}{||Q - Q_{*}||} - 1 
   \\ \\ 
   k_t(s) = K(s).
    \end{array}
\end{equation}

The first expression in Eq.~\ref{eq:rewrite} is obtained as $\lambda_t$ is guaranteed to go to 0 in the limit. In the second expression, Eq.~\ref{eq:Ptapplication} is being used. The third expression is from Assumption~\ref{assumption:advisorassumption}. Now, from Eq.~\ref{eq:rewrite}, we can see that $k_t(s)$ is a constant for a given state that has no dependence on time $t$. This satisfies the last condition of Lemma~\ref{corrolorylemma}. 

Hence, all the conditions of Lemma~\ref{corrolorylemma} are satisfied and, therefore, the process $ Q_{t+1} = (1 - \alpha_t) Q_t + \alpha_t[P_t Q_t]$ converges to a bounded distance from the Nash equilibrium $(Q_{*} - S)$. Thus, the equilibrium point reached is an epsilon equilibrium, where the epsilon value can be given by the point $S$. The expression for this point is as given in Theorem \ref{maintheorem} with the value of $K$ being given in Eq.~\ref{eq:rewrite}. Here the agents can still unilaterally deviate and possibly obtain an additional payoff consistent with Definition~\ref{def:epsionadvisorq}.
\end{proof}

Now, we move on to providing theoretical guarantees for Algorithm \ref{alg:advisorQ2}. We will show in Theorem~\ref{theorem:onpolicy} that this algorithm will converge to a Nash equilibrium under a set of assumptions. Again, these assumptions are weaker than others previously considered in literature. To prove this convergence result for Algorithm~\ref{alg:advisorQ2} we will retain the first two assumptions but replace 
 Assumption~\ref{assumption:advisorassumption} with two other assumptions.

\begin{assumption}\label{assumption:GLIE}
The algorithm is Greedy in the limit with Infinite Exploration (GLIE).
\end{assumption}

\begin{assumption}\label{assumption:globaloptimum}
The Nash equilibrium is a global optimum or saddle point in every stage game of the stochastic game. 
\end{assumption}

Assumption~\ref{assumption:GLIE} allows the policy to explore, but in the limit ($t \xrightarrow{} \infty$) this assumption requires the policy to choose greedy actions.  Assumption~\ref{assumption:globaloptimum} is strong, however, \cite{hu2003nash} show that similar assumptions are required to prove convergence in theory but not necessary to observe convergence in practice, which is consistent with our observations as well. Thus, even if such assumptions are violated in practice, convergence is still observed. Further, it is to be noted that Assumption~\ref{assumption:globaloptimum} is a 
weaker condition than the assumption in \citet{hu2003nash}, since it does not require calculating the Nash equilibrium at each stage game and using the same to update the $Q$-values.

\begin{theorem}\label{theorem:onpolicy}

When we update the $Q$-functions according to Eq.~\ref{Eq:qupdate}, they converge to the Nash $Q$-function under Assumptions~\ref{assumption:visitassumption}, \ref{assumption:learningrate}, \ref{assumption:GLIE}, and \ref{assumption:globaloptimum}, in the limit ($t \xrightarrow{} \infty$).

\end{theorem}

Theorem~\ref{theorem:onpolicy} proves that the Q-updates in Algorithm~\ref{alg:advisorQ2} converges to the Nash equilibrium strategies. The proof is along the lines of Theorem~1 in \cite{singh2000convergence}, but involves significant modifications to cater to the multi-agent scenario.

\begin{proof}
The proof will involve using two lemmas from previous work, one from \citet{jaakkola1994convergence} and the other from \citet{hu2003nash}. We will start the proof of this theorem by stating the first lemma.

\begin{lemm}\label{lemma:randomprocess}

A random iterative process 

\begin{equation}\label{eq:deltaeq}
\begin{array}{l}
     \Delta_{t+1}(x) = (1 - \alpha_t(x))\Delta_t(x) + \alpha_t(x) F_t(x)
\end{array}{}
\end{equation}

\noindent where $x \in X$, $t = 0,1, \ldots, \infty$, converges to zero with probability one (w. p. 1) if the following properties hold: 
 
1. The set of possible states $X$ is finite. 

2. $0 \leq \alpha_t(x) \leq 1$, $\sum_t \alpha_t(x) = \infty$, $\sum_t \alpha^2_t(x) < \infty$ w. p. 1, where the probability is over the learning rates $\alpha_t$. 

3. $|| \E \{{F_t(x)|\mathscr{P}_t}\} ||_W \leq \mathscr{K} ||\Delta_t||_W + c_t$, where $\mathscr{K} \in [0,1)$ and $c_t$ converges to zero w. p. 1. 

4. $\textrm{\textbf{var}}\{F_t(x) | \mathscr{P}_t\} \leq K(1 + ||\Delta_t||_W)^2$, where $K$ is some constant. 

\noindent Here $\mathscr{P}_t$ is an increasing sequence of $\sigma$-fields that includes the past of the process.  In particular, we assume that $\alpha_t, \Delta_t, F_{t-1} \in \mathscr{P}_t$. The notation $||\cdot||_W$ refers to some (fixed) weighted maximum norm and the notation $\textrm{\textbf{var}}$ refers to the variance.

\end{lemm}

Let us define a Nash operator $P_t$, consistent with the definition in \cite{hu2003nash}. The Nash operator is defined using the following equation, 

\begin{equation}\label{eq:nashoperator}
\begin{array}{l}
    P_t Q^k(s, a^1, \ldots, a^n) = \E_{s' \sim p} [r^k_t(s,a^1, \ldots, a^n)  + \gamma \pi^1_{*} (s') \cdots \pi^n_{*}(s') Q^k(s')]
    \end{array}
\end{equation}

\noindent where $s'$ is the state at time $t+1$, $(\pi^1_{*} (s') ,\ldots, \pi^n_{*}(s'))$ is the Nash equilibrium solution for the stage game $(Q^1(s'), \ldots, Q^n(s'))$, and $p$ is the transition function. $Q^k$ denotes the $Q$-value of a representative agent $k$.

\begin{lemm}\label{lemm:nashoperator}
Under Assumption~\ref{assumption:globaloptimum}, the Nash operator as defined in Eq.~\ref{eq:nashoperator} forms a contraction mapping with the fixed point being the Nash $Q$-value of the game. 

\end{lemm}

Now, since the $P_t$ operator forms a contraction mapping, $||P_t Q - P_t Q_*|| \leq \beta || Q - Q_*||$, is satisfied for some $\beta \in [0,1)$ and all $Q$. Here $Q_*$ is the Nash $Q$-value.

We will apply Lemma~\ref{lemma:randomprocess} to show that the $Q$-values converge to the Nash $Q$-value. We will drop the agent index in all the expressions for simplicity. The rest of the proof is conducted for the $Q$-values of a representative agent. 

The first two conditions of Lemma~\ref{lemma:randomprocess} are satisfied from the assumptions. Comparing Eq.~\ref{eq:deltaeq} and Eq.~\ref{Eq:qupdate} we get that $x$ can be associated with the state action pairs $(s,a^1,\ldots,a^n) $ and  $\Delta_t(s_t,a_t)$ can be associated with $Q_t(s,a^1,\ldots,a^n) - Q_*(s,a^1,\ldots,a^n)$. Here, $Q_*(s,a^1,\ldots,a^n)$ can be considered as the Nash $Q$-value ($Q$-values under the Nash $Q$-function). 

Now we get 

\begin{equation}
    \begin{array}{l}
         \Delta_{t+1}(x) = (1 - \alpha_t(x)) \Delta_t(x) + 
         \alpha_t(x)F_t(x), 
    \end{array}{}
\end{equation}

\noindent where 

\begin{equation}\label{eq:contraction}
    \begin{array}{l}
         F_t(x) = r_t + \gamma v^{Nash}(s_{t+1}) 
         - Q_*(s_t,a_t^1,\ldots,a_t^n)  
          + \gamma[Q_t(s_{t+1}, a_{t+1}^{1}, \ldots, a_{t+1}^{n}) - v^{Nash}(s_{t+1})] 
         \\ \\
        \delequal r_t + \gamma v^{Nash}(s_{t+1})  - Q_*(s_t,a_t^1,\ldots,a_t^n)  
        + C_t(s_t,a_t^1,\ldots,a_t^n)
        \\ \\ 
         \delequal  F_t^{Q}(s_t,a_t^1,\ldots,a_t^n) + C_t(s_t,a_t^1,\ldots,a_t^n).
         
     \end{array}{}
\end{equation}{}

\noindent We define  $F_t(s_t,a_t^1,\ldots,a_t^n) = F_t^{Q}(s_t,a_t^1,\ldots,a_t^n)$ = $C_t(s_t,a_t^1,\ldots,a_t^n) = 0 $ if $(s,\boldsymbol{a}) \neq (s_t, \boldsymbol{a}_t) $. We also define $\boldsymbol{a} = (a_1, \ldots, a_n)$ and $\boldsymbol{a}_t = (a_{1,t}, \ldots, a_{n,t})$. Let the $\sigma$-field generated by all the random variables $ (s_t, \alpha_t, a^1_t, \ldots, a^n_t, r_{t-1}, \ldots, s_1, \alpha_1, a_1, Q_0 )$ be represented by $\mathscr{P}_t$. Now, all the $Q$-values are $\mathscr{P}_t$ measurable which makes $\Delta_t$ and $F_t$, $\mathscr{P}_t$ measurable and this satisfies the measurability condition of Lemma~\ref{lemma:randomprocess}. 

\citet{hu2003nash} showed that $v^{Nash}(s_{t+1}) \triangleq v^{k}(s', \pi^1_*, \ldots, \pi^n_*) = \pi^1_*(s') \cdots \pi^n_{*}(s') Q^k_*(s')$ (see the proof in Lemma 10 of \cite{hu2003nash}). Hence, from Lemma \ref{lemm:nashoperator}, we can show that the $\E[F^{Q}_t]$ forms a contraction mapping. This can be done using the fact that $\E(P_t Q_*) = Q_*$ (Lemma~\ref{lemm:expectation}). Here, the norm is the maximum norm on the joint action. 

Now, we have the following for all $t$,

\begin{equation}
    \begin{array}{l}
         ||\E[F_t^Q(s_t,a_t^1,\ldots,a_t^n) | \mathscr{P}_t] || \leq \gamma||Q_t(s_t,a_t^1,\ldots,a_t^n) 
         - Q_*(s_t,a_t^1,\ldots,a_t^n) || = \gamma||\Delta_t||.
    \end{array}{}
\end{equation}

\noindent Now from Eq.~\ref{eq:contraction},
\begin{equation}
    \begin{array}{l}
         ||\E[F_t(s_t,a_t^1,\ldots,a_t^n) | \mathscr{P}_t] || \leq
         
         ||\E[ F_t^{Q}(s_t,a_t^1,\ldots,a_t^n) | \mathscr{P}_t] ||
         
        +  ||\E [C_t (s_t,a_t^1,\ldots,a_t^n) | \mathscr{P}_t] ||
         \\ \\
        \leq \gamma ||\Delta_t || + || \E[C_t(s_t,a_t^1,\ldots,a_t^n)|\mathscr{P}_t]||. 
    \end{array}
\end{equation}

\noindent This satisfies the third condition of Lemma~\ref{lemma:randomprocess} if $c_t = || \E[C_t(s_t,a_t^1,\ldots,a_t^n)|\mathscr{P}_t]|| $ converges to 0 w. p. 1. \\

Let us rewrite the definition of the $C_t$, 

\begin{equation}\label{eq:Ceqution}
\begin{array}{l}
   C_t(s_t,a_t^1,\ldots,a_t^n) =  \gamma[Q_t(s_{t+1}, a_{t+1}^{1'}, \ldots, a_{t+1}^{n'})  - v^{Nash}(s_{t+1})] \\ \\
   C_t(s_t,a_t^1,\ldots,a_t^n) =  \gamma[\max Q_t(s_{t+1}, a_{t+1}^{1'} \ldots, a_{t+1}^{n'}) -  v^{Nash}(s_{t+1})]
   \\  \\
   C_t(s_t,a_t^1,\ldots,a_t^n) =  \gamma[v(s_{t+1}) - v^{Nash}(s_{t+1})].
\end{array}
\end{equation}

In the second step, we are using the assumption that the $Q$-value is GLIE. The max operator operates over the action space of the representative agent. 

According to  Assumption~\ref{assumption:globaloptimum}, the Nash equilibrium could only be a global optimum or a saddle point. Now, if it is a global optimum, the value of maximizing all the actions in Eq.~\ref{eq:Ceqution} will lead to the global optimum for all the agents and this will be the Nash payoff, thus leading to $C_t$ evaluating to 0 in the limit. Furthermore, \citet{hu2003nash} show that a global optimum is always a Nash equilibrium and all global optima are guaranteed to have equal values. Alternatively, if the Nash equilibrium is a saddle point, consider a stage game, with saddle point equilibrium payoff, $\sigma$ and $\pi$. Then, $\sigma^k \sigma^{-k}Q^k(s) \geq \pi^k \sigma^{-k}Q^k(s) $, as deviating from the equilibrium when the others are playing the equilibrium strategy will leave an agent worse off by definition of a Nash equilibrium. Also,  $\pi^k \pi^{-k}Q^k(s) \leq \pi^k \sigma^{-k}Q^k(s)$, as in a saddle point, if others deviate the agent should be better off (see Definition~13 in \cite{hu2003nash}). Thus, we will get the relation, $\pi^k \pi^{-k}Q^k(s) \leq \sigma^k \sigma^{-k}Q^k(s)$. Since $\sigma$ and $\pi$ are saddle points, the previous argument holds without the loss of generality. Hence, the following is also true, $\sigma^k \sigma^{-k}Q^k(s) \leq \pi^k \pi^{-k}Q^k(s)$. Thus, the value obtained is the same in the saddle points and the value would be Nash value if all the agents are being greedy given the strategies of all other agents. Thus, we have proved that for all cases $C_t$ converges to 0 in the limit. 

The fourth condition of the Lemma~\ref{lemma:randomprocess} is satisfied since we have the reward to be bounded (Assumption~\ref{assumption:visitassumption}) and we know the variance of $F_t^Q$ is bounded \citep{jaakkola1994convergence}. 

Thus, it follows from Lemma~\ref{lemma:randomprocess} that the process $\Delta_t$ converges to 0 and hence, $Q_t$ converges to the Nash $Q$-function $Q_*$. 
\end{proof}

Theorem~\ref{theorem:onpolicy} shows that complicated steps in algorithms such as Nash-Q \citep{hu2003nash} to predict the actions of other agents using an equilibrium calculation are unnecessary. The other agents' current policy can be used directly instead of the equilibrium calculations. Assumption~\ref{assumption:globaloptimum} is strong enough to ensure that such processes converge. 
Also, Theorem~\ref{theorem:onpolicy} proves convergence without any restrictions on the nature of advisors. They could be sub-optimal or adversarial.
Thus, in both the Theorem \ref{convergencetheorem} and Theorem \ref{theorem:onpolicy}, we have proved fixed point guarantees in general-sum stochastic games with  weaker assumptions than those used by earlier work (\citet{hu2003nash}).
We also note that Theorem \ref{theorem:onpolicy} just assumes that the advisor influence decays to 0 in the limit, and so is also applicable to learning algorithms without advisors. 

To conclude, from the Theorem~\ref{convergencetheorem} and Theorem~\ref{theorem:onpolicy} we see that both our algorithms (i.e., ADMIRAL-AE and ADMIRAL-DM) have a suitable fixed point and a guarantee of converging to that fixed point in the limit. This shows that our algorithms are theoretically grounded.

\section{Experiments}\label{sec:mainexperiments}

We experimentally validate our algorithms, showing their effectiveness in a variety of situations using different testbeds. We also demonstrate superior performance to common baselines previously used in literature. The source code for the experiments has been open sourced \citep{sourcecode}.

\subsection{Experimental Results - Tabular Version}\label{sec:gridworldappendix}

The objective of this section is to provide a simple illustration of the tabular version of our algorithms using different kinds of advisors. 

\begin{figure}
    \centering
    \includegraphics[width=0.3\textwidth]{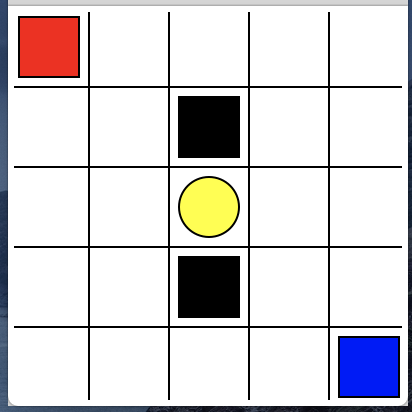}
    \caption{Grid Maze environment}
    \label{fig:gridworldenv}
\end{figure}

For our first experiment, we investigate the performance of ADMIRAL-AE, where we control the quality of the advisors. To this end, we use a $5 \times 5$ Grid Maze environment for the empirical evaluation. The schematic representation for this environment is available in Figure~\ref{fig:gridworldenv}. Here the red and blue agents are trying to reach the yellow goal. The agents must learn to avoid the black pitfalls while reaching the goal. This is a cooperative game where the two agents must learn to perfectly coordinate the task of reaching the goal to achieve maximum reward.  However, the agents are unable to communicate directly. An agent has to learn to take the correct actions to reach the goal in relation to the observed behaviour of the other agent. The game terminates if at least one of the agents reaches the goal state or hits a pitfall.  
We implemented four rule-based advisors of decreasing quality, from Advisor~1, which provides the best action for each state (relative to the other advisors), to Advisor~4, which makes random suggestions. The state in this game corresponds to the positions (grid coordinates) of both the agents and the action involves moving in one of the four cardinal directions. A detailed description of the environment, reward function, action selection, and more details on the advisors are in Appendix~\ref{appendix:experimentaldetails}.

First we conduct the `pre-learning' phase where we study the performance of ADMIRAL-AE using the different advisors in the Grid Maze domain. In our implementations, both agents play a separate instance of the same algorithm (ADMIRAL-AE with a particular advisor). We consider the cumulative rewards obtained by ADMIRAL-AE with each of the advisors over a period of 2000 episodes of training.  Figure~\ref{fig:gridworld}(a) shows the average result over five runs.  Since perfect coordination gives the large positive rewards, we can see from Figure~\ref{fig:gridworld}(a) that the ADMIRAL-AE implementation with the best advisor (Advisor~1) results in the  highest performance for the task. The performance progressively degrades from Advisor~1 to Advisor~4. We highlight that the performance of ADMIRAL-AE depends on the quality of the advisor, with the best advisor (Advisor~1) leading to the best overall performance. This also shows that we would find most value in learning from Advisor~1, as compared to the other advisors. Thus, from Figure~\ref{fig:gridworld}(a) we conclude that there is great value in using ADMIRAL-AE when the quality and suitability of advisors are the objective of study. 

\begin{figure}

	\centering
	\subfloat[Performance of ADMIRAL-AE using 4 different advisors.
	]{{\includegraphics[width=0.45\textwidth, height=5cm]{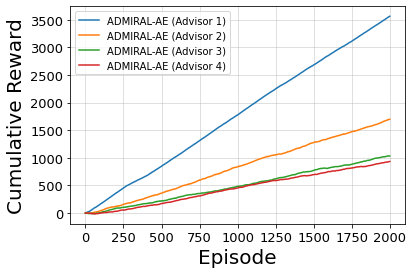} }}
	\quad
	\subfloat[Mean Square Error (MSE), between current $Q$-value of ADMIRAL-AE using Advisor~1 and the $Q_\sigma$, which is the true value of the advisor.
	]{{\includegraphics[width=0.45\textwidth, height=5cm]{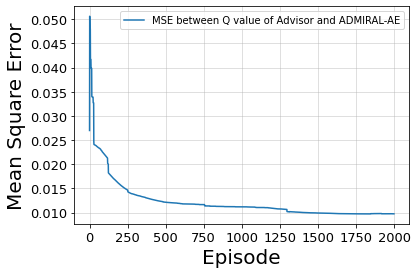}}}
  \caption{Experimental findings using the tabular version of ADMIRAL-AE with different advisors. (a) shows that ADMIRAL-AE with the best advisor (Advisor 1) gives the best overall performance, and ADMIRAL-AE with the worst advisor (Advisor 4) gives the worst overall performance. (b) shows that the MSE between the $Q$-value from ADMIRAL-AE and the value of the advisor $Q_\sigma$ progressively reduces, and hence ADMIRAL-AE evaluates correctly. All results show an average of five runs, and they have negligible standard deviation.}%
	\label{fig:gridworld}
\end{figure}

Further, we plot the mean square error (MSE) between the $Q$-values (of every state and joint action) obtained from playing ADMIRAL-AE (using the best advisor, Advisor~1) at the end of each episode and the true $Q$-value of the advisor (denoted as $Q_\sigma$). As mentioned before, the true $Q$-value of the advisor captures the expectation of the sum of the immediate reward and the discounted sum of future rewards obtained, when all agents follow the advisor's strategy for infinite periods starting from the current state and joint action. We obtain this value by running trajectories from each state and joint action pair until the end of the episode and calculating the expected discounted sum of rewards. The plot of the MSE is given in Figure~\ref{fig:gridworld}(b). From this figure, we see that the MSE approaches close to zero after about 2000 episodes of training, which shows that the ADMIRAL-AE algorithm correctly evaluates the advisor. This experiment provides another illustration for the effectiveness of ADMIRAL-AE in evaluating the given advisor.

In Table~\ref{tab:mazeepsilon} we use the average cumulative performances (5 runs) of ADMIRAL-AE along with all the advisors (from Figure~\ref{fig:gridworld}(a)) to find a value for $\epsilon'_0$, the hyperparameter that controls the advisor influence in ADMIRAL-DM. Recall that this is a major objective of the `pre-learning' phase. We use the Eq.~\ref{eq:normalize} given in Section~\ref{sec:offpolicy}. The performance of ADMIRAL-AE using each of the advisors should lie between the maximum possible performance and the performance of ADMIRAL-AE using a random advisor. We normalize the average performances between the range of $[0,1]$, which will be used as the initial value of $\epsilon'$ (or $\epsilon'_0$) in ADMIRAL-DM. The maximum possible performance in Table~\ref{tab:mazeepsilon} is adjusted for random exploration, which is approximately 5\% of all actions. Hence, we subtract this portion from the theoretical possible maximum performance of 4000 for this domain. The initial values of $\epsilon'$ that would pertain to the advisors, are given in the last column of Table~\ref{tab:mazeepsilon}. From this table, it can be seen that the $\epsilon'_0$ value is high for Advisor~1, since there is good value to be gained in listening to this advisor. On the other hand, $\epsilon'_0$ is lower for other advisors and for Advisor~4 this value is 0, which means there will be no advisor influence.  Since Advisor 4 only provides random advice, the agent is better off listening to its own policy rather than following actions suggested by  Advisor~4. Thus, using the values given in Table~\ref{tab:mazeepsilon}, ADMIRAL-DM would listen more to the good advisor and listen less (or not at all) to the bad advisors, as was our goal.

\begin{table}
\begin{center}
 \begin{tabular}{||p{0.12 \linewidth} |p{0.15 \linewidth} |p{0.18 \linewidth} |p{0.15 \linewidth} | p{0.15 \linewidth}||} 
 \hline
 Advisor & Average cumulative reward (rounded to nearest 10) & Maximum possible cumulative reward (adjusted for random exploration) & Average performance of ADMIRAL-AE using a random advisor (rounded to nearest 10) & Normalized value (rounded up to nearest first decimal) \\ [0.5ex] 
 \hline\hline
 Advisor~1 & 3560 & 3800 & 930 & 1 \\ 
 \hline
 Advisor~2 & 1700 & 3800 & 930 & 0.3 \\
 \hline
 Advisor~3 & 1030 & 3800 & 930 & 0.1  \\
 \hline
 Advisor~4 & 930 & 3800 & 930 & 0  \\[1ex] 
 \hline
\end{tabular}
\caption{Finding $\epsilon'_0$ using ADMIRAL-AE for the Grid Maze environment.}
\label{tab:mazeepsilon}
\end{center}
\end{table}

\begin{figure}
    \centering
	\subfloat[Performance of ADMIRAL-DM in the Grid Maze domain.
	]{{\includegraphics[width=0.45\textwidth, height=5cm]{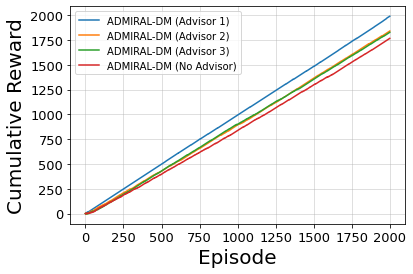} }} \quad
	\subfloat[Mean Square Error (MSE), between current $Q$-value of ADMIRAL-DM (using Advisor~1) and $Q_{*}$, which is the Nash $Q$-value. 
	]{{\includegraphics[width=0.45\textwidth, height=5cm]{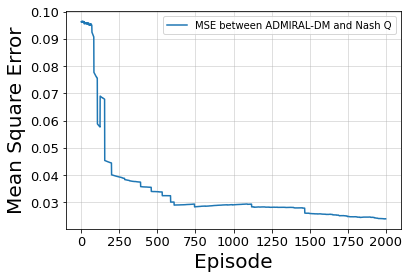}}}
    \caption{Experimental findings on the Grid Maze domain with both agents playing a tabular implementation of ADMIRAL-DM with different advisors with hyperparameter $\epsilon'_0$ obtained from Table~\ref{tab:mazeepsilon}. (a) shows that using a tuned value for $\epsilon'_0$, gives a good performance for all the implementations of ADMIRAL-DM with the different advisors. However, better advisors help in getting a relatively better performance. (b) shows that the MSE between the current $Q$-estimate of ADMIRAL-DM using Advisor~1, and the Nash $Q$-value progressively reduces. All results show an average of five runs, and they have negligible standard deviation.}%
	\label{fig:gridmazesarsa}
\end{figure}

Next, we show an illustration of the tabular implementation of our ADMIRAL-DM algorithm (Algorithm~\ref{alg:advisorQ2}). We use the same Grid Maze domain along with the same four advisors for this study. In this setting, we have the ADMIRAL-DM algorithm training along with each of the advisors for 2000 episodes. In each implementation, both agents use the same algorithm for training (ADMIRAL-DM with an advisor) similar to the previous experiments. We set the initial value of $\epsilon'$ (or $\epsilon'_0$) to the values obtained from ADMIRAL-AE (given in Table~\ref{tab:mazeepsilon}). As described, since the Advisor~4 is bad, the value of $\epsilon'_0$ is set to 0 and there is no influence from this advisor during learning. Further, for all the implementations, the advisor influence through $\epsilon'$ is linearly decayed during training, as described in the algorithmic steps for ADMIRAL-DM. More details regarding the game conditions and reward functions are given in Appendix~\ref{appendix:experimentaldetails}. The results (cumulative rewards) are in Figure~\ref{fig:gridmazesarsa}(a), where we plot the averages of five experimental runs. We note that learning from the best advisor (Advisor~1) using ADMIRAL-DM obtains the best overall performance, while ADMIRAL-DM with the other advisors requires more episodes to obtain similar performances to that of the Advisor~1. Since Advisor~1 teaches useful strategies, ADMIRAL-DM using Advisor~1 sees a good performance early on in training, even when there have been only limited interactions with the environment. This shows the value of positive influences from good advisors for improving the sample efficiency of MARL algorithms.   

Now, we  show that the $Q$-values of an agent following the ADMIRAL-DM algorithm converges to the Nash $Q$-value of the stochastic game. In Figure~\ref{fig:gridmazesarsa}(b) we plot the MSE between the $Q$-values (of every state and joint action) of ADMIRAL-DM using the Advisor~1, and the Nash $Q$-value. To obtain the Nash $Q$-value we construct the Nash policy 
and obtain the value of this policy by running trajectories from each state and joint action pair till the end of the episode and calculating the expected discounted sum of rewards (similar to obtaining the value of the advisor in the previous experiment). In this environment, the Nash equilibrium strategies will provide the actions of perfect coordination that obtains large positive rewards. The MSE in Figure~\ref{fig:gridmazesarsa}(b) approaches very close to zero after 2000 episodes of training, showing that the $Q$-values following ADMIRAL-DM finds the Nash $Q$-value in the limit.

%

\subsection{Experimental Results - Function Approximation - ADMIRAL-AE}\label{sec:experimentswithmaeqlee}

 We present results for our advisor evaluation algorithm (Algorithm~\ref{alg:advisorQ}) on the large state-action Pommerman environment \citep{resnick2018pommerman}. The objective is to conduct the `pre-learning' phase to evaluate a set of advisors and pick a suitable $\epsilon'_0$ for learning using ADMIRAL-DM, which we  study in the upcoming sub-sections. We  use a two-agent version of Pommerman, which we denote as Domain~OneVsOne of Pommerman (we will consider another domain of Pommerman shortly). Pommerman is a complex multi-agent domains, with each state containing more than 200 elements describing the position of the board, special features like bombs, and the position of other agents. Each agent can perform 6 actions, which include moving in the grid and laying bombs to kill the opponent. The reward function in Pommerman is quite sparse with the agents getting a +1 for winning the game, -1 for losing or a draw, with nothing in between. This game is general-sum since both agents get -1 for a draw. There is a maximum of 800 steps and the games where there are no winners after 800 steps are declared to be a draw. It is very hard for RL agents to learn good performance in Pommerman due to difficulties in balancing the twin goals of the killing of opponent and protecting themselves \citep{gao2019hard}.

\begin{figure}[h]
    \centering
	\subfloat[Advisor 1]{{\includegraphics[width=0.45\textwidth, height=4cm]{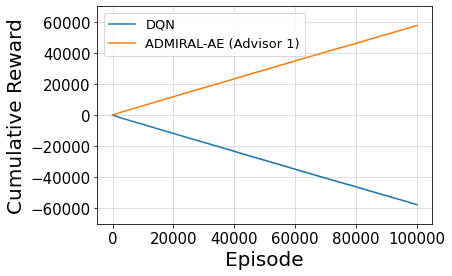} }}
	\subfloat[Advisor 2]{{\includegraphics[width=0.45\textwidth, height=4cm]{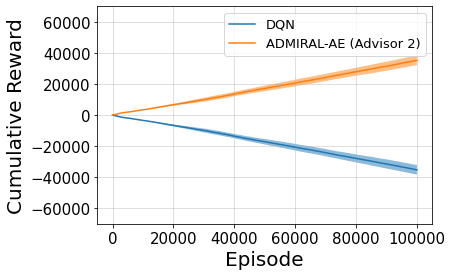} }}
	\\
	\subfloat[Advisor 3]{{\includegraphics[width=0.45\textwidth, height=4cm]{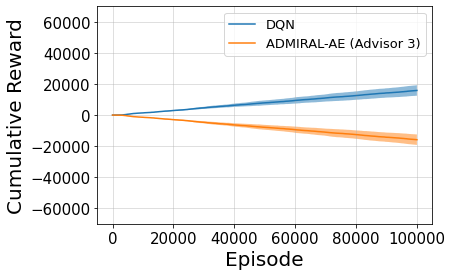} }}
	\subfloat[Advisor 4]{{\includegraphics[width=0.45\textwidth, height=4cm]{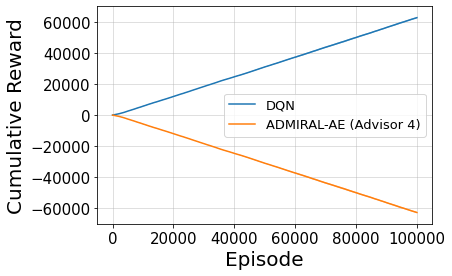} }}
  \caption{Analysis of ADMIRAL-AE algorithm on Pommerman (Domain~OneVsOne) against (single-agent) DQN. The standard deviation in (a) and (d) are very small (negligible). The best advisor (Advisor 1) makes the agent reach the best overall performance. The performance steadily decreases from Advisor 1 to Advisor 4. All results are averages of 10 experimental runs ((a) and (d) have negligible standard deviations).
  }%
	\label{fig:pommermanoffpolicy}
\end{figure}

For these experiments, we use the neural network implementation of the ADMIRAL-AE algorithm (Algorithm~\ref{alg:maeqlae}). We  consider four different advisors. 
Advisors are arranged in decreasing quality of advice, with Advisor 1 providing the best advice and Advisor 4 proving the worst (e.g. random advice).
 The Advisors~1 and 2 have a positive influence on learning, as they can teach many useful techniques to win the game, while Advisor~4 has a negative influence on learning. The Advisor~3 is also capable of teaching some useful strategies and in general, is better than a random advisor (Advisor~4). However, it is much worse compared to Advisor~1 or Advisor~2.  Appendix~\ref{appendix:experimentaldetails} contains  complete descriptions of the advisors and the implementation details of the algorithms used.

We conduct all experimental runs on 100,000 Pommerman games (episodes). Each episode is a full Pommerman game containing a maximum of 800 steps. Each experiment has a DQN (single-agent version as introduced in \citet{mnih2015human}) agent and an ADMIRAL-AE agent training and competing against each other. The experiments analyze the performance of ADMIRAL-AE with each of the four advisors against the common opponent (DQN). The performance is plotted in Figure~\ref{fig:pommermanoffpolicy}. We repeat the experiments 10 times and plot the averages and standard deviations. We observe that using the best advisor (Advisor 1) clearly results in the best performance of the ADMIRAL-AE algorithm, with an   overall cumulative reward reaching around 60,000 (Figure~\ref{fig:pommermanoffpolicy}(a)). The second-best advisor (Advisor 2) results in cumulative reward around 35,000 (Figure~\ref{fig:pommermanoffpolicy}(b)). When the ADMIRAL-AE uses Advisor 3 and Advisor 4, DQN results in  better performance cumulatively than the ADMIRAL-AE algorithm (Figures~\ref{fig:pommermanoffpolicy}(c) and (d)). 
The results (Figure~\ref{fig:pommermanoffpolicy}) show that the ADMIRAL-AE algorithm can distinguish between different quality advisors. From Figure~\ref{fig:pommermanoffpolicy}, it is clear that the advisor of choice for learning in this domain is Advisor~1. This result is obtained by running a separate instance of the ADMIRAL-AE algorithm with each of the advisors. This is consistent with our description of possible ways of evaluating the advisors using the ADMIRAL-AE algorithm in Section~\ref{sec:offpolicy}. 

In Table~\ref{tab:pommermanvsdqnepsilon} we tabulate the results and normalize the average performances to obtain a suitable value for $\epsilon'_0$ (using Eq.~\ref{eq:normalize}). The procedure is the same as that adopted in Section~\ref{sec:gridworldappendix}. We adjust the column for maximum possible performance value for 10\% random exploration as done in Table~\ref{tab:mazeepsilon}. The Advisor~1 along with its initial value of $\epsilon'$ is used in the next sub-section for learning using the ADMIRAL-DM method.

\begin{table}[h]
\begin{center}
 \begin{tabular}{||p{0.12 \linewidth} |p{0.15 \linewidth} |p{0.18 \linewidth} |p{0.15 \linewidth} | p{0.15 \linewidth}||} 
 \hline
 Advisor & Average cumulative reward (rounded to nearest 1000) & Maximum possible cumulative reward (adjusted for random exploration) & Average performance of the  ADMIRAL-AE using a random advisor (rounded to nearest 1000) & Normalized value (rounded up to nearest first decimal)\\ [0.5ex]
 \hline\hline
 Advisor~1 & 58000 & 90000 & -63000 & 0.8 \\ 
 \hline
 Advisor~2 & 35000 & 90000 & -63000 & 0.7  \\
 \hline
 Advisor~3 & -16000 & 90000 & -63000 & 0.4 \\
 \hline
 Advisor~4 & -63000  & 90000 & -63000 & 0  \\[1ex] 
 \hline
\end{tabular}
\caption{Finding $\epsilon'_0$ using ADMIRAL-AE for the Pommeran Domain~OneVsOne against DQN.}
\label{tab:pommermanvsdqnepsilon}
\end{center}
\end{table}

\subsection{Experimental Results - Function Approximation - ADMIRAL-DM}\label{sec:experiments}

 We now show that it is possible to extend our tabular ADMIRAL-DM method to function approximation based implementations that make our algorithms more generally applicable to environments with large state-action spaces. We will use the neural network-based version of ADMIRAL-DM as discussed in Section~\ref{sec:nnimplementation}. Additionally, we show that ADMIRAL-DM is capable of outperforming several strong baselines from literature.

We perform comparative experiments in three domains. All our experiments in this section are repeated 30 times, and we plot the mean and standard deviation. The important elements of our experimental domains are mentioned here, while the complete details of the domains and implementation details of all algorithms are in Appendix~\ref{appendix:experimentaldetails}. Neural network implementations of decision-making algorithms (ADMIRAL-DM, ADMIRAL-DM(AC)) are used in this sub-section. The first domain we consider is Domain~OneVsOne of Pommerman introduced in Section~\ref{sec:experimentswithmaeqlee}. Our baselines are DQfD, CHAT, and DQN. We perform 50,000 episodes of training, where the algorithms train against specific opponents. Each episode is a full Pommerman game (lasting a maximum of 800 steps).
All the algorithms relying on demonstrations (DQfD, CHAT, ADMIRAL-DM, and ADMIRAL-DM(AC)) use the Advisor~1 considered in Section~\ref{sec:experimentswithmaeqlee}. The probability of using the advisor action ($\epsilon_t'$ in Algorithm~\ref{alg:advisorQ2}) starts from 0.8 (obtained from Table~\ref{tab:pommermanvsdqnepsilon}) and linearly decays to be close to zero at the end of training for both ADMIRAL-DM and ADMIRAL-DM(AC). To provide data for offline pretraining in the case of DQfD, two instances of Advisor~1 is used to play many Pommerman games that generate the required data. The DQfD is pretrained with all of this data, before entering the training phase of our experiments.

\begin{figure}[h]
    \centering
	\subfloat[ADMIRAL-DM vs DQN]{{\includegraphics[width=0.45\textwidth, height=3cm]{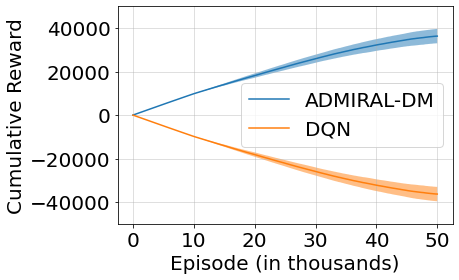} }}
	\subfloat[ADMIRAL-DM vs DQfD]{{\includegraphics[width=0.45\textwidth, height=3cm]{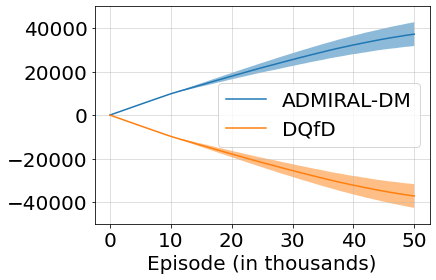} }}
	\\
	\subfloat[ADMIRAL-DM vs CHAT]{{\includegraphics[width=0.45\textwidth, height=3cm]{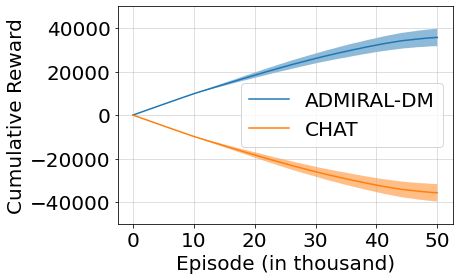} }}
	\subfloat[ADMIRAL-DM vs ADMIRAL-DM(AC)]{{\includegraphics[width=0.45\textwidth, height=3cm]{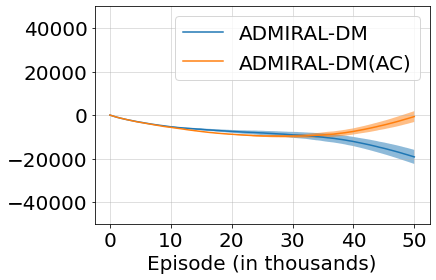} }}
	\\
	\subfloat[Faceoff against ADMIRAL-DM]{{\includegraphics[width=0.8\textwidth, height=3cm]{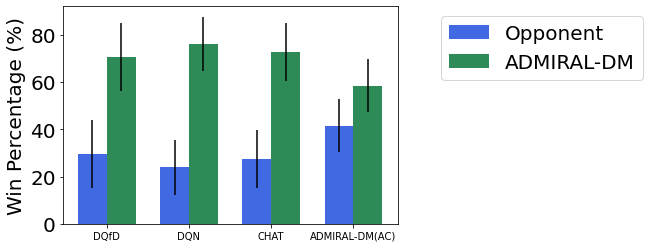} }}
  \caption{Pommerman competition against ADMIRAL-DM. The ADMIRAL-DM beats all baselines in the execution phase. In the training phase ADMIRAL-DM(AC) performs better than ADMIRAL-DM in a head-to-head challenge. 
  }%
	\label{fig:pommermantraining}
\end{figure}
 
 After the training phase, the trained algorithms enter a face-off competition of 10,000 games where there is no more training, no further exploration and additionally ADMIRAL-DM and ADMIRAL-DM(AC) play without any advisor influence. ADMIRAL-DM(AC) is a CTDE technique, which 
 only performs decentralized execution in face-off using the trained actor-network.
 We plot the cumulative rewards in the training phase (Figure~\ref{fig:pommermantraining} (a), (b), (c), (d)), from which it can be seen that ADMIRAL-DM's performance is better than the baselines (DQN, DQfD, and CHAT). The face-off plots in Figure~\ref{fig:pommermantraining}(e) show that ADMIRAL-DM wins more games on average against all the other baselines, showing its dominance. DQfD relies on pretraining, which is harder in MARL, as the nature of opponents that an agent will face during competition is impossible to determine upfront. The algorithms that use online advisors to give real-time feedback (capturing the changing nature of the opponent) tend to do better. DQfD has also been previously reported to have over-fitting issues \citep{gao2018reinforcement}, which is likely to hurt its performance more in multi-agent environments as compared to single-agent environments. In multi-agent environments, it is more important to be able to generalize to unseen dynamic opponent behaviour, which is different from that seen in pre-collected demonstration data. As discussed previously, CHAT maintains a confidence measure on the advisor, which depends on the advisor's consistency in action recommendations at different states. In MARL, this measure is not completely reliable, since even good advisors may need to formulate stochastic action recommendations as responses to the opponent. DQN, on the other hand, learns directly from interaction experiences and cannot learn from advisor inputs. This is a disadvantage in environments where external sources of knowledge, such as advisors, are available to be leveraged. 
 Furthermore, since our baselines are independent algorithms (that consider opponents to be part of the state), they lose out to ADMIRAL-DM, which explicitly tracks opponent action. ADMIRAL-DM loses to ADMIRAL-DM(AC) during training (Figure~\ref{fig:pommermantraining} (d)). Though ADMIRAL-DM(AC) shows slower learning overall  (as it is training both actor and critic), it ultimately learns a higher performing policy. One important reason is that the actor-critic method trains a stochastic policy that can explore naturally, whereas the $Q$-learning method needs a hyperparameter to conduct forced exploration ($\epsilon$-greedy). Another reason could be that ADMIRAL-DM(AC) learns from each recent experience, while ADMIRAL-DM has delayed learning using the replay buffer. However, in the face-off, ADMIRAL-DM has an edge over ADMIRAL-DM(AC) (Figure~\ref{fig:pommermantraining}(e)), probably due to being centralized. Since the performance of ADMIRAL-DM and ADMIRAL-DM (AC) in the face-off results given in Figure~\ref{fig:pommermantraining}(e) are close, we perform a Fischer's exact test for the average performances to check statistical significance. We get p < 0.03 which shows that this result is statistically significant (we treat p < 0.05 as statistically significant as in common practice).

\begin{figure}[h]
    \centering
	\subfloat[ADMIRAL-DM(AC) vs DQN]{{\includegraphics[width=0.45\textwidth, height=4cm]{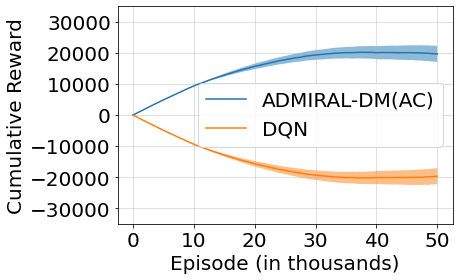} }}
	\subfloat[ADMIRAL-DM(AC) vs DQfD]{{\includegraphics[width=0.45\textwidth, height=4cm]{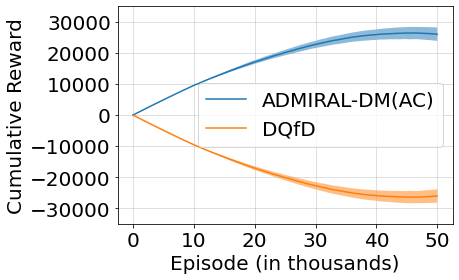} }}
	\\
	\subfloat[ADMIRAL-DM(AC) vs CHAT]{{\includegraphics[width=0.45\textwidth, height=4cm]{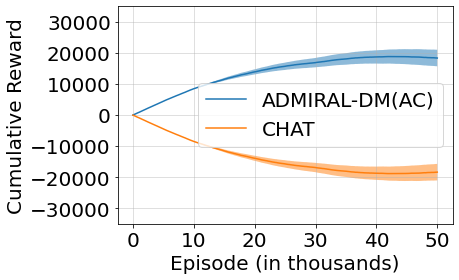} }}
	\quad
	\subfloat[Faceoff against ADMIRAL-DM(AC)]{{\includegraphics[width=0.5\textwidth, height=4cm]{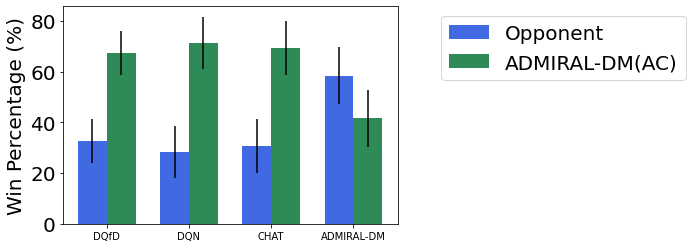} }}
  \caption{Pommerman competition against ADMIRAL-DM(AC). ADMIRAL-DM(AC) defeats all other baselines in both the training and execution phases. 
  }%
	\label{fig:pommermantraining-actorcritic}
\end{figure}

Next, we conduct similar experiments with ADMIRAL-DM(AC) that show it outperforms the baselines. The ADMIRAL-DM(AC) is explicitly compared to all the baselines in a training and face-off scheme similar to that done with ADMIRAL-DM. To recall, we perform training experiments of 50,000 full Pommerman games and face-off contests of 10,000 games, where the trained agents compete against each other without any further training or advisor influence. The results are plotted in Figure~\ref{fig:pommermantraining-actorcritic}, where the ADMIRAL-DM(AC) shows better performance than the baselines (in both training and face-off phases). In training,  ADMIRAL-DM(AC) dominates all the other baselines by winning around 20,000 -- 30,000 games out of the 50,000 games conducted. As observed in the previous experiments, the ADMIRAL-DM(AC) algorithm's learning is slower than that of the $Q$-learning variant, making the overall number of games won (captured by cumulative rewards), against the baselines to be lower than that of the corresponding training of ADMIRAL-DM against the baselines in Figure~\ref{fig:pommermantraining}. In the face-off contests, the ADMIRAL-DM(AC) algorithm wins more than 50\% of the games against all baselines except ADMIRAL-DM. As noted previously, the ADMIRAL-DM algorithm has a slight edge in performance over that of ADMIRAL-DM(AC) in the face-off stage.

Next, we use two cooperative domains from the Stanford Intelligent Systems Laboratory (SISL) \citep{gupta2017cooperative}. These experiments  have two phases --- training and execution. The algorithms train for 1000 games in the training phase and then enter an execution phase, where they execute the trained policy for 100 games. We choose to set the value of advisor influence $\epsilon_t'$ to 0.8 at the start of training and linearly decay it the same way as in the above experiments with Pommerman (since we are using good advisors). In the execution phase, there is no further exploratory actions for all algorithms, and no more influence of the advisor for ADMIRAL-DM and ADMIRAL-DM(AC).

\begin{figure}[h]
    \centering
	\subfloat[Pursuit-Training]{{\includegraphics[width=0.49\textwidth, height=4cm]{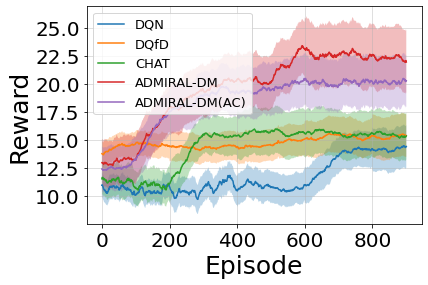} }}
	\subfloat[Pursuit-Execution]{{\includegraphics[width=0.49\textwidth, height=4cm]{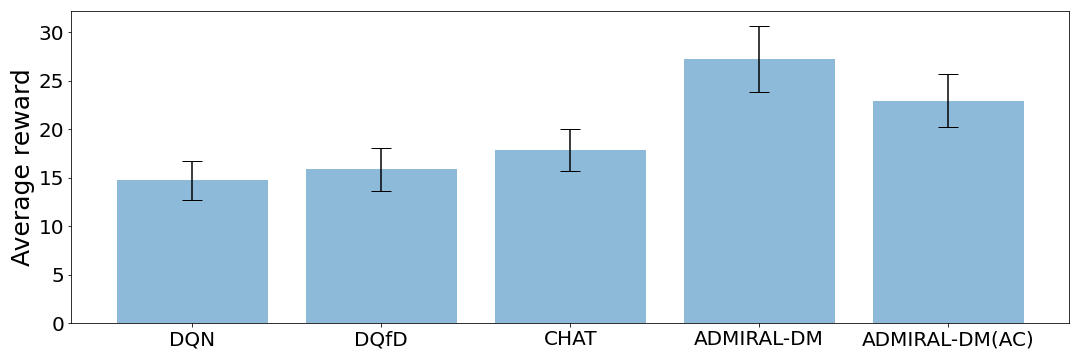} }}
	\\
	\subfloat[Waterworld - Training]{{\includegraphics[width=0.49\textwidth, height=4cm]{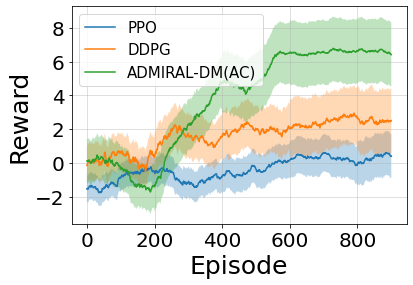} }}
	\subfloat[Waterworld - Execution]{{\includegraphics[width=0.49\textwidth, height=4cm]{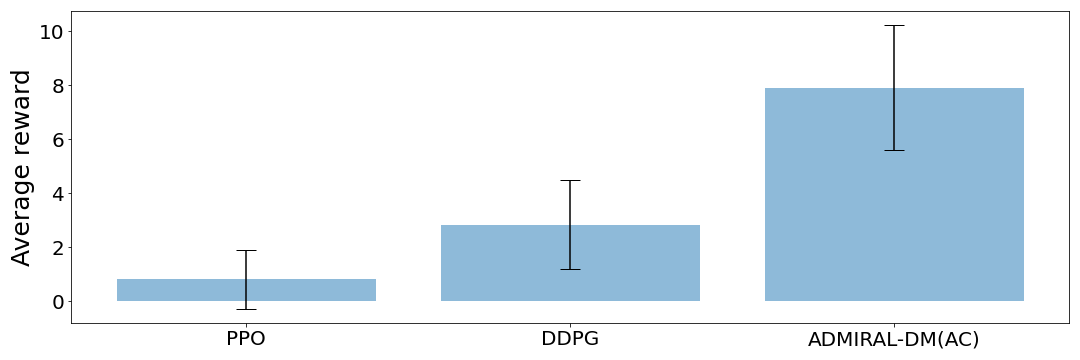} }}
	
  \caption{SISL Environments - Training and Execution. The ADMIRAL algorithms give a better performance than all the baselines in both the phases. Training graphs have been smoothed with a running average of 100. 
  }%
	\label{fig:pettingzoo}
\end{figure}

The first SISL environment is a Pursuit environment, which contains 8 pursuers controlled by learning algorithms, trying to capture 30 evaders moving randomly in the grid-based environment. Rewards have a local structure, where the pursuers participating in the capture of an evader or the pursuers encountering evaders are rewarded individually. The game is general-sum and does not have a global reward structure. The local reward structure helps to tackle the issue of credit assignment.
We use a pre-trained policy of DQN (trained for 1000 episodes) as the advisor. We plot the reward obtained (averaged per agent) for the training and execution phases in Figures~\ref{fig:pettingzoo}(a) and (b). The execution performance bars in Figure~\ref{fig:pettingzoo}(b) is the average performance across the 100 execution games. The results show that ADMIRAL-DM has  better performance than all other baselines, including DQN used as the advisor, in both phases. Thus, our algorithm can ultimately outperform the advisor. This environment is highly non-stationary (due to having more number of learning agents), so completely centralized ADMIRAL-DM has an edge over ADMIRAL-DM(AC) which uses decentralized actors. 

Since some performances in Figure~\ref{fig:pettingzoo}(b) are close, we perform an unpaired 2-sided $t$-test for statistical significance. Regarding the performances of CHAT and ADMIRAL-DM(AC) we get a value of $p < 0.02$ and similarly for the performances of ADMIRAL-DM and ADMIRAL-DM(AC) we get a $p < 0.02$, which shows that both these comparisons are statistically significant.\footnote{We consider $p < 0.05$ as statistically significant.}

Our second SISL environment is the continuous action space Waterworld environment, which has 5 pursuer agents trying to consume food and avoid poison. The actions are continuous-valued thrust, that the agents can apply to move in a particular direction, and with a desired speed. Here, multiple pursuer agents need to work together to consume food. 
Agents get rewards based on foods captured and punishments based on poison consumed. Rewards have a local structure similar to the Pursuit environment. 
The advisor here is a pre-trained proximal policy optimization (PPO) \citep{schulman2017proximal} agent (trained for 1000 episodes). 
Similar to the Pursuit environment, we plot the performances for both training and execution phases (averaged per agent) in the Waterworld environment (see Figures~\ref{fig:pettingzoo}(c) and (d)). 
The ADMIRAL-DM(AC) alone is used for these experiments, as the $Q$-learning variant is not applicable for continuous action spaces. 
We use two popular RL algorithms for continuous control, PPO and deep deterministic policy gradients (DDPG) \citep{lillicrap2015continuous} as baselines. 
The results show that ADMIRAL-DM(AC) has better performance than others in both phases (Figure~\ref{fig:pettingzoo}(c) and (d)). The ADMIRAL-DM(AC) algorithm has two important advantages over the other baselines here. The first advantage is that it is capable of leveraging an advisor.  The second advantage is that ADMIRAL-DM(AC) is trained in a centralized fashion by tracking the opponent behaviour while the other algorithms are independent methods. Still, ADMIRAL-DM(AC) is decentralized in execution. Notably, the ADMIRAL-DM(AC) algorithm also improves upon PPO, used as the advisor, similar to our observation in the Pursuit environment. 

In both the above SISL experiments, we see a small improvement in performance in the execution phases for both algorithms, ADMIRAL-DM and ADMIRAL-DM(AC), as compared to the final training performances. This is due to the fact that, at the end of the training, there is still a small amount of exploration and advisor influence (1 \% of actions) involved, whereas during execution both these influences are completely removed, which contributes to a net improvement in performance.

To summarize, our experimental results show that the ADMIRAL-DM and ADMIRAL-DM(AC) algorithms  make the best use of advisors in multi-agent settings compared to the other state-of-the-art algorithms. After the advisor influence completely stops, the performance of ADMIRAL-DM and ADMIRAL-DM(AC) is better than the others. We have also demonstrated that our methods can be extended to continuous action spaces and work in decentralized environments using the popular CTDE technique.


\subsection{Performance of ADMIRAL-DM Under The Influence Of Different Advisors}\label{sec:pommerman}

Next we  study the impact of using the ADMIRAL-AE in a `pre-learning' phase to determine the value of $\epsilon'_0$. Towards the same, we would like to use an algorithm to serve as a common opponent. We choose to use a different algorithm compared to the baselines considered in the previous sub-section (where the objective was to show better performances of ADMIRAL-DM as compared to these baselines). The algorithm we choose to use as the opponent is Deep Sarsa, which is similar to DQN but uses a ``Sarsa-like'' \citep{sutton1998introduction} Bellman update for the $Q$-values. We clarify that our objective in this section is not to show better performances against any baseline (which we have already done in Section~\ref{sec:experiments}). 

Further, in this section, we provide additional experiments that evaluate ADMIRAL-DM on different advisors with the common opponent (Deep Sarsa) and show that ADMIRAL-DM is capable of recovering from bad action-advice. All results reported in this section use averages and standard deviations of 30 runs.

\begin{figure}[h]
    \centering
	\subfloat[Advisor 1]{{\includegraphics[width=0.45\textwidth, height=4cm]{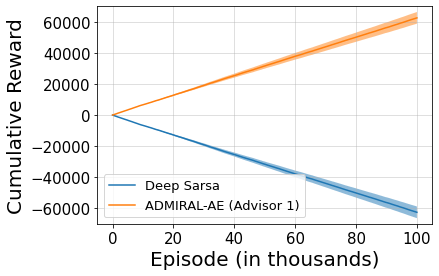} }}
	\subfloat[Advisor 2]{{\includegraphics[width=0.45\textwidth, height=4cm]{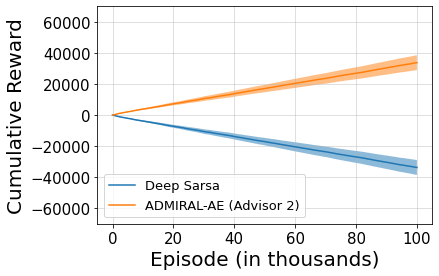} }}
	\\
	\subfloat[Advisor 3]{{\includegraphics[width=0.45\textwidth, height=4cm]{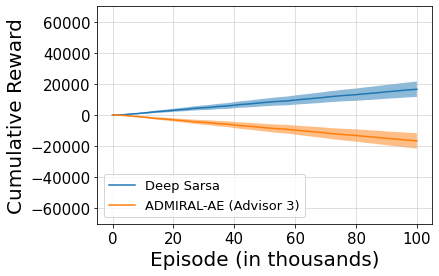} }}
	\subfloat[Advisor 4]{{\includegraphics[width=0.45\textwidth, height=4cm]{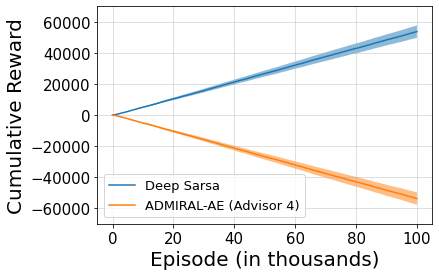} }}
  \caption{Results in Domain~OneVsOne of Pommerman using different advisors with ADMIRAL-AE and Deep Sarsa.  The best advisor (Advisor~1) gives the highest overall performance while the worst advisor (Advisor~4) gives the lowest performance. 
  }

	\label{fig:onevsoneoffpolicy}
\end{figure}

We describe two sets of experiments with two different domains of Pommerman. In the first set of experiments, we  use the neural network implementation of ADMIRAL-DM and ADMIRAL-AE on the Domain~OneVsOne of Pommerman using the four different advisors introduced in Section~\ref{sec:experimentswithmaeqlee}. To recall, Advisor~1 is the best advisor who can give the best action (relative to other advisors) at all states and Advisor~4 is the worst. The quality of advisors reduces from Advisor~1 to Advisor~4. As mentioned, we use a common opponent as an agent playing Deep Sarsa.

First, we wish to evaluate the given advisors against the performance of Deep Sarsa in the `pre-learning' phase. To do this, we run a series of training experiments, where we implement ADMIRAL-AE using each of the four advisors against Deep Sarsa. The results are plotted in Figure~\ref{fig:onevsoneoffpolicy}. As observed earlier, the best advisor leads to the best overall performance and the worst advisor leads to the worst performance. Using these performances the $\epsilon'_0$ values are determined in Table~\ref{tab:pommermanvsdeepsarsaepsilon} (using Eq.~\ref{eq:normalize}). It can be seen that the suggested value of $\epsilon'_0$ is highest (0.9) for the best advisor and the smallest (0) for the worst advisor. These values of $\epsilon'_0$ show that the agent will listen more to the good advisors and listen less (or not at all) to the bad ones.

\begin{table}[h]
\begin{center}
 \begin{tabular}{||p{0.12 \linewidth} |p{0.15 \linewidth} |p{0.18 \linewidth} |p{0.15 \linewidth} | p{0.15 \linewidth}||}  
 \hline
 Advisor & Average cumulative reward (rounded to nearest 1000) & Maximum possible cumulative reward (adjusted for random exploration) & Average performance of agent using a random advisor (rounded to nearest 1000) & Normalized value (rounded up to nearest first decimal)\\ [0.5ex] 
 \hline\hline
 Advisor~1 & 63000 & 90,000 & -54000 & 0.9 \\ 
 \hline
 Advisor~2 & 34000 & 90,000 & -54000 & 0.7  \\
 \hline
 Advisor~3 & -16400 & 90,000 & -54000 & 0.3  \\
 \hline
 Advisor~4 & -54000 & 90,000 & -54000 & 0  \\[1ex] 
 \hline
\end{tabular}
\caption{Finding an initial value for $\epsilon'_0$ using ADMIRAL-AE for the Pommeran Domain~OneVsOne against Deep Sarsa.}
\label{tab:pommermanvsdeepsarsaepsilon}
\end{center}
\end{table}

Next, we run the ADMIRAL-DM algorithm against Deep Sarsa, using each of the four advisors. We use four initial values of $\epsilon'_0$ for each of the advisors, where one of these values  corresponds to the choice of $\epsilon'_0$ as obtained from our previous experiment with ADMIRAL-AE reflected in Table~\ref{tab:pommermanvsdeepsarsaepsilon}. In addition to these four values, we also consider a value of 0 for $\epsilon'_0$, which considers the performance of ADMIRAL-DM with no advisor inputs to serve as a baseline. As done previously, the value of $\epsilon'$ is decayed linearly, during training, for ADMIRAL-DM in all the experiments. All training is conducted for 100,000 episodes with the advisor influence ($\epsilon'$) being linearly decayed to 0 at 50,000 episodes, \emph{i.e.}~there is no advisor influence after 50,000 episodes. Each episode is a complete Pommerman game involving a maximum of 800 steps as in the previous experiments. More details of the game conditions and the advisors can be found in Appendix~\ref{appendix:experimentaldetails}. The results showing the performance of ADMIRAL-DM in each of these experiments are presented in Figure~\ref{fig:onevsonecompetitionresults}.

\begin{figure}[h]   
    \centering
	\subfloat[Advisor 1 ]{{\includegraphics[width=0.45\textwidth, height=4cm]{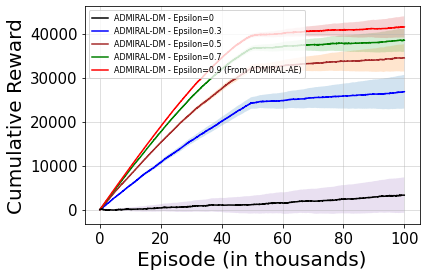} }}
	\subfloat[Advisor 2]{{\includegraphics[width=0.45\textwidth, height=4cm]{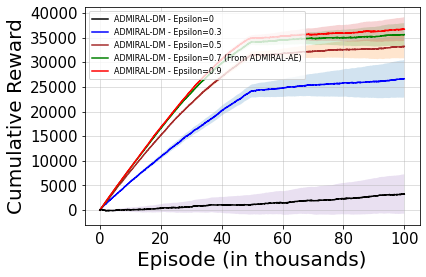} }}
	\\
	\subfloat[Advisor 3]{{\includegraphics[width=0.45\textwidth, height=4cm]{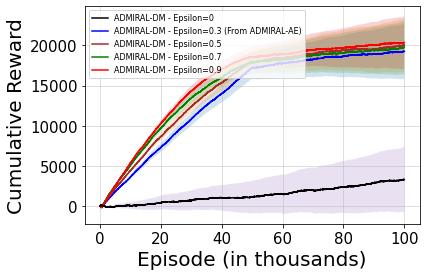} }}
	\subfloat[Advisor 4]{{\includegraphics[width=0.45\textwidth, height=4cm]{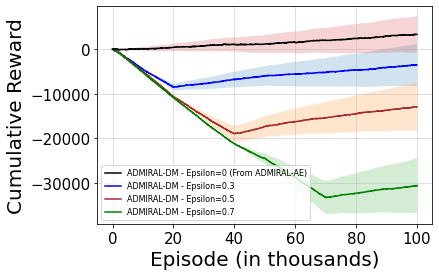} }}
  \caption{Results in Domain~OneVsOne of Pommerman using different advisors with ADMIRAL-DM and Deep Sarsa. The result plots show the performance of ADMIRAL-DM in the training against Deep Sarsa. The results show that $\epsilon'_0$ value obtained from the performance of ADMIRAL-AE in Table~\ref{tab:pommermanvsdeepsarsaepsilon} show either the best performance or is very close to the best possible performances amongst all the $\epsilon'_0$ values. 
  }

	\label{fig:onevsonecompetitionresults}
\end{figure}

\begin{figure}[h]
\centering
	\subfloat[OneVsOne - $\epsilon'=0.3$ ]{{\includegraphics[width=0.45\textwidth, height=4cm]{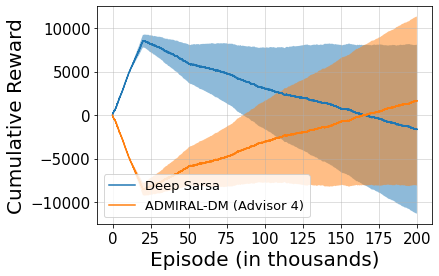} }}
	\subfloat[OneVsOne - $\epsilon'=0.5$]{{\includegraphics[width=0.45\textwidth, height=4cm]{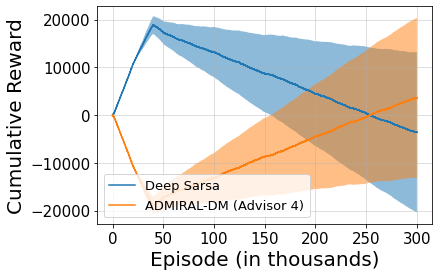} }}
	\\
	\subfloat[OneVsOne - $\epsilon'=0.7$]{{\includegraphics[width=0.45\textwidth, height=4cm]{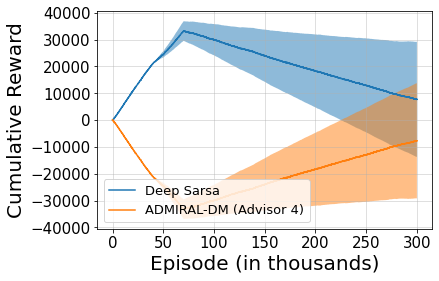} }}

  \caption{Results of ADMIRAL-DM vs Deep Sarsa using the Advisor~4, which is the worst advisor among the ones considered. The plots correspond to the OneVsOne domain. All the figures show that, ADMIRAL-DM is capable of recovering from bad action advice.  Greater influence of the bad advisor (larger $\epsilon'_0$) leads to a larger time needed for recovering from the bad action influence. 
  }%
	\label{fig:advisor4performanceonevsone}
\end{figure}

We highlight several observations. Figure~\ref{fig:onevsonecompetitionresults}(a) and (b) show that ADMIRAL-DM, using the good advisors (Advisor~1 and Advisor~2), achieves the maximum overall performance since the positive influence from the good advisors helps. However, if the advisor influence is limited ($\epsilon'=0.3$), the good advisors  only have a limited impact. As the value of $\epsilon'_0$ increases, we see that the performance of ADMIRAL-DM using the first two advisors improves (Figures~\ref{fig:onevsonecompetitionresults}(a) and (b)), as expected. Since Advisor~1 is even better than Advisor~2, we see from Figures~\ref{fig:onevsonecompetitionresults}(a) and (b) that ADMIRAL-DM using Advisor~1 clearly shows superior performance to that of Advisor~2 for the highest value of $\epsilon'_0$, 0.9. When $\epsilon'_0$ values were lower (such as 0.5), performance using both Advisor~1 and Advisor~2 are comparable since these advisors did not have many opportunities to make an impact. Notably, ADMIRAL-DM using Advisor~1 shows the best performance for $\epsilon'=0.9$, the value obtained from ADMIRAL-AE as seen in Table~\ref{tab:pommermanvsdeepsarsaepsilon}. ADMIRAL-DM using Advisor~2 almost provides the same performances for the highest values of $\epsilon'_0$, 0.7 and 0.9. Additional inputs from this advisor are not as useful (compared to Advisor~1), since it is weaker. Hence, a value of $\epsilon'=0.7$ as obtained from ADMIRAL-AE is sufficient for this advisor.

Turning our attention to the performance of ADMIRAL-DM with the third advisor, we find that it is significantly inferior compared to the other two advisors, yet still has a limited positive influence (Figure~\ref{fig:onevsonecompetitionresults}(c)). Furthermore,  the performance using this advisor is considerably better than using no advisor at all ($\epsilon'=0$). However, while using Advisor~3, we notice that, as the values of $\epsilon'_0$ increases, there is no appreciable improvement in performance. This shows that more influence of a comparatively less effective advisor does not lead to much improvement in performance. Again, the value suggested by ADMIRAL-AE (0.3), comes very close to the best possible performance with other values of $\epsilon'_0$. 

The performance of ADMIRAL-DM using the last advisor (Advisor~4) is interesting. This advisor has a negative influence on learning and makes ADMIRAL-DM lose  for the first few episodes (Figure~\ref{fig:onevsonecompetitionresults}(d)). However, ADMIRAL-DM recovers after the advisor influence wanes in all cases. As the value of $\epsilon'_0$ increases, we see that further influence from the bad advisor is  detrimental and a larger number of episodes is required before ADMIRAL-DM shows signs of recovery from the poor advice. Hence, the best value of $\epsilon'_0$, in this case, is 0, since listening to this advisor only harms learning. Again, this was the value obtained for Advisor~4, in Table~\ref{tab:pommermanvsdeepsarsaepsilon}. We present a more elaborate set of results on the experiments with Advisor~4 in Figure~\ref{fig:advisor4performanceonevsone}. Here we show that for all cases of $\epsilon'_0$, ADMIRAL-DM is capable of recovering from bad action-advising and after a suitable number of episodes, can overtake the performance of Deep Sarsa. However, higher values of $\epsilon'_0$ makes the learning from the bad advisor  problematic, since while ADMIRAL-DM shows signs of recovery it still cannot overtake the cumulative performance of Deep Sarsa even after 300,000 episodes (Fig:~\ref{fig:advisor4performanceonevsone}(c)).

\begin{figure}[h]
\centering
	\subfloat[Advisor 1]{{\includegraphics[width=0.45\textwidth, height=4cm]{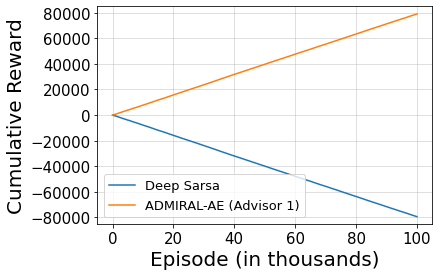} }}
	\subfloat[Advisor 2]{{\includegraphics[width=0.45\textwidth, height=4cm]{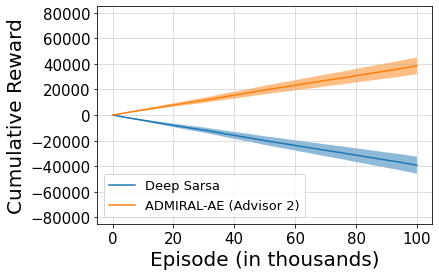} }}
	\\
	\subfloat[Advisor 3]{{\includegraphics[width=0.45\textwidth, height=4cm]{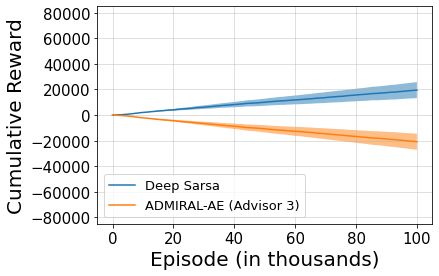} }}
	\subfloat[Advisor 4]{{\includegraphics[width=0.45\textwidth, height=4cm]{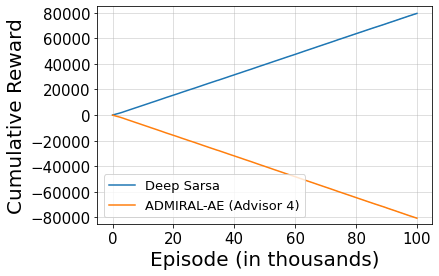} }}
  \caption{Results in Domain~TwoVsTwo of Pommerman using different advisors with ADMIRAL-AE and Deep Sarsa.  The best advisor (Advisor~1) gives the highest overall performance while the worst advisor (Advisor~4) gives the lowest performance. The standard deviation of (a) and (d) are negligible.
  }%
	\label{fig:teamcompetitionoffpolicy}
\end{figure}

We  now provide a brief description of another domain of Pommerman. Domain~TwoVsTwo is a larger version of Pommerman, where there are a total of four agents, with two of the four belonging to the same team. The state space is much larger than Domain~OneVsOne, with each state containing 372 elements. The reward function remains sparse, with the two agents belonging to the winning team getting +1 and the two agents belonging to the losing team getting -1 at the end of the game. In case of a draw, all the agents get -1. In Domain~TwoVsTwo, we consider one team of Deep Sarsa and one team of the ADMIRAL-DM. This makes this domain harder than the Domain~OneVsOne, as the agents must learn to cooperate amongst the members of the same team and compete against the members of the opponent team to win the game.  The Domain~TwoVsTwo is a mixed competitive-cooperative domain, which is different from Domain~OneVsOne that had only pure competition. We  use the same four advisors as considered before for the Domain~TwoVsTwo.

 Similar to the experiments with Domain~OneVsOne, we first evaluate the advisors against Deep Sarsa using the ADMIRAL-AE algorithm. The results are presented in Figure~\ref{fig:teamcompetitionoffpolicy}. Again, the best advisor gives the maximum overall performance and the worst advisor results in  the minimum performance.  The $\epsilon'_0$ values are determined based on these results in Table~\ref{tab:pommermanteamcompetitionvsdeepsarsaepsilon} (using Eq.~\ref{eq:normalize}). These values are used in further experiments using ADMIRAL-DM and Deep Sarsa in the Domain~TwoVsTwo of Pommerman.

\begin{table}[h]
\begin{center}
  \begin{tabular}{||p{0.12 \linewidth} |p{0.15 \linewidth} |p{0.18 \linewidth} |p{0.15 \linewidth} | p{0.15 \linewidth}||} 
 \hline
 Advisor & Average cumulative reward (rounded to nearest 1000) & Maximum possible cumulative reward (adjusted for random exploration) & Average performance of agent using a random advisor (rounded to nearest 1000) & Normalized value (rounded up to nearest first decimal)\\ [0.5ex]
 \hline\hline
 Advisor~1 & 79000 & 90000 & -81000 & 1 \\ 
 \hline
 Advisor~2 & 39000 & 90000 & -81000 & 0.7 \\
 \hline
 Advisor~3 & -21000 & 90000 & -81000 & 0.4  \\
 \hline
 Advisor~4 & -81000 & 90000 & -81000 & 0  \\[1ex] 
 \hline
\end{tabular}
\caption{Finding $\epsilon'_0$ using ADMIRAL-AE for the Pommeran Domain~TwoVsTwo against Deep Sarsa.}
\label{tab:pommermanteamcompetitionvsdeepsarsaepsilon}
\end{center}
\end{table}

\begin{figure}[h]
\centering
	\subfloat[Advisor 1 ]{{\includegraphics[width=0.45\textwidth, height=4cm]{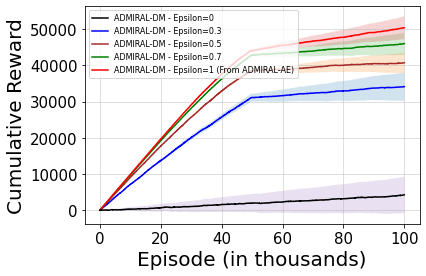} }}
	\subfloat[Advisor 2]{{\includegraphics[width=0.45\textwidth, height=4cm]{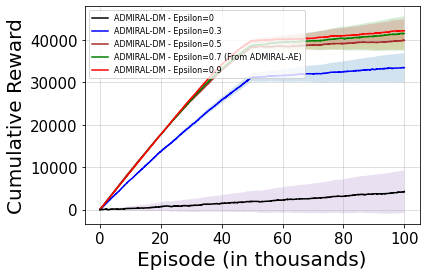} }}
	\\
	\subfloat[Advisor 3]{{\includegraphics[width=0.45\textwidth, height=4cm]{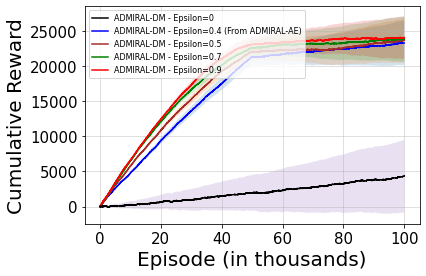} }}
	\subfloat[Advisor 4]{{\includegraphics[width=0.45\textwidth, height=4cm]{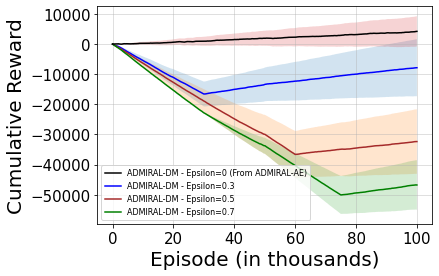} }}
  \caption{Results in Domain~TwoVsTwo of Pommerman using different advisors with ADMIRAL-DM and Deep Sarsa. The result plots show the performance of team playing ADMIRAL-DM in training competitions against Deep Sarsa. For this domain too, $\epsilon'_0$ value obtained from the performance of ADMIRAL-AE in Table~\ref{tab:pommermanteamcompetitionvsdeepsarsaepsilon} show either the best performance or is very close to the best possible performances amongst all the $\epsilon'_0$ values. 
  }

	\label{fig:teamcompetitionresults}
\end{figure}

\begin{figure}[h]
\centering

	\subfloat[TwoVsTwo - $\epsilon'=0.3$]{{\includegraphics[width=0.45\textwidth, height=4cm]{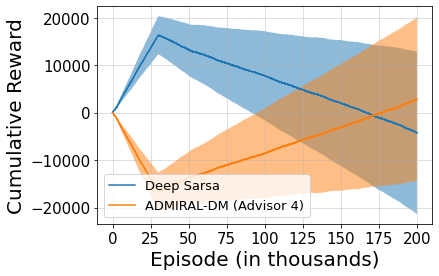} }}
	\subfloat[TwoVsTwo - $\epsilon'=0.5$]{{\includegraphics[width=0.45\textwidth, height=4cm]{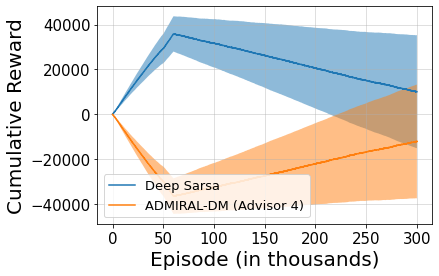} }}\\
	\subfloat[TwoVsTwo - $\epsilon'=0.7$]{{\includegraphics[width=0.45\textwidth, height=4cm]{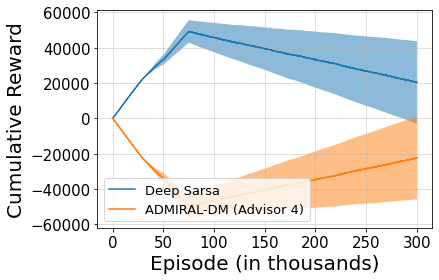} }}
	
  \caption{Results of ADMIRAL-DM vs. Deep Sarsa using the Advisor~4 in the TwoVsTwo domain. Similar to the first domain, all the figures show that ADMIRAL-DM is capable of eventually recovering from bad action advice.
  }%
	\label{fig:advisor4performancetwovstwo}
\end{figure}

 Our next set of experiments train in the Domain~TwoVsTwo with each agent in one team of Pommerman agents playing ADMIRAL-DM and each agent in the other team playing Deep Sarsa. We  consider four different initial values of $\epsilon'$ (where one of these values will correspond to the value obtained in Table~\ref{tab:pommermanteamcompetitionvsdeepsarsaepsilon}), as we did in the  OneVsOne domain. In addition to these $\epsilon'_0$ values, we  continue to use the value of $\epsilon'_0=0$ as the baseline. This corresponds to the situation of ADMIRAL-DM learning with no advisor influence. All training is run for 100,000 episodes, with the value of $\epsilon'$ decaying linearly to 0  at 50,000 episodes. Thus, there is no more influence from advisors for the last 50,000 episodes of training. The results showing the performance of the team playing ADMIRAL-DM (in the competition against a team playing Deep Sarsa) are in Figure~\ref{fig:teamcompetitionresults}. The ADMIRAL-DM performances show similar characteristics to that seen in Domain~OneVsOne. ADMIRAL-DM, using the best advisor (Advisor~1), shows the best performance for all values of $\epsilon'_0$ and its performance keeps improving with the increase in value of $\epsilon'_0$ (refer Figure~\ref{fig:teamcompetitionresults}(a)). Notably, comparing Figure~\ref{fig:teamcompetitionresults}(a) and Figure~\ref{fig:onevsonecompetitionresults}(a), the performance of ADMIRAL-DM using the best advisor is better in the case of Domain~TwoVsTwo as compared to the Domain~OneVsOne. This suggests that a good advisor has comparatively higher impact when tasks get harder. This is because there are far more strategies that need to be learned to do well in this domain and the opponent learning from scratch needs more time for learning the hard task. A good advisor, on the other hand, can teach the different strategies needed much faster and provide an early lead for ADMIRAL-DM. Similar observations apply to   the performances of ADMIRAL-DM using Advisor~2 as well (refer Figure~\ref{fig:teamcompetitionresults}(b) and Figure~\ref{fig:onevsonecompetitionresults}(b)).

 For Advisor~3, the performance does not change much with varying values of $\epsilon'_0$ (refer Figure~\ref{fig:teamcompetitionresults}(c)) as observed in the case of Domain~OneVsOne. Comparing Figure~\ref{fig:teamcompetitionresults}(c) and Figure~\ref{fig:onevsonecompetitionresults}(c), we note that the best performance of ADMIRAL-DM using  Advisor~3 is better for the Domain~TwoVsTwo as compared to Domain~OneVsOne, reinforcing our inferences earlier about a higher potential for impact in useful advisors when the tasks get harder. Regarding Advisor~4, the greater negative influence from this advisor necessitates a longer time for recovery (refer Figure~\ref{fig:teamcompetitionresults}(d)). 
 
 Again, from Figure~\ref{fig:teamcompetitionresults}, for all the four advisors, we  observe that the value for $\epsilon'_0$ as obtained from Table~\ref{tab:pommermanteamcompetitionvsdeepsarsaepsilon} gives the best possible performance or comes quite close to the best possible performance, compared to other possible values of $\epsilon'_0$. This shows the advantage of evaluation using ADMIRAL-AE. In Figures~\ref{fig:advisor4performancetwovstwo}(a), (b) and (c), ADMIRAL-DM shows signs of recovery for all values of $\epsilon'_0$ while using Advisor~4 for learning. However, for Advisor~4, the smaller the value of $\epsilon'_0$ the better. When making a comparison between  Figure~\ref{fig:advisor4performanceonevsone} and Figure\ref{fig:advisor4performancetwovstwo}, we note that, as the negative influence from the advisor increases (through a larger $\epsilon'_0$) the time needed for recovery also rises in the case of Domain~TwoVsTwo and is greater than that needed in the Domain~OneVsOne. As the complexity of the tasks increase, the agents need a lot more time to learn good policies that recovers the loss of performance from bad action-advice.  This shows that negative influence from an advisor is more costly in the case of harder MARL tasks/environments as against comparatively simpler environments.

From our experiments, we conclude that ADMIRAL-DM is capable of recovering from bad advisor recommendations and is able to suitably leverage  advisors who have some positive influence on learning. However, if possible, it is best to evaluate an advisor using the ADMIRAL-AE method and obtain a suitable initial value for the hyperparameter that determines the advisor influence ($\epsilon'_0$). This helps in learning good policies faster.

\subsection{Summary}

To summarize, in Section~\ref{sec:gridworldappendix} we showed an experimental illustration of our algorithms in a tabular domain. We showed that, while using ADMIRAL-AE, the best advisor gives the best overall performance. Further, ADMIRAL-AE provides a suitable value for the hyperparameter $\epsilon'_0$, which when used by ADMIRAL-DM subsequently, provides good performances with different types of advisors. We also provided an experimental illustration of our theoretical convergence results in the case of both ADMIRAL-DM and ADMIRAL-AE. In Section~\ref{sec:experimentswithmaeqlee} we provided an illustration of ADMIRAL-AE in a large environment with neural networks as function approximators. Again, we illustrated that using ADMIRAL-AE with the best advisor provides the best performance amongst other (comparatively worse) advisors. Obtaining an appropriate value for $\epsilon'_0$ from Section~\ref{sec:experimentswithmaeqlee} we showed that ADMIRAL-DM and ADMIRAL-DM(AC) provide better performances than a set of baselines in Section~\ref{sec:experiments}. We tested our algorithms in both competitive and cooperative domains, as well as settings with discrete and continuous action spaces.   

In Section~\ref{sec:pommerman}, we used two Pommerman domains to illustrate that using ADMIRAL-AE to obtain a suitable value for $\epsilon'_0$, provides the best performance for ADMIRAL-DM. Hence, when possible, it would be best to use ADMIRAL-AE for determining $\epsilon'_0$ using a pre-learning phase. Also, we illustrated that ADMIRAL-DM is capable of recovering from bad action advice from advisors if appropriate values for $\epsilon'_0$ cannot be determined before training ADMIRAL-DM. 

Additionally, in Appendix~\ref{sec:adaptive} we show another important advantage of using the principled method of ADMIRAL-AE for advisor evaluation in environments having dynamically learning and adapting advisors. We show that a principled method like ADMIRAL-AE would find a suitable value for $\epsilon'_0$ when it is presented with a learning advisor, where other methods based on simple heuristics may have a high chance of failure.


\section{Conclusion}

In this paper, we introduced the problem of learning under the influence of external advisors in MARL. We provided a  principled framework for MARL algorithms learning to use advisors.
Using $Q$-learning based methods, we proposed two MARL algorithms for this problem. We conducted theoretical analyses of these algorithms, establishing conditions under which fixed point guarantees can be provided regarding their learning in general-sum stochastic games. We  proved that previous theoretical results can extend to this setting under a comparatively weaker set of assumptions than previously considered. Empirically, we showed that our algorithms can be scaled to domains with large state-action spaces using traditional function approximators like neural networks. We also introduced an additional actor-critic variant of our ADMIRAL-DM algorithm that can operate under the CTDE paradigm and can learn in environments with continuous action spaces. Our empirical results further established the superiority of our algorithms compared to standard baselines. Furthermore,  we have shown that our methods would be useful in a wide variety of problems and that the algorithms can recover from the influence of weak/bad advisors during learning.

While there is a rich body of literature on the use of external knowledge sources in single-agent RL \citep{bignold21}, MARL provides additional challenges which mean that not all the results and approaches can transfer over directly. We discussed the important problems of directly using the single-agent based methods that learn from external sources in MARL. Additionally, we performed direct comparison experiments to elucidate a few of these problems.  Our approach to learning from advisors in MARL may look more complex compared to other single-agent approaches, however, the non-stationarity of the environment makes learning under the influence of advisors in MARL considerably more challenging. In MARL, quick adaptation to the changing environment is the key to better performance \citep{littman2001friend}. Our approach of using an online advisor is a more appropriate formulation of advisors in MARL, as real-time feedback against non-stationary opponents are critical for learning effective multi-agent policies, as demonstrated in our experiments.

Importantly, we consider a general setting, where we had no restrictions on the type or quality of the advisor, and no restrictions on the relation between the reward functions of different agents (general-sum). Particularly in MARL, the assumption of optimal advisors could be overly strong, since performance depends on the nature of opponents. The advisor could be capable of providing good feedback in strategizing against a particular class of opponents yet be useless against another class of opponents. Furthermore, a sub-optimal advisor could be good only in a very narrow portion of the state space, which is still useful for an agent learning from scratch in this environment. By explicitly allowing nonrestrictive \textbf{sub-optimal} advisors, our work is more widely applicable than previous methods that make an assumption of optimal (or near-optimal) experts to help RL training \citep{ross2011reduction, giusti2015machine, sonabend2020}.  

From the empirical perspective, as future work, we would like to study the performance of ADMIRAL algorithms under the simultaneous influence of multiple advisors providing conflicting demonstrations. This problem has been studied in single-agent RL environments \citep{li2019two}, but not yet in the MARL context. There is an emerging line of work that studies the possibility of multiple agents learning from peers in cooperative MARL settings \citep{omidshafiei2019learning}. Our paper has the potential to contribute to this line of work as well. Further, in this paper, we provided a simple technique of making use of the evaluation of advisors in a learning algorithm by setting the value of $\epsilon'_0$. More sophisticated ways of analyzing the performance of ADMIRAL-AE and using the results for learning faster and more effective decision-making policies is left to future work. An observation about ADMIRAL-AE is that, in MARL, the advisors can be used as a way to predict the behaviour of other agents as well, which is not relevant in single-agent settings. In MARL, each agent needs to have the ability to perform accurate opponent modelling, based on its observations, to obtain strong performances \citep{hernandez2019survey}. This is because the reward function and the transition dynamics depends on the joint action at each state. Previous methods have used several techniques, such as using a separate neural network for learning opponent behaviour \citep{he2016opponent}, learning policy features from raw observations \citep{zhang2018deep}, and using the agent's own policy to predict opponent actions \citep{roberta2018modeling}. However, many of these methods are computationally expensive and scale poorly with the number of states, actions, and agents. Another possibility for opponent modelling is leveraging an external advisor that can possibly predict opponent behaviour, as done in ADMIRAL-AE, which could be relatively computationally friendly given the availability of an appropriate advisor. This could open up a very interesting research direction in learning from advisors in MARL. Both our ADMIRAL-DM and ADMIRAL-AE algorithms are exponential in the number of agents as described in our complexity analysis (Section~\ref{sec:advisorqlearning}). Hence, our algorithms are not easily scalable to environments with a large number of agents which is one limitation of our framework. As future work, our framework could be combined with works on mean field games \cite{lasry2007mean} which can make the approach more scalable (since mean field methods guarantee a constant dependence on the number of agents).

From the theoretical perspective, as future work, we would like to fully characterize the convergence rates of our algorithms. Additionally, some of our theoretical assumptions for the environmental settings may seem restrictive, however, we assert that this work is the first to provide a theoretical foundation for MARL with advisors and that these assumptions are useful in understanding the strengths and limitations of such an approach. Furthermore, we note that the assumptions we make are also made by other works exploring the foundations of MARL, such as \citet{hu2003nash}. In future work, we wish to explore the ramifications of relaxing some of these assumptions. The theoretical understanding of MARL, in general, is still in its infancy and much more research into MARL theory is required to enhance our understanding of this area \citep{zhang2019multi}. In this paper, we restrict the theoretical analysis to tabular settings, which is in line with the state-of-the-art in theoretical analysis of learning in general-sum stochastic games \citep{zhang2019multi}. The objective is to provide a theoretical guarantee in the most basic (baseline) setting possible. Using a similar approach to single-agent RL methods that extend the tabular results to the function approximation setting \citep{carvalho2020new}, it would be possible to extend our theoretical results in this paper to the function approximation setting as well. An elaborate theoretical study of this is left to future work.

Two specific motivational applications for our work were discussed in Section~\ref{sec:introduction}. Similar to these examples, other domains such as managing natural disasters, controlling disease outbreaks, and learning to play multi-player/multi-team sports would also find benefit from our approach. These domains also have state-of-the-art advisor models. These models, while not optimal, are still used in practice. Due to the inherent poor sample efficiency of MARL methods, it becomes critical to use all available sources of knowledge judiciously. Hence, learning from fallible advisors is important for many real-world domains that typically have the structure of multiple agent interaction.

The framework we have introduced in this paper on using external information through advisors is expected to improve the sample complexity of MARL algorithms and is an important step towards making MARL methods usable in real-world environments. We expect that our paper will spark more work in the  area of accelerating RL training using other available sources of knowledge (or advisors) and that it will interest a broad community of researchers in the area of RL, game theory, machine learning, and multi-agent systems. 

\section*{Acknowledgements}

Kate Larson is an affiliate of the Vector Institute, Toronto. The authors would like to thank Seyed Majid Zahedi at  University of Waterloo for his comments on the paper draft. Resources used in preparing this research at the University of Waterloo were provided by the province of Ontario and the government of Canada through CIFAR, NRC, NSERC and companies sponsoring the Vector Institute. Part of this work has taken place in the Intelligent Robot Learning (IRL) Lab at the University of Alberta, which is supported in part by research grants from the Alberta Machine Intelligence Institute (Amii); a Canada CIFAR AI Chair, Amii; Compute Canada; Mitacs; and NSERC.

\clearpage

\appendix

\section{Proof For The Lemmas In The Main Paper}\label{appendix:proofs}

In this section, we will restate all the lemmas in the main paper with detailed proofs. No new lemmas are given in this section. 

\begin{lemm2}
Let us fix an arbitrary positive constant $C$, an arbitrary $w_0$, and a sequence $\epsilon$. Then provided that 

(i) $ y_t(w_0, \epsilon)$ converges to some point (independent of $t$) $\mathcal{D}$ 

(ii) The sequence $\epsilon$ converges to 0 in the limit ($t \xrightarrow{} \infty$)

\noindent The homogeneous process $x_t(w_0, \epsilon)$ converges to a point $\frac{1}{\hat{\beta}} \mathcal{D}$ w. p. 1, where $\hat{\beta}$, satisfying $0 < \hat{\beta} \leq 1$, is the scaling factor applied.  
\end{lemm2}

\begin{proof}
We state that 
\begin{equation}\label{eq:rescaling}
    \begin{array}{l}
         y_t(w, \epsilon) = x_t(d_t w, c_t \epsilon_t)
    \end{array}
\end{equation}

\noindent for some sequences $\{c_t\}$ and $\{d_t\}$, where $c_t = (c_{t0}, c_{t1}, \ldots, c_{ti}, \ldots)$, $\{c_t\}$ and $\{d_t\}$ satisfy $0,d_t, c_{ti} \leq 1$, and $c_{ti} = 1$ if $i\geq t$. Here, the product of the sequences $c_t$ and $\epsilon_t$ is component wise: $(c_t \epsilon_t)_i = c_{ti} \epsilon_i$. Note that $y_t(w, \epsilon_t)$ and $x_t(w,\epsilon_t)$ depend only on $\epsilon_0, \cdots, \epsilon_{t-1}$. Thus, it is possible to prove Eq.~\ref{eq:rescaling} by constructing the appropriate sequence $c_t$ and $d_t$. 

Set $c_{0i} = d_i = 1$ for all $i = 0,1,2, \ldots$. Then Eq.~\ref{eq:rescaling} holds for $t=0$. Let us assume that $(c_i, d_i)$ is defined in a way that Eq.~\ref{eq:rescaling} holds for $t$. Let $B_t$ be the ``scaling coefficient'' of $y_t$ at step $t+1$ ($B_t = 1$ if there is no scaling, otherwise $0<B_t<1$ with $B_t = C/||G_t(y_t, \epsilon_t||$). Now we have: 

\begin{equation}
    \begin{array}{l}
         y_t(w, \epsilon_t) = B_t G_t (y_t(w,\epsilon_t), \epsilon_t) \\
         
         = G_t(B_t y_t(w, \epsilon_t), B_t \epsilon_t)\\
         
         = G_t(B_t x_t(d_tw, c_t\epsilon_t), B_t \epsilon_t).
    \end{array}
\end{equation}

We claim that 
\begin{equation}\label{eq:claimforS}
    \begin{array}{l}
         B{x_t} (w, \epsilon_t) = x_t(B w, B \epsilon_t)
    \end{array}
\end{equation}

\noindent holds for all $w$, $\epsilon_t$ and $B>0$. For $t = 0$, this obviously holds. Assume that it holds for some time $t$. Then, from Eq.~\ref{eq:homogeneous},

\begin{equation}
    \begin{array}{l}
         B{x_{t+1}} (w, \epsilon_t) = B G_t (x_t(w, \epsilon_t), \epsilon_t)\\
         = G_t(B{x_t}(w, \epsilon_t), B\epsilon_t)
        = x_{t+1}(B w, B \epsilon_t)
             \end{array}
\end{equation}

\noindent Thus, 

\begin{equation}
    \begin{array}{l}
         y_{t+1}(w,\epsilon_t) = G_t (x_t (B_t d_t w, B_t c_t \epsilon_t), B_t \epsilon_t), 
        
    \end{array}
\end{equation}

\noindent and we see that Eq.~\ref{eq:rescaling} holds if we define $c_{t+1,i}$ as $c_{t+1,i} = B_t c_{ti}$ if $0 \leq i \leq t$, $c_{t+1, i} = 1$ if $i>t$ and $d_{t+1} = B_t d_t$.

Thus, we get that with the sequences 

\begin{equation}
    \begin{array}{l}
         c_{t,i} = \Pi_{j=i}^{t-1}B_j, \textrm{ if } i<t;
         \\
          c_{t,i} = 1, otherwise;
         
    \end{array}
\end{equation}

\noindent $d_0 = 1$ and 

\begin{equation}
    \begin{array}{l}
         d_{t+1} = \Pi^t_{i=0} B_i, 
    \end{array}
\end{equation}

\noindent Eq.~\ref{eq:rescaling} is satisfied for all $t\geq0$. 




We know that $y_t(w, \epsilon_t) \xrightarrow{} \mathcal{D}$ w. p. 1. Since the process $y_t$ has been constructed by bounding the process with $|| y_t ||\leq C$, it follows that $\mathcal{D} \leq C$. Then, there exists a finite index $M$ such that if $t>M$ then 

\begin{equation}\label{eq:probability}
    \begin{array}{l}
         Pr(||y_t(w, \epsilon_t)|| < C) > 1 - \delta
    \end{array}
\end{equation}

\noindent without applying any more rescaling. Now, let us restrict our attention to those events $\omega$ for which $||y_t(w, \epsilon_t(\omega))|| < C$ for all $t > M$ without rescaling. Thus, we have no rescaling for all $t$, such that $t>M$. Thus, $c_{t,i} = c_{M+1,i}$ for all $t \geq M+1$ and $i$, and specifically $c_{ti} = 1$ if $i,t \geq M + 1$. Similarly, if $t>M$ then $d_{t+1}(\omega) = \Pi_{i=0}^{M} B_i (\omega) = d_{M +1} (\omega)$. Let $A_\omega = \{\omega: ||y_t(w, \epsilon)(\omega)|| < C\}$ without rescaling. By Eq. \ref{eq:rescaling}, we have that if $t>M$ then, 

\begin{equation}
    \begin{array}{l}
         y_t(w, \epsilon_t(\omega)) = 
         x_t(d_{M+1}(\omega)w, c_{M+1}(\omega) \epsilon_t(\omega)).
    \end{array}
\end{equation}

Thus, it follows from our assumption concerning $y_t$ that $x_t(d_{M+1}(\omega) w, c_{M+1} \epsilon_t(\omega))  $ converges to $\mathcal{D}$ almost everywhere (a. e.) on $A_\omega$ and, consequently, by Eq.~\ref{eq:claimforS}, $x_t(w, c_{M+1}\epsilon_t(\omega)/d_{M+1}(\omega))$ converges to $\mathcal{D}' = \frac{1}{d_{M+1}}\mathcal{D}$ a. e. on $A_{\omega}$. Since $c_{M+1} = 1$ in the limit and we are analyzing in the space of $A_\omega$, $x_t(w, \epsilon_t(\omega)/d_{M+1}(\omega))$ converges to $\mathcal{D}'$ too. Now, since $\epsilon$ converges to 0 in the limit, and we are analyzing values in the space of $A_\omega$, we can write $x_t(w, \epsilon_t(\omega)) \approx x_t(w, \epsilon_t(\omega)/d_{M+1}(\omega)))$ which also converges to $\mathcal{D}'$. All these hold with probability at least $1-\delta$, since, by Eq.~\ref{eq:probability}  $Pr(A_{\omega} > 1- \delta)$. Since $\delta$ was arbitrary, the lemma follows. 

\end{proof}

\begin{lemm2}
Let $X$ and $Y$ be normed vector spaces, $U_t: X \times Y \xrightarrow{} X (t=0,1,2, \ldots)$ be a sequence of mappings, and $\theta_t \in Y$ be an arbitrary sequence. Let $\theta_\infty \in Y$ and $x_\infty \in X$. Consider the sequences $x_{t+1} = U_t(x_t, \theta_\infty)$, and $y_{t+1} = U_t(y-t, \theta_t)$ and suppose that $x_t$ and $\theta_t$ converge to $x_\infty$ and $\theta_\infty$ respectively, in the norm of the appropriate spaces. 

Let $L^\theta_k$ be the uniform Lipschitz index of $U_k(x,\theta)$ with respect to $\theta$ at $\theta_\infty$ and, similarly, let $L^\mathscr{X}_t$ and $L^\theta_t$ satisfy the relations $L^\theta_t \leq C(1 - L^\mathscr{X}_t)$, and $\Pi_{m=t}^\infty L^\mathscr{X}_m = 0$ where $C>0$ is some constant and $t = 0,1,2, \ldots,$ then $\lim_{t \xrightarrow{} \infty} || y_t - x_\infty|| = 0$.
 
\end{lemm2}

\begin{proof}
See Theorem 15 in \citet{szepesvari1999unified} for detailed proof.  
\end{proof}

\begin{lemm2}

Let $\mathcal{Z}$ be an arbitrary set and consider the process
\begin{equation*}
    \begin{array}{l}
         x_{t+1}(z) = G_t(z)x_t(z) + F_t(z) (C + k_t(z) C)
    \end{array}{}
\end{equation*}

\noindent where $x_1, F_t, G_t \geq 0 $ are random processes,  $||x_1|| < C < \infty$ w. p. 1 for some $C>0$, and $z$ is an element in $\mathcal{Z}$. Assume that for all $k$, $\lim_{n \xrightarrow{} \infty} \Pi_{t=k}^n G_t(z) = 0$ uniformly in $z$ w. p. 1 and $F_t(z) = \gamma(1 - G_t(z))$, for some $0 \leq \gamma < 1$, and $\forall z \in \mathcal{Z}$,  w. p. 1. Also, $k_t(z)$ converges to $K(z)$ in the limit. Then, $x_t(z)$ converges to a point $D(z) = \gamma(C + K(z)C)$ w. p. 1.

\end{lemm2}

\begin{proof}

Consider the process that is obtained from substituting the value of $F_t(z)$ in terms of $G_t(z)$ in Eq.~\ref{eq:limitlemma},

\begin{equation}\label{eq:limitlemma2}
    \begin{array}{l}
         x_{t+1}(z) = G_t(z)x_t(z) + \gamma(1 - G_t(z)) (C + k_t(z) C).
    \end{array}{}
\end{equation}

Now, subtracting $\gamma (C + k_t(z)C)$ on both sides of the equation we get 

\begin{equation}\label{eq:limitlemma3}
    \begin{array}{l}
         x_{t+1}(z) - \gamma (C + k_t(z)C) =  G_t(z)(x_t(z) - \gamma (C + k_t(z) C)).
    \end{array}{}
\end{equation}

The above equation converges to 0 in the limit as $\lim_{n \xrightarrow{} \infty} \Pi_{t=k}^n G_t(z) = 0$. Thus, we have that $x_{t+1}(z) - \gamma (C + k_t(z)C)$ converges to 0 in the limit. Hence, $x_t$ converges to $\gamma (C + K(z)C)$, since $k_t(z)$ converges to $K(z)$ in the limit. 

\end{proof}

\begin{lemm2}
Consider an equation of the form
\begin{equation*}
    \begin{array}{l}
         x_{t+1}(z) = G_t(z) x_t(z) + F_t(z) (||x_t|| + \epsilon_t + k_t(z) ||x_t||)
    \end{array}
\end{equation*}

\noindent where the sequence $\epsilon_t $ converges to zero w. p. 1. Assume that for all $k$, $\lim_{n \xrightarrow{} \infty} \Pi_{t=k}^n G_t(z) = 0$ uniformly in $z$ w. p. 1 and $F_t(z) = \gamma(1 - G_t(z))$, for some $0 \leq \gamma < 1$, and $\forall z \in \mathcal{Z}$,  w. p. 1. Assume further that $k_t(z)$ is finite, and it converges to $K(z)$ in the limit ($t \xrightarrow{} \infty$). Then $x_t(z)$ converges to a point represented by $S'(z) = \frac{1}{\hat{\beta}} (\gamma C_1 + K(z) C_1) $, where $C_1$ is a small positive constant, w. p. 1 uniformly over $\mathcal{Z}$. Here $\hat{\beta}$ is a scaling factor satisfying $0 < \hat{\beta} \leq 1$.

\end{lemm2}

\begin{proof}

Let us consider a process $y_t$ that is obtained from keeping the original process $||x_t||$ bounded by a constant $C_1$. This is an arbitrary bound, with $C_1$ specified to be a small positive constant. Since, $||x_t||$ is guaranteed to be positive, we can find such a positive $C_1$. Now we get,

\begin{equation}\label{eq:originalprocess}
    \begin{array}{l}
      y_{t+1}(z) = G_t(z) y_t(z)  + \gamma (1 - G_t(z)) (C_1 + \epsilon_t + k_t(z) C_1 ).
    \end{array}
\end{equation}

By Lemma~\ref{lemma:normedvectorspace}, $y_t$ converges to $\gamma C_1 + K(z) C_1 $ as the following bindings show $X, Y := \textbf{R}, \theta_t := \epsilon_t, U_t(x,\theta) := G_t(z)x + \gamma (1 - G_t(z)) (C_1 + k_t(z) C_1 + \theta)$, where $z \in Z$ is arbitrary. Then, $L^X_t = G_t(z)$ and $L^\theta_t = \gamma(1 - G_t(z))$ satisfying the conditions of Theorem~\ref{lemma:normedvectorspace}. We know that $x$ in the expression $U_t(x,\theta)$ converges to $\gamma C_1 + K(z) C_1 $ as proved in Lemma~\ref{lemm:boundedlemma}. 

In Lemma~\ref{lemm:rescaling}, we proved that the original process will converge to a point $\frac{1}{\hat{\beta}} \mathcal{D}$, if the bounded process converges to $\mathcal{D}$. Here $\hat{\beta}$ is the scaling factor applied, satisfying the relation $0 < \hat{\beta} \leq 1$. Using this result, now we get that the process represented by the Eq.~\ref{eq:perterbationterm} converges to a point $S'(z) = \frac{1}{\hat{\beta}} (\gamma C_1 + K(z) C_1) $ which is a constant for a given $z$. This gives an expression for the point $S$ in Theorem~\ref{maintheorem}. 


\end{proof}

\begin{lemm2}
Let $\mathcal{F}_t$ be an increasing sequence of $\sigma$-fields, let $0\leq \alpha_t$ and $w_t$ be random variables such that $\alpha_t$ and $w_{t-1}$ are $\mathcal{F}_t$ measurable. Assume that the following hold w. p. 1: $E[w_t|\mathcal{F}_t, \alpha_t \neq 0] = A$, $E[w^2_t|\mathcal{F}_t] < B < \infty$, $\sum_{t=1}^{\infty} \alpha_t = \infty$ and $\sum_{t=1}^{\infty} \alpha^2_t < C < \infty$ for some $B,C > 0$. Then the process 

\begin{equation*}
    \begin{array}{l}
         Q_{t+1} = (1 - \alpha_t)Q_t + \alpha_t w_t
    \end{array}
\end{equation*}

\noindent converges to A w. p. 1. 

\end{lemm2}

\begin{proof}
Refer to Lemma 4 in \cite{szepesvari1999unified} for detailed proof. 
\end{proof}

\begin{lemm2}
Assume that $\alpha_t$ satisfies Assumption 2 and the mapping $P_t: \mathcal{Q} \xrightarrow{} \mathcal{Q}$ satisfies the condition that, there exists a scalar $\gamma$ satisfying $0 \leq \gamma < 1$, a sequence $\lambda_t \geq 0$ converging to zero w. p. 1, and a finite sequence $k_t(s)$ such that $|| P_tQ  - P_t Q_{*} || = \beta ||Q - Q_{*} || + \lambda_t + \beta k_t(s)||Q - Q_{*} ||$ for all $Q$, and all $s \in \mathcal{S}$. Assume further that, $k_t(s)$ converges to a finite point $K(s)$ in the limit. Additionally, $Q_{*}(s, \boldsymbol{a}) = E[P_tQ_{*}(s, \boldsymbol{a}) ]$, then the iteration defined by 
\begin{equation*}
    \begin{array}{l}
         Q_{t+1}(s, \boldsymbol{a}) = (1 - \alpha_t)Q_t(s, \boldsymbol{a})  +\alpha_t[P_t Q_t(s, \boldsymbol{a})]
    \end{array}
\end{equation*}
converges to $(Q_{*} - S)$ w. p. 1, where $S$ is as given in Theorem~\ref{maintheorem}. 
\end{lemm2}

\begin{proof}
This lemma directly follows from Corollary~\ref{corr:conditionalCorollary} and Lemma~\ref{lemm:conditionalaveraging}. 
\end{proof}

\begin{lemm2}
For a $n$-player stochastic game, $E[P_t Q_{*}] = Q_{*}$ where $Q_* = (Q^1_{*}, \ldots, Q^n_{*}).$
\end{lemm2}

\begin{proof}


         
Refer to Lemma 11 in  \cite{hu2003nash} for proof. 
\end{proof}

\begin{lemm2}

A random iterative process 

\begin{equation*}
\begin{array}{l}
     \Delta_{t+1}(x) = (1 - \alpha_t(x))\Delta_t(x) + \alpha_t(x) F_t(x)
\end{array}{}
\end{equation*}

\noindent where $x \in X$, $t = 0,1, \ldots, \infty$, converges to zero with probability one (w. p. 1) if the following properties hold: 
 
1. The set of possible states $X$ is finite. 

2. $0 \leq \alpha_t(x) \leq 1$, $\sum_t \alpha_t(x) = \infty$, $\sum_t \alpha^2_t(x) < \infty$ w. p. 1, where the probability is over the learning rates $\alpha_t$. 

3. $|| \E \{{F_t(x)|\mathscr{P}_t}\} ||_W \leq \mathscr{K} ||\Delta_t||_W + c_t$, where $\mathscr{K} \in [0,1)$ and $c_t$ converges to zero w. p. 1. 

4. $\textrm{\textbf{var}}\{F_t(x) | \mathscr{P}_t\} \leq K(1 + ||\Delta_t||_W)^2$, where $K$ is some constant. 

\noindent Here $\mathscr{P}_t$ is an increasing sequence of $\sigma$-fields that includes the past of the process. In particular, we assume that $\alpha_t, \Delta_t, F_{t-1} \in \mathscr{P}_t$. The notation $||\cdot||_W$ refers to some (fixed) weighted maximum norm.

\end{lemm2}

\begin{proof}

Refer to Theorem 1 in \citet{jaakkola1994convergence} for proof. 

\end{proof}

\begin{lemm2}
Under Assumption~\ref{assumption:globaloptimum}, the Nash operator as defined in Eq.~\ref{eq:nashoperator} forms a contraction mapping with the fixed point being the Nash $Q$-value of the game. 

\end{lemm2}

\begin{proof}
The detailed proof is given in Theorem~17 of \citet{hu2003nash}. 
\end{proof}

\section{Additional Definitions For Theorem~\ref{maintheorem}}\label{appendix:definitions}

We restate some definitions from \cite{szepesvari1999unified}, needed for us in Theorem \ref{maintheorem}, to stay self-contained.

Let us consider an arbitrary operator, $T: \mathcal{B} \xrightarrow{} \mathcal{B}$, where $\mathcal{B}$ is a normed vector space with norm $||.||$. Let $\mathcal{T} = (T_0, T_1, \cdots, T_t, \cdots)$ be a sequence of random operators, $T_t$ mapping $\mathcal{B} \times \mathcal{B}$ to $\mathcal{B}$.

\begin{defn}
Let $F \subseteq \mathcal{B} $ be a subset of $\mathcal{B}$ and let $\mathcal{F}_0: F \xrightarrow{} 2^{\mathcal{B}}$ be a mapping that associates subsets of $\mathcal{B}$ with the elements of $F$. If, for all $f \in F$ and all $m_0 \in \mathcal{F}_0(f)$, the sequence generated by the recursion $m_{t+1} = T_t(m_t, f)$ converges to $Tf$ in the norm of $\mathcal{B}$ with probability 1, then we say that $\mathcal{T}$ approximates T for initial values from $\mathcal{F}_0(f)$ and on the set $F \subseteq \mathcal{B}$. Further, we say that $\mathcal{T}$ approximates $T$ on the singleton set ${f}$ and the initial value mapping $\mathcal{F}_0: F \xrightarrow{} B$ defined by $\mathcal{F}_0(f) = F_0$.

\end{defn}

\begin{defn}
The subset $F \subseteq \mathcal{B}$ is invariant under $T: \mathcal{B} \times \mathcal{B} \xrightarrow{} \mathcal{B}$ if, for all $f,g \in F, T(f,g) \in F$. If $\mathcal{T}$ is an operator sequence as above, then $F$ is said to be invariant under $\mathcal{T}$ if for all $i \geq 0$, $F$ is invariant under $T_i$. 
\end{defn}

\section{ADMIRAL-AE Using An Adaptive Advisor}\label{sec:adaptive}

In this sub-section, we aim to provide an illustration of the behaviour of ADMIRAL-AE in the presence of an adaptive advisor. This advisor would actively change and adapt its strategies according to the changing opponent. The objective is to show that the ADMIRAL-AE algorithm will be able to capture the strength of the adaptive advisor, and hence there is merit in using a principled approach to evaluate an advisor. Also, we would like to discuss the advantages of keeping this evaluation method separate from another approach that aims to learn from the advisor. In this section, we clarify that we are only considering the `pre-learning' phase introduced in our paper, since our objective is to analyze the performance of ADMIRAL-AE with two different advisors. Results in this section use the average and standard deviation of 30 runs.

\begin{figure}[h]
    \centering
    \includegraphics[width=0.5\textwidth]{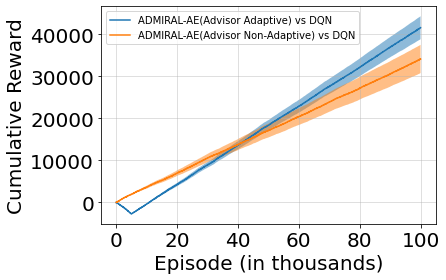}
    \caption{Two instances of ADMIRAL-AE along with an adaptive advisor and a non-adaptive advisor on Pommerman Domain~OneVsOne. The plot shows the performance of ADMIRAL-AE using the Advisor Non-Adaptive against the common opponent of DQN (orange line) and the performance of ADMIRAL-AE using the Advisor Adaptive against the common opponent of DQN (blue line). This plot shows that the ADMIRAL-AE using an adaptive advisor starts off week, but eventually surpasses the performance of the non-adaptive advisor. Thus, a principled method like ADMIRAL-AE can evaluate an adaptive advisor appropriately, while other non-principled approaches may have a high percentage of failure.}
    \label{fig:adaptiveadvisor}
\end{figure}

We will continue to use the two-agent version (Domain~OneVsOne) of Pommerman for this experiment. We consider two different advisors, namely, Advisor~Non-Adaptive and Advisor~Adaptive. The Advisor~Non-Adaptive does not actively track opponent strategies or adapt to them. This advisor plays a balanced strategy --- choosing to be risk-seeking (actively laying bombs to kill the opponent) while the opponent is in proximity and choosing to be risk-averse (escaping from the enemy) otherwise. This advisor is reactive and only responds to the present position of the opponent, and does not attempt to model the opponent's nature actively. It has been observed that this is a relatively good strategy in Pommerman, particularly in the early stages of training \citep{meisheri2019accelerating}.  At this stage, an agent playing a relatively conservative strategy could wait for the opponent to make a mistake and kill itself. However, this strategy is not very strong and could lose out once the opponent is well-trained. A well-trained opponent is less likely to make the mistake of killing itself, and there is the added possibility of the non-adaptive strategy becoming predictable, which could be figured out by the opponent. Hence, the Advisor~Non-Adaptive will find it hard to win games in the middle and later stages of training.  On the other hand, the Advisor~Adaptive uses an adaptive strategy that plays a risk-averse strategy when the enemy is risk-seeking, and a risk-seeking strategy  when the enemy is risk-averse. This is a very strong strategy for winning in Pommerman \citep{zhou2018hybrid} but requires active modelling of the opponent. The Advisor~Adaptive tracks the percentage of bombs played by the opponent to determine the nature of the opponent. In the initial few episodes (about 5000) the Advisor~Adaptive's behaviour is close to random since it still does not have enough information to learn the nature of the opponent. 

We implement a separate instance of the ADMIRAL-AE algorithm with both the advisors and plot the performance against a common baseline agent using DQN for learning. We run all the training for 100,000 episodes and plot the performances of both algorithms. The results are captured in Figure~\ref{fig:adaptiveadvisor}. The results show that ADMIRAL-AE using the Advisor~Adaptive loses out in the beginning while it is still figuring out the nature of the opponent. However, it soon shows a much stronger performance that surpasses the performance of the ADMIRAL-AE using the Advisor~Non-Adaptive. ADMIRAL-AE using the Advisor~Non-Adaptive, while showing good overall performance, does not quite reach the levels of the performance of ADMIRAL-AE using the Advisor~Adaptive, due to the non-adaptive strategy possibly becoming predictable and prone to exploitation by the opponent, after the opponent has trained for a sufficient number of episodes. This performance shows that an agent during learning should listen more to the Advisor~Adaptive as compared to the Advisor~Non-Adaptive for the best outcome. If the advisors were directly used for learning without being evaluated separately, then a learning agent would be prone to discarding Advisor~Adaptive quickly due to its initial weak performance, while the agent would listen more to the Advisor~Non-Adaptive. Yet we have seen that the opposite behaviour would actually be better using an implementation of ADMIRAL-AE with each of these advisors. This demonstrates that the evaluation would be prone to inaccuracy and inconsistency if it is combined with learning a policy. The analysis in this sub-section shows the advantage of using a principled evaluation algorithm (ADMIRAL-AE) for evaluating an advisor, especially when they have adaptive characteristics.

\section{Experimental Details}\label{appendix:experimentaldetails}

In this section, we provide the experimental details for all the experiments. We have given detailed information about the advisors and reward functions for all the experiments. Hyperparameters for all algorithms are also provided.  

\subsection{Grid Maze Domain}
The reward function is defined in such a way that the agents get a +1 if any one of two agents reaches the goal, and they get a -1 if any one of the two gets to the pitfall. Both the agents get a +2 if both reach the goal at the same instant, and both the agents get a -2 if they reach a pitfall at the same instant. If one of the two gets to the goal and the other gets to the pitfall, they both still get a +1. In this environment, each agent obtains a local state (observation) which corresponds to the coordinates of the grid cell the agent is currently located. We use the joint observations of the two agents as the state in the environment for determining the $Q$-values in the experiments showing convergence in Section~\ref{sec:gridworldappendix}. This state is available to both agents. The other experiments simply use the observation of the concerned agent in the $Q$-updates, in case of both ADMIRAL-DM and ADMIRAL-AE, to make the challenge harder.  

The actions that the agents can take in this game are one of moving up, down, left or right. If the wall obstructs an action, then the agent will remain at the same spot.

All four advisors used are rule-based agents, where Advisor~1 follows high-quality rules at each grid cell, which enables the agent to move to the goal state and avoid the pitfalls. Advisor~2 on the other hand can only suggest the correct actions for the agents to reach the goal and escape the pitfall if the agent is right next to the goal or pitfall (within one step). It is not capable of giving the right action in the other parts of the grid. Thus, Advisor~2 cannot teach perfect coordination to obtain the large positive reward (+2). Advisor~3 can only suggest actions that make the agents move closer to the goal state, but cannot give accurate actions that avoid pitfalls. Advisor~4 is a random advisor which only suggests random actions.

The ADMIRAL-AE implementation for this experiment chooses to do the advisor actions with a probability of 50\% ($\eta' = 0.5$). The action that maximizes the $Q$-value (greedy action) is chosen with probability 45\%  and a random action with probability 5\% ($\eta = 0.05$) to satisfy Assumption~\ref{assumption:visitassumption}. We use the previous actions of the other agents while determining actions at the current time step $t$ as done in the other algorithms. The ADMIRAL-DM algorithm simply follows the scheme in Algorithm~\ref{alg:advisorQ2} where the advisor suggestions are performed with decreasing probabilities and the algorithm becomes completely greedy (dependent only on the agent's own policy) after a finite number of episodes.

\subsection{Pommerman Domain}

All the advisors are based on the rule-based agent (called the simple agent) already provided in the Pommerman framework. It is well documented that the simple agent is hard to beat even for complex RL algorithms in the Pommerman environment \citep{resnick2018pommerman}. We use four advisors for the ADMIRAL-AE experiments in Section~\ref{sec:experimentswithmaeqlee} and \ref{sec:pommerman}, of which the first advisor alone is used in the experiments that study the performance of ADMIRAL-DM and ADMIRAL-DM(AC) in Section~\ref{sec:experiments}. The first advisor (Advisor 1) is the best of all the four advisors and has rules for all the different aspects of the game (escaping from enemies, collecting bombs, laying bombs, etc.). The second advisor (Advisor 2) has rules for moving away from danger and collecting powerups like life and bombs. This is a relatively conservative advisor which relies on staying safe and hoping for the opponent to make a mistake. Actually, this is a very effective strategy in the Pommerman game and therefore Advisor 2 is also a useful one. Advisor 3 only has rules for collecting powerups, but cannot teach any of the other strategies needed to win the Pommerman game. When there is no possibility of teaching any useful strategy based on the situation of the game (no nearly enemies, powerups, etc.), the first three advisors provide pseudo-random strategies that encourage exploration of states which are relatively less visited. The fourth advisor (Advisor 4) is the weakest advisor of the four. It suggests random actions to the agent and does not contribute to the learning of the agent (rather, it harms learning). 

Each new Pommerman game gives a randomized game map and there is a maximum of 800 steps for each game. 

\subsection{Pursuit Domain} 

The Pursuit domain was initially defined in SISL \citep{gupta2017cooperative}. We use the canonical domain implemented in the petting zoo environment \citep{terry2020pettingzoo}. All the environmental parameters including the rewards are left as the same as the defaults in \cite{gupta2017cooperative}. For the training phase, the game has 30 evader agents and 8 pursuer agents, where the evader agents move randomly, and the pursuer agents are controlled by learning algorithms. The game is cooperative, with the pursuers receiving a reward of 5 for fully surrounding an evader (the evader is removed from the environment). The pursuers receive a reward of 0.01 for each time they touch an evader. The actions space is discrete with 5 actions, each action corresponding to moving to a neighbouring grid (including stay). Each episode is a full game with a maximum of 500 steps. Each agent in this environment is able to observe a grid of 7 $\times$ 7 centered around itself (the whole grid is 16 $\times$ 16). Since this environment is fully cooperative with homogeneous learning agents, all training (for each algorithm) is completely centralized (all agents train the same network). 

\subsection{Waterworld Domain}

Waterworld is also a SISL environment with a set of 5 pursuers tasked with collecting food and avoiding poison. We use the corresponding petting zoo environment \citep{terry2020pettingzoo}. The environmental parameters are the same as that in \cite{gupta2017cooperative}. This is also a cooperative environment where multiple pursuers need to work together to capture food. The environment is a two-dimensional continuous space, where the action space is a continuous two-dimensional value representing the horizontal and vertical thrust that makes the agents move in particular directions with the desired speed. The local state (observation) of each agent consists of multiple sensor features and two other elements that indicate the collision of the agent with a food or poison respectively. In this environment, at least two agents need to attack a food together to capture it. There are a total of 5 food particles (not destroyed upon capture) and 10 poison particles in the environment. The agents get a reward of 10.0 for capturing food and a 0.01 for encountering food. Further, the agents have a thrust penalty of -0.5, and a penalty of -1.0 for encountering poison. The poison particles and food particles move in the environment with a speed of 0.01. Since this environment is also fully cooperative with homogenous learning agents, all training is completely centralized for this domain too.


\subsection{Hyperparameters And Implementation Details}

The hyperparameters for the baselines were chosen to be the same as those recommended by the respective papers. Some minor modifications were made due to performance and computational efficiency considerations. 

Regarding the hyperparameters of DQfD, we set $1 \times 10^6$ as the demo buffer size and perform 50,000 mini-batch updates for pretraining. The replay buffer size is twice the size of the demo buffer. The N-step return weight is 1.0, the supervised loss weight is 1.0 and the L2 regularization weight is $10^{-5}$. The epsilon greedy exploration is 0.9. The discount factor is 0.99 and the learning rate is 0.002. 
The pretraining for DQfD comes from a data buffer related to a series of games where two rule-based agents (advisors) compete against each other.  All other values are similar to that used in \cite{hester2018deep}.

The CHAT \citep{wang2017improving} implementation uses a neural network for confidence measurement (termed NNHAT in \cite{wang2017improving}). The learning rate is 0.01, we use a discount factor of 0.9 and a fixed exploration constant ($\epsilon$-greedy) of 0.9. We use the extra action variant of HAT \citep{taylor2011integrating} in the CHAT implementation, as this gave the best performance in most of our comparative experiments. A neural network is used as the function approximator, as described in \cite{mnih2015human}. The target net is replaced every 10 learning iterations. The confidence threshold is set as 0.6 and the default action as ``action-0''. The mini-batch size is 32 and learning rate is $\alpha = 0.01$. The CHAT algorithm can directly use advisors in an online fashion similar to ADMIRAL-DM. We omit the rule summarization step of CHAT, and directly use the advisor, to make the performance as good as possible.  The replay buffer size is $2 \times 10^6$. 


The DQN and ADMIRAL-DM hyperparameters are the same as those mentioned for CHAT (as relevant). These algorithms also perform $\epsilon$-greedy exploration with a constant value of 0.9 for $\epsilon$. The advisor influence parameter ($\epsilon'_t$ in Algorithm~\ref{alg:advisorQ2}) for ADMIRAL-DM and ADMIRAL-DM(AC) starts at a high value of 0.8 at the beginning of the training and linearly decays to 0.01 during training. The face-off and execution phases have this parameter to be 0 (no advisor influence). All other hyperparameters for DQN, ADMIRAL-DM and CHAT are similar to that used in  \citet{mnih2015human}. The ADMIRAL-DM(AC) has the critic learning rate set at $10^{-3}$ and the actor learning rate to be $10^{-5}$. The discount factor for ADMIRAL-DM(AC) is 0.9. 


The ADMIRAL-AE algorithm in Section~\ref{sec:experimentswithmaeqlee} chooses to do the advisor action with a probability of 50\%. The action that maximizes the $Q$-value (the greedy action) is chosen with probability 45\% and a random action with probability 5\% to satisfy Assumption~\ref{assumption:visitassumption}. This is the same as that in the tabular domain (Section~\ref{sec:gridworldappendix}). The hyperparameters of ADMIRAL-AE are the same as ADMIRAL-DM. 

For the function approximation experiments, both ADMIRAL-DM and ADMIRAL-DM(AC) simply use the observed current actions of other agents for action selection instead of maintaining copies of policies of the other agents as specified in Algorithm~\ref{alg:maeqldm} and Algorithm~\ref{alg:maeacdm}, since the opponents could possibly be using different algorithms (all agents are not always using the same algorithmic steps). This is also computationally more efficient. The current actions of all other agents are either directly observed or provided by the game engine to each agent in all our experiments, to perform a joint action update.

The DDPG uses the learning rate of the actor as 0.001 and the critic as 0.002. The discount factor is 0.9. We use the soft replacement strategy with a learning rate of 0.01. The batch size is 32. The PPO implementation also uses the same batch size and actor and critic learning rates. All other values are similar to those used in \cite{schulman2017proximal} (PPO) and \cite{lillicrap2015continuous} (DDPG). 
For PPO, we used a single-thread implementation, which we found to be as good as the multi-threaded implementation for our experiments, and more computationally efficient. This could be because the data correlations are already being broken by the multi-agent (non-stationary) nature of the domains.

For the Deep Sarsa implementation used in this paper, we follow almost the same steps as done in Algorithm~\ref{alg:maeqlae}, except that we do not have any advisors and hence the algorithm does not have the term $\epsilon'_t$ in the implementation. As a consequence, the greedy action is selected with probability $(1 - \epsilon_t)$. Also, we are using an independent implementation, where the actions of the other agents are not considered during action selection by the algorithms playing Deep Sarsa. The action is chosen only based on the current state. Deep Sarsa uses the same hyperparameters as ADMIRAL-DM. The values are either the same or closely match those considered in previous research \citep{mnih2015human}. 

We use a set of 30 random seeds (1--30) for all training experiments and a new set of 30 random seeds (31--60) for all execution experiments. 

\subsection{Wall Clock Times} 

The experiments on the Grid Maze domain can just be run on the CPU and takes less than 20 minutes to complete. 

The training for all the experiments on the Pommerman domain in Section~\ref{sec:experimentswithmaeqlee} and Section~\ref{sec:experiments} was run on a 2 GPU virtual machine with 16 GB GPU memory per GPU. The experiments take an average of 18 hours wall clock time to complete. We use Nvidia Volta-100 (V100) GPUs for all these experiments. The CPUs use Skylake as the processor microarchitecture.

The experiments on Pursuit and Waterworld domains in Section~\ref{sec:experiments} were run on a virtual machine having the same configuration containing 2 GPUs. These experiments take an average of 12 hours wall clock time to complete. The experiments on Pommeran in Section~\ref{sec:pommerman} took an average of 15 hours wall clock time to complete.

\vskip 0.2in
\bibliography{main}
\bibliographystyle{apalike}

\end{document}